\documentclass{article}
\usepackage{graphicx}
\usepackage{multirow}
\usepackage{amsmath,amssymb,amsfonts}
\usepackage{amsthm}
\usepackage{mathrsfs}
\usepackage[title]{appendix}
\usepackage{xcolor}
\usepackage{textcomp}
\usepackage{manyfoot}
\usepackage{booktabs}
\usepackage{algorithm}
\usepackage{algorithmicx}
\usepackage{algpseudocode}
\usepackage{listings}
\usepackage{authblk}

\theoremstyle{plain}
\newtheorem{theorem}{Theorem}

\newtheorem{example}{Example}

\newtheorem{definition}{Definition}

\usepackage{hyperref}

\usepackage{amsmath}
\usepackage{amssymb}
\usepackage{enumerate,enumitem}
\usepackage{amsthm}
\usepackage{mathtools}
\usepackage{xcolor,colortbl}
\usepackage{xfrac}
\usepackage{tabularx,tabu}
\usepackage{enumerate}
\usepackage{stmaryrd}
\usepackage{algorithm,algorithmicx}
\usepackage{xspace}
\usepackage{xfrac}
\usepackage{tikz}
\usetikzlibrary{arrows.meta, positioning}

\usepackage{caption}
\usepackage{subcaption}

 \usetikzlibrary{positioning,calc,backgrounds,automata, arrows,matrix}
 \usepackage[bb=boondox]{mathalfa}
 \usepackage{wrapfig}
 \usepackage{siunitx}
 \usepackage{array}
 \usepackage{multirow}

\newtheorem{corollary}{Corollary}[theorem]
\newtheorem{lemma}[theorem]{Lemma}

 \newtheorem{assumption}{Assumption}
 \newtheorem{problem}{Problem}

\mathchardef\mhyphen="2D

\newcommand{\seq}[1]{\vec{#1}} 
\newcommand{\sem}[1]{\llbracket #1\rrbracket}

\newcommand{\NN}{\mathbb{N}}
\newcommand{\pNN}{\mathbb{N}^+}
\newcommand{\RN}{\mathbb{R}}

\newcommand{\inter}[1]{I(#1)}

\newcommand{\identity}{\mathrm{id}}
\newcommand{\indi}[1]{\mathbb{1}\{#1\}}

\newcommand{\set}[1]{\lbrace #1\rbrace}
\newcommand{\tup}[1]{\left(#1 \right)}

\newcommand{\contrfoo}{h}
\newcommand{\conf}{\delta}

\newcommand{\conferror}{\varepsilon}
\newcommand{\prob}{\mathbb{P}}
\newcommand{\expe}{\mathbb{E}}

\newcommand{\sproc}{\mathcal{P}}

\newcommand{\bernoulli}{\mathrm{Bernoulli}}
 \newcommand{\aQ}{\mathcal{Q}}
  \newcommand{\aU}{\mathcal{U}}
   
 \newcommand{\aW}{\mathcal{W}}
\newcommand{\aX}{\mathcal{X}}

\newcommand{\rW}{W}
\newcommand{\rX}{X}
\newcommand{\rY}{Y}

\newcommand{\cU}{u}

\newcommand{\cW}{w}
\newcommand{\cX}{x}
\newcommand{\cY}{y}

\newcommand{\aS}{\mathcal{S}}
\newcommand{\aO}{\mathcal{O}}
\newcommand{\aH}{\mathcal{Q}}

\newcommand{\aF}{\aW}
\newcommand{\aFF}{\aU}

\newcommand{\rS}{S}
\newcommand{\rO}{W}

\newcommand{\cS}{s}
\newcommand{\cO}{w}

\newcommand{\cF}{\cW}
\newcommand{\cFF}{\cU}

\newcommand{\obs}{\ell}

\newcommand{\pr}{\rho}

 \newcommand{\counter}{c}

\newcommand{\monitor}{\mathcal{A}}

\newcommand{\verdict}{\Lambda}
\newcommand{\dom}{\mathit{Dom}}

\newcommand{\depends}{\ensuremath{\mathit{Dep}}}
\newcommand{\transrel}{T}

\newcommand{\varsym}{\Sigma}

\newcommand{\var}{\nu}

\newcommand{\iprop}{\varphi}

\newcommand{\mc}{\mathcal{M}}
\newcommand{\pomc}{\mathcal{H}}

\newcommand{\Q}{\aH}

\newcommand{\trm}{M}
\newcommand{\paths}[1]{\mathit{Paths}(#1)}
\newcommand{\initd}{\lambda}

\newcommand{\init}{\mathsf{init}}

\newcommand{\st}{\pi}
\newcommand{\outpaths}[1]{\mathit{OutPaths}(#1)}
\newcommand{\rout}{W}

\newcommand{\pomclending}{\pomc_{\mathsf{po\mhyphen lend}}}
\newcommand{\mclending}{\mc_{\mathsf{lend}}}

\newcommand{\cavcifoo}{\mathrm{CI}_{\mathrm{MC}}}

\newcommand{\cavconferrors}{\conferror_{\mathrm{MC}}^p}
\newcommand{\cavconferroru}{\conferror_{\mathrm{MC}}^u}

\newcommand{\cavvar}{\var}
\newcommand{\caviprop}{\varphi}

\newcommand{\caviproc}{\mc}
\newcommand{\cavsproc}{\sproc_{\mathrm{MC}}}

\newcommand{\pse}{PSE\xspace}

\newcommand{\algfreqdivfree}{\texttt{MC-MonitorDivFree}\xspace}
\newcommand{\algfreq}{\texttt{MC-Monitor}\xspace}

\newcommand{\oldx}{z}

\newcommand{\rvvar}{\var}
\newcommand{\rviprop}{\varphi}

\newcommand{\rvcifoo}{\mathrm{CI}_{\mathrm{POMC}}}
\newcommand{\rvconferrors}{\conferror_{\mathrm{POMC}}^p}
\newcommand{\rvconferroru}{\conferror_{\mathrm{POMC}}^u}
\newcommand{\rviproc}{\pomc}
\newcommand{\rvsproc}{\sproc_{\mathrm{POMC}}}

\newcommand{\vbse}{BSE\xspace}
\newcommand{\quanvbse}{\vbse}
\newcommand{\taumix}{\tau_{\mathrm{mix}}}

\newcommand{\specset}{\aW}

\DeclareSIUnit{\microsecond}{\SIUnitSymbolMicro s}

\title{Monitoring of Static Fairness}
\author[1]{Thomas A. Henzinger}
\author[1]{Mahyar Karimi}
\author[1]{Konstantin Kueffner}
\author[2]{Kaushik Mallik}
\affil[1]{ISTA, Klosterneuburg, Austria}
\affil[2]{IMDEA Software Institute, Pozuelo de Alarcón, Spain}
\date{}

\begin{document}

\maketitle

\begin{abstract}
	Machine-learned systems are in widespread use for making decisions about humans, and it is important that they are \emph{fair}, i.e., not biased against individuals based on sensitive attributes.
		We present a general framework of runtime verification of algorithmic fairness for systems whose models are unknown, but are assumed to have a Markov chain structure, with or without full observation of the state space.
		We introduce a specification language that can model many common algorithmic fairness properties, such as demographic parity, equal opportunity, and social burden.
		We build monitors that observe a long sequence of events as generated by a given system, and output, after each observation, a quantitative estimate of how fair or biased the system was on that run until that point in time.
		The estimate is proven to be correct modulo a variable error bound and a given confidence level, where the error bound gets tighter as the observed sequence gets longer.
		We present two categories of monitoring algorithms, namely ones with a uniform error bound across all time points, and ones with weaker non-uniform, pointwise error bounds at different time points.
        Our monitoring algorithms use statistical tools that are adapted to suit the dynamic requirements of monitoring and the special needs of the fairness specifications.
		Using a prototype implementation, we show how we can monitor if a bank is fair in giving loans to applicants from different social backgrounds, and if a college is fair in admitting students while maintaining a reasonable financial burden on the society.
		In these experiments, our monitors took less than a millisecond to update their verdicts after each observation.
\end{abstract}

%!TEX root=main.tex

\section{Introduction}

Runtime verification complements traditional static verification techniques, by offering lightweight solutions for checking properties based on a single, possibly long execution trace of a given system \cite{bartocci2018lectures}.
We present new runtime verification techniques for the problem of bias detection in decision-making software.
The use of software for making critical decisions about humans is a growing trend; example areas include judiciary \cite{chouldechova2017fair,dressel2018accuracy}, policing \cite{ensign2018runaway,lum2016predict}, banking \cite{liu2018delayed}, etc.
It is  important that these software systems are unbiased towards the protected attributes of humans, like gender, ethnicity, etc.
However, they have often shown biases in their decisions in the past \cite{dressel2018accuracy,lahoti2019ifair,obermeyer2019dissecting,scheuerman2019computers,seyyed2020chexclusion}.
While there are many approaches for mitigating biases before deployment \cite{dressel2018accuracy,lahoti2019ifair,obermeyer2019dissecting,scheuerman2019computers,seyyed2020chexclusion}, recent runtime verification approaches \cite{albarghouthi2019fairness,henzinger2023runtime} offer a new complementary tool to oversee \emph{algorithmic fairness} in AI and machine-learned decision makers during deployment.

To verify algorithmic fairness at runtime, the given decision-maker is treated as a \emph{generator} of events with an unknown model.
The goal is to algorithmically design lightweight but rigorous \emph{runtime monitors} against quantitative formal specifications. 
The monitors observe  a long stream of events and, after each observation, output a quantitative, statistically sound estimate of how fair or biased the generator was until that point in time.
While the existing approaches \cite{albarghouthi2019fairness,henzinger2023runtime} considered only static decision-makers whose inputs and outputs are fully observable, we present monitors for when the monitored stochastic processes are Markov chains whose state spaces are either fully or partially observable.

Monitoring algorithmic fairness involves on-the-fly statistical estimations, a feature that has not been well-explored in the traditional runtime verification literature.
As far as the algorithmic fairness literature is concerned, the existing works are mostly \emph{model-based}, and either minimize decision biases of machine-learned systems at \emph{design-time} (i.e., pre-processing) \cite{kamiran2012data,zemel2013learning,berk2017convex,zafar2019fairness}, or verify their absence at \emph{inspection-time} (i.e., post-processing) \cite{hardt2016equality}.
In contrast, we verify algorithmic fairness at \emph{runtime}, and do not require an explicit model of the generator.
On one hand, the model-independence makes the monitors trustworthy, and on the other hand, it complements the existing model-based static analyses and design techniques, which are often insufficient due to partially unknown or imprecise models of systems in real-world environments.

We first present monitoring algorithms for when the monitored systems are modeled using partially observed Markov chains (POMC) with unknown transition probabilities, and the properties are specified as the so-called bounded specification expressions (BSE) that are able to express many common algorithmic fairness properties from the literature, like demographic parity \cite{dwork2012fairness}, equal opportunity \cite{hardt2016equality}, and disparate impact \cite{feldman2015certifying}.
The difficulty of monitoring BSEs on POMCs comes from the fact that a random \emph{observation} sequence that is visible to the monitor may not follow a Markovian pattern, even though the underlying \emph{state} sequence is Markovian.
We argue that this makes it impossible to establish if the given BSE is fulfilled or violated based on one single observation trace of the system.
To circumvent this, we propose to assume that the POMC starts in the stationary distribution, which in turn guarantees a certain uniformity in how the observations follow each other.
We argue that the stationarity assumption is fulfilled whenever the system has been running for a long time, which is suitable for long term monitoring of fairness properties.
With the help of a few additional standard assumptions on the POMC, like aperiodicity and the knowledge of a bound on the mixing time, we can compute PAC estimates on the limiting value of the given BSE over the distribution of all runs of the system from a single monitored observation sequence.

While POMC and BSE provide us a very general and unifying framework for building the foundations of monitors, we also consider the special case when the POMCs have fully observable state spaces, i.e., they are Markov chains (MC), and the BSE specifications are restricted to a special fragment called PSE that is able to only express arithmetics over the unknown transition probabilities.
% Although PSEs are restricted classes of BSEs, they still suffice for expressing most existing fairness properties from the literature, like demographic parity and equal opportunity, but fail to capture properties that depend on 
For this special class of problems, we present monitors that require significantly less assumptions on the monitored system and yet are significantly more accurate. 

% For both classes of problems, we propose two different classes of monitors with different strengths of statistical correctness guarantees.
% The first class of monitors are called \emph{sound}, and they guarantee that 

For both classes of problems, POMCs with BSEs and MCs with PSEs, the basic schemes of the monitoring algorithms are similar.
In both cases, the monitor observes one long random execution sequence from the generator, and after each new observation outputs an updated PAC-style estimate of the value of the given specification.
Each PAC estimate at time $t$ consists of two parts, namely a real interval $[l_t,u_t]$ and a probability value $p_t\in [0,1]$.
We offer two separate families of monitoring algorithms, and they cover two different correctness interpretation of the monitors' outputs $[l_t,u_t]$ and $p_t$:
First, the pointwise sound monitors guarantee that at every $t$, the true fairness value lies within $[l_t,u_t]$ with probability at least $p_t$.
Second, the uniformly sound monitors guarantee that with probability at least $p_t$, for every $t$---including the future time points---the true fairness value lies within $[l_t,u_t]$.
It is easy to show that, uniformly sound monitors are more conservative and their intervals would always contain the intervals generated by the pointwise sound monitors.
In other words, every uniformly sound monitor is also a pointwise sound monitor, but not the other way round.
Our monitoring algorithms combine statistical tools to design pointwise sound and uniformly sound monitors for POMCs with BSEs and MCs with PSEs.

In short, our contributions are as follows:
\begin{enumerate}
    \item \textbf{Monitoring POMCs with BSEs:} 
    We propose a general framework for monitoring algorithmic fairness of unknown AI decision-makers with partially observable state space.
    \item \textbf{Monitoring MCs with PSEs:}
    We also consider the special case when the underlying state space of the monitored system is fully observable. 
    This special case is practically relevant and provides tighter probabilistic correctness bounds.
    \item \textbf{Pointwise sound monitoring algorithms:}
    We formalize pointwise soundness as a kind of (probabilistic) correctness requirement of monitors, where the monitors' outputs must be correct at each step with high probability.
    \item \textbf{Uniformly sound monitoring algorithms:}
    In addition to pointwise soundness, we formalize uniform soundness as a second type of correctness requirement, where the monitors' outputs must be correct at all time on the entire trace with high probability.
    \item \textbf{Empirical evaluations:}
    We empirically demonstrate the usefulness of our monitors on a canonical examples from the fairness literature such as the bank loan and the admission example from D'Amour~\cite{damour2020fairness}.
\end{enumerate}

This paper unifies our previous papers that cover Item~1 and Item~2 but exclusively for pointwise sound monitors (Item~3).
We additionally extend our previous results to the novel class of uniformly sound monitors (Item~4). 
Our uniformly sound monitors utilise the powerful martingale-based tools developed by Howard~\cite{howard2021time} to compute tight interval estimates. 
Furthermore, all the proofs that were omitted from our previous papers are now included.

\subsection{Motivating Examples}
\label{sec:motivating examples}

We first present two real-world examples from the algorithmic fairness literature to motivate the problem; these examples will later be used to illustrate the technical developments.

\smallskip
\noindent\textbf{The lending problem \cite{liu2018delayed}:}
Suppose a bank lends money to individuals based on certain attributes, like credit score, age group, etc.
The bank wants to maximize profit by lending money to only those who will repay the loan in time---called the ``true individuals.''
There is a sensitive attribute (e.g., ethnicity) classifying the population into two groups $g$ and $\overline{g}$.
The bank will be considered fair (in lending money) if its lending policy is independent of an individual's membership in $g$ or $\overline{g}$.
Several \emph{group fairness} metrics from the literature are relevant in this context.
\emph{Disparate impact} \cite{feldman2015certifying} quantifies the \emph{ratio} of the probability of an individual from $g$ getting the loan to the probability of an individual from $\overline{g}$ getting the loan, which should be close to $1$ for the bank to be considered fair.
\emph{Demographic parity} \cite{dwork2012fairness} quantifies the \emph{difference} between the probability of an individual from $g$ getting the loan and the probability of an individual from $\overline{g}$ getting the loan, which should be close to $0$ for the bank to be considered fair.
\emph{Equal opportunity} \cite{hardt2016equality} quantifies the \emph{difference} between the probability of a \emph{true} individual from $g$ getting the loan and the probability of a \emph{true} individual from $\overline{g}$ getting the loan, which should be close to $0$ for the bank to be considered fair.
A discussion on the relative merit of various different algorithmic fairness notions is out of scope of this paper, but can be found in the literature \cite{wachter2020bias,kleinberg2016inherent,corbett2017algorithmic,dwork2018individual}.
We show how we can monitor whether a given group fairness criteria is fulfilled by the bank, by observing a sequence of lending decisions.

\smallskip
\noindent\textbf{The college admission problem \cite{milli2019social}:}
Consider a college that announces a cutoff of grades for admitting students through an entrance examination.
Based on the merit, every truly qualified student belongs to group $g$, and the rest to group $\overline{g}$.
Knowing the cutoff, every student can choose to invest a sum of money---proportional to the gap between the cutoff and their true merit---to be able to reach the cutoff, e.g., by taking private tuition classes.
On the other hand, the college's utility is in minimizing admission of students from $\overline{g}$, which can be accomplished by raising the cutoff to a level that is too expensive to be achieved by the students from $\overline{g}$ and yet easy to be achieved by the students from $g$.
The \emph{social burden} associated to the college's cutoff choice is the expected expense of every student from $g$, which should be close to $0$ for the college to be considered fair (towards the society).
We show how we can monitor the social burden, by observing a sequence of investment decisions made by the students from $g$.

\subsection{Related Work}
There has been a plethora of work on algorithmic fairness from the machine learning standpoint  \cite{mehrabi2021survey,dwork2012fairness,hardt2016equality,kusner2017counterfactual,kearns2018preventing,sharifi2019average,bellamy2019ai,wexler2019if,bird2020fairlearn,zemel2013learning,jagielski2019differentially,konstantinov2022fairness}.
In general, these works improve algorithmic fairness through de-biasing the training dataset (pre-processing), or through incentivizing the learning algorithm to make fair decisions (in-processing), or through eliminating biases from the output of the machine-learned model (post-processing).
All of these are interventions in the design of the system, whereas our monitors treat the system as already deployed.

Recently, formal methods-inspired techniques have been used to guarantee algorithmic fairness through the verification of a learned model  \cite{albarghouthi2017fairsquare,bastani2019probabilistic,sun2021probabilistic,ghosh2020justicia,meyer2021certifying}, and enforcement of robustness \cite{john2020verifying,balunovic2021fair,ghosh2021algorithmic}.
All of these works verify or enforce algorithmic fairness \emph{statically} on all runs of the system with high probability.
This requires certain knowledge about the system model, which may not be always available. 
Our runtime monitor dynamically verifies whether the current run of an opaque system is fair.

Our monitors are closely related to the seminal work of Albarghouthi et al.~\cite{albarghouthi2019fairness}, where the authors build a programming framework that allows runtime monitoring of algorithmic fairness properties on programs.
Their monitor evaluates the algorithmic fairness of repeated ``single-shot'' decisions made by functions on a sequence of samples drawn from an underlying unknown but fixed distribution, which is a special case of our more general Markov chain  model of the generator.
Moreover, we argue and empirically show in Sec.~\ref{subsec:comparison} that our approach produces significantly tighter statistical estimates than their approach on most PSEs.
On the flip side, their specification language is more expressive, e.g., it is capable of specifying individual fairness criteria \cite{dwork2012fairness}.
We only consider group fairness, and monitors for individual fairness will independently appear in separate works \cite{gupta2025monitoring}.
Also, they allow logical operators (like boolean connectives) in their specification language.
However, we obtain tighter statistical estimates for the core arithmetic part of algorithmic fairness properties (through PSEs), and point out that we can deal with logical operators just like they do in a straightforward manner.

Shortly after our first paper on pointwise sound monitors for MCs with PSEs~\cite{henzinger2023monitoring}, we published a separate work for monitoring long-run fairness in sequential decision making problems, where the feature distribution of the population may dynamically change due to the actions of the individuals \cite{henzinger2023runtime}.
Although this other work generalizes our current paper in some aspects (support for dynamic changes in the model), it only allows sequential decision making models (instead of Markov chains), and only offers pointwise sound monitors and lacks support for uniformly sound monitors.
We also developed runtime active intervention tools for \emph{enforcing fairness} by altering decisions of the AI agent if needed~\cite{cano2025fairness}.
In contrast, the monitors presented in the current paper are passive entities that can only raise warnings when fairness is violated.
A limitation of our fairness enforcement tools is that they are applicable to sequential decision-makers, and not probabilistic generators such as POMCs and MCs.
Furthermore, they require precise knowledge of the probability distributions over the inputs and outputs of the sequential decision-making agent.

Traditional runtime verification techniques support mainly temporal properties and employ finite automata-based monitors 
\cite{stoller2011runtime,junges2021runtime,faymonville2017real,maler2004monitoring,donze2010robust,bartocci2018specification,baier2003ctmc}.
In contrast, runtime verification of algorithmic fairness requires checking statistical properties, which is beyond the limit of what automata-based monitors can accomplish.
Although there are some works on quantitative runtime verification using richer types of monitors (with counters/registers like us) \cite{finkbeiner2002collecting,henzinger2020monitorability,otop2019quantitative,henzinger2021quantitative}, the considered specifications usually do not extend to statistical properties such as algorithmic fairness.
One exception is the work by Ferr{\`e}re et al.~\cite{ferrere2019monitoring}, which monitors certain statistical properties, like mode and median of a given sequence of events. 
Firstly, they do not consider algorithmic fairness properties.
Secondly, their monitors' outputs are correct only as the length of the observed sequence approaches infinity (asymptotic guarantee), whereas our monitors' outputs are \emph{always} correct with high confidence (finite-sample guarantee), and the precision gets better for longer sequences.

Although our work uses similar tools as used in statistical verification \cite{ashok2019pac,younes2002probabilistic,clarke2011statistical,david2013optimizing,agha2018survey}, the goals are different. 
In traditional statistical verification, the system's runs are chosen probabilistically, and it is verified if any run of the system satisfies a boolean property with a  certain probability.
For us, the run is given as input to the monitor, and it is this run that is verified against a quantitative algorithmic fairness property with statistical error bounds.
To the best of our knowledge, existing works on statistical verification do not consider algorithmic fairness properties.

%!TEX root=main.tex

\section{Preliminaries}

\subsection{Basic Notation}
Let $\NN$ be the set of natural numbers, $\pNN$ be the set of natural numbers excluding zero, and $\RN$ be the set of real numbers. For a given subset $R\subseteq \RN$, we define $\inter{R}\coloneqq \{[a,b]\subseteq R\mid a,b\in R: a< b\}$ as the set of all real intervals over $R$. Let $a,b\in \NN$ such that $a<b$.
We define $[a;b]\coloneqq \{a, a+1,\dots b\}$ as the interval from $a$ to $b$ over the natural numbers and as a shorthand we will use $[b]\coloneqq [1;b]$.

\vspace{0.5em}
\noindent\textbf{Sequences:}
Let $\aF$ be a given alphabet. 
A word of length $n$ over $\aF$ is a sequence of characters $\seq{\cF}\coloneqq (\cF_1, \dots, \cF_n) \in \aF^n$.
We will write $|\seq{\cF}|$ to denote the length of the word $\seq{\cF}$, where the length can be either finite or infinite. 
We denote the set of all finite and infinite words over $\aF$ as $\aF^*$ and $\aF^{\omega}$, respectively, and we denote the set of all words as $\aF^{\infty}\coloneqq \aF^*\cup\aF^{\omega}$. 
% For all words $\seq{\cF}\in\aF^{\infty}$, we denote its length as $|\seq{\cF}|$, i.e.\ let $|\seq{\cF}|\coloneqq n$, if there exists an $n$ such that $\seq{\cF}\in\aF^n$ and let $|\seq{\cF}|\coloneqq \infty$ otherwise. 
For every $\seq{\cF}\in\aF^{\infty}$ and every $m<n<|\seq{\cF}|$, we will write $\seq{\cF}_{m:n}$ to denote the infix $\cO_m\ldots \cO_n$ of $\seq{\cF}$ from $m$ to $n$, and will use the shorthand $\seq{\cF}_n$ in place of $\seq{\cF}_{1:n}$, called the prefix of length $n$, where $\seq{\cF}_0$ will denote the empty word.
% \KM{Do we need the cases $n<0$ and $n>|\seq{\cF}|$?}
% \old{
% For every $\seq{\cF}\in\aF^{\infty}$ and every $n\in \IN$, if $0<n\leq|\seq{\cF}|$ let $\seq{\cF}_n\coloneqq \cF_1, \dots \cF_{n}$ be the finite prefix of length $n$, if $n\leq 0$ $\seq{\cF}_n$ is the empty word, and if $n>|\seq{\cF}|$ then $\seq{\cF}_n=\seq{\cF}$.}
% Moreover, for $n, m\in \NN$ s.t.\ $m\leq n\leq |\seq{\cF}|$ let $\seq{\cF}_{m:n}\coloneqq \cF_{m}\dots \cF_{n}$ be the infix of $\seq{\cF}$ from $m$ to $n$.
Let $\aF$ and $\aFF$ be a given pair of alphabets and $f\colon \aF\to\aFF$ be a mapping.
We will lift $f$ to map every word $\seq{\cF}$ over $\aF$ to an equally long word $\cFF$ over $\aFF$ by writing $\cFF = f(\cF_1)f(\cF_2)\ldots$.
% Given any two finite words $\seq{\cF},\seq{\cF}'\in\aF^{*}$ of lengths $m$ and $n$, respectively, we define the counting function
% \begin{align*}
%     \counting{\seq{\cF}'}(\seq{\cF})\coloneqq \sum_{i=1}^{n-m} \indi{\seq{\cF}_{i:i+m}=\seq{\cF}'}.
% \end{align*}

% \KK{Function evals}
% \vspace{0.5em}
% \noindent\textbf{Probability:}
% Let $\aF$ be some set, then we define $\distr(\aF)$ as the set of all probability distributions over $\aF$.
% \KK{Sub-Gaussian, Sub-Exponential}

\subsection{Markov Chains, with full or partial state observation}
We use partially observed Markov chains (POMC) as sequential randomized generators of events.

\paragraph{Syntax.}
A POMC is a tuple $\pomc\coloneqq \tup{\aS,\trm,\initd,\aO,\obs}$, where
$\aS = \mathbb{N}^+$ is a countable set of states,
$\trm$ is a stochastic matrix of dimension $|\aS|\times |\aS|$, called the \emph{transition probability matrix},
$\initd$ is a probability distribution over $\aS$ representing the \emph{initial state distribution},
$\aO$ is a countable set of \emph{observation labels}, and
$\obs\colon \aS\to \aO$ is called the \emph{labeling function} mapping every state to an observation.
% The observation labels $\aO$ may contain the empty word symbol $\varepsilon$, which will make our notation simpler in some cases.
%All POMCs in this paper are time-homogeneous, i.e., their transition probabilities do not vary over time. 

\paragraph{Semantics.}
Every POMC $\pomc$ induces a probability measure $\prob_{\initd}^\pomc(\cdot)$ over the generated state and observation sequences.
For every finite state sequence $\seq{\cS} = \cS_1\cS_2\ldots \cS_t \in \aS^{*}$, the probability that $\seq{\cS}$ is generated by $\pomc$ is given by $\prob_{\initd}^\pomc(\seq{\cS}) = \initd_{\cS_1}\cdot\prod_{i=1}^{t-1} \trm_{\cS_i\cS_{i+1}}$.
Every finite state sequence $\seq{\cS}\in \aS^*$ for which $\prob_{\initd}^\pomc(\seq{\cS})>0$ is called a finite \emph{internal path} of $\pomc$; we omit $\pomc$ and $\initd$ if it is clear from the context.
It is known that the probability measure $\prob_{\initd}^\pomc(\cdot)$ can be extended to the set of infinite paths.
The set of every infinite internal path will be denoted as $\paths{\pomc}$. 

%Every finite internal path $\seq{\cS}$ can be extended to a set of infinite internal paths, which is called the cylinder set induced by $\seq{\cS}$, and is defined as $\cyl{\seq{\cS}}\coloneqq \set{\seq{r}\in \aS^\omega\mid \seq{\cS}\pref\seq{r}}$.
%The probability measure $\prob_\pomc(\cdot)$ on finite internal paths induces a pre-measure on the respective cylinder sets, which can be extended to a unique measure on the infinite internal paths by means of the Carath\'eodory's extension theorem \cite[pp.~757]{baier2008principles}.
%The probability measure on the set of infinite internal paths is also denoted using $\prob_\pomc(\cdot)$.
%
\paragraph{Observed Sequences.}
An external observer can only observe the observable part of an internal path of a POMC.
Given an internal path $\seq{\cS} = \cS_1\cS_2\ldots \in \aS^\infty$, we write $\obs(\seq{\cS})$ to denote the observation sequence $\obs(\cS_1)\obs(\cS_2)\ldots\in \aO^\infty$.
For a set of internal paths $S\subseteq \aS^\infty$, we write $\obs(S)$ to denote the respective set of observation sequences $\set{\seq{\cO}\in \aO^\infty\mid \exists \seq{\cS}\;.\;\seq{\cO}=\obs(\seq{\cS})}$.
An observation sequence $\seq{\cO}\in \aO^\infty$ is called an \emph{observed} path (of $\pomc$) if there exists an internal path $\seq{\cS}$ for which $\obs(\seq{\cS})=\seq{\cO}$.
As before, we use $\outpaths{\pomc}$ to denote the set of every finite observed path.
For every observed path $\seq{\cO}$, let $\obs^{-1}(\seq{\cO})$ be the set $S$ of every internal path such that $\obs(S) = \set{\seq{\cO}}$.
We lift the probability measure $\prob_{\initd}^\pomc(\cdot)$ to the set of observed paths in the usual way, namely, for every measurable set $\aU$ of observed paths, $\prob_{\initd}^\pomc(\aU) = \prob_{\initd}^\pomc\left( \set{\seq{\cS}\in \paths{\pomc}\mid \exists \seq{\cO}\in \aU\;.\;\seq{\cS}\in \obs^{-1}(\set{\seq{\cO}})}\right)$.

\paragraph{Markov Chains.}
A POMC $\mc\coloneqq\tup{\aS,\trm,\initd,\aO,\obs}$ is called a (fully observed) \emph{Markov chain} (MC) if $\obs$ is a bijection.
In the special case when the state and observation labels coincide, i.e., $\aS=\aO$ and $\obs\colon q\mapsto q$ for every $q\in\aS$, we will drop the observations and observation function and represent the MC as the tuple $\mc\coloneqq \tup{\aS,\trm,\initd}$.
Clearly, for MCs, the concept of internal and observed paths coincide, and they will be referred to simply as \emph{paths} and the set of all infinite paths of an MC $\mc$ will be denoted as $\paths{\mc}$.

\begin{example}
\label{ex:lending POMC}
As a running example, we introduce a POMC $\pomclending$ and its MC variant $\mclending$ that model the sequential interaction between a bank and loan applicants with varying levels of observability of states.
Suppose there is a population of loan applicants, where each applicant has a credit score between $1$ and $4$, and belongs to either group $A$ or group $B$ based on some of their sensitive attributes (like gender or race).
At every step, the bank receives loan application from one applicant, and, based on some unknown (but non-time-varying) criteria, decides whether to grant loan or reject the application.
%We want to monitor, for example, the difference between loan acceptance probabilities for people belonging to the two groups.

The underlying state transition diagram is shown in Fig.~\ref{fig:illustrative POMC model of loan example}, and the observation labels are explained in the caption.
A possible internal path is $S(A,1)NS(A,4)YS(B,3)N\ldots$, whose corresponding observed path in $\pomclending$ is $ANAYBN\ldots$ and the same in $\mclending$ is $(A,1)N(A,4)Y(B,3)N\ldots$.
In other words, in $\pomclending$, the credit scores of the individuals is hidden, whereas in $\mclending$ the credit scores are observable.
In our experiments, we use a more realistic model of the POMC with richer set of features for the individuals.
\end{example}

\begin{figure}
	\vspace{-2.5cm}
	\begin{tikzpicture}[scale=0.4,node distance=0.5cm,trim left=-6cm]
		\tikzstyle{arrow} = [thick,->,>=stealth]
		\tikzstyle{every node} = [font=\tiny,inner sep=0,outer sep=0]
		
		\node[state]	(s)		at	(0,0)	{$S$};
		\node[state]	(a1)		[below left=0.75 cm of s]	{$(A,4)$};
		\node[state]	(a2)		[left=of a1]		{$(A,3)$};
		\node[state]	(a3)		[left=of a2]		{$(A,2)$};
		\node[state]	(a4)		[left=of a3]		{$(A,1)$};
		\node[state]	(b1)		[below right=0.75 cm of s]	{$(B,1)$};
		\node[state]	(b2)		[right=of b1]		{$(B,2)$};
		\node[state]	(b3)		[right=of b2]		{$(B,3)$};
		\node[state]	(b4)		[right=of b3]		{$(B,4)$};
		\node[state]	(y)		[below =of a4]		{$Y$};
		\node[state]	(n)		[below =of b4]		{$N$};
		
%		\begin{scope}[on background layer]
%			\draw[fill=gray!15!white]	($(a4.south west)-(0.3,0.3)$)	rectangle	($(a1.north east)+(0.3,0.3)$);
%			\draw[fill=gray!15!white]	($(b1.south west)-(0.3,0.3)$)	rectangle	($(b4.north east)+(0.3,0.3)$);
%			\draw[fill=gray!15!white]	($(s.south west)-(0.3,0.3)$)		rectangle	($(s.north east)+(0.3,0.3)$);
%			\draw[fill=gray!15!white]	($(y.south west)-(0.3,0.3)$)		rectangle	($(y.north east)+(0.3,0.3)$);
%			\draw[fill=gray!15!white]	($(n.south west)-(0.3,0.3)$)		rectangle	($(n.north east)+(0.3,0.3)$);
%		\end{scope}
		
%		\node	at	($(a4.north)+(0,0.4)$)		{$A$};
%		\node	at	($(b4.north)+(0,0.4)$)		{$B$};
%		\node	at	($(s.north)+(0,0.4)$)		{$S$};
%		\node	at	($(y.south)+(0,-0.4)$)		{$Y$};
%		\node	at	($(n.south)+(0,-0.4)$)		{$N$};
		
		\path[->]
			(s)		edge		(a1)
					edge[bend right]		(a2)
					edge	[bend right]	(a3)
					edge	[bend right]	(a4)
					edge		(b1)
					edge[bend left]		(b2)
					edge[bend left]		(b3)
					edge[bend left]		(b4)
			(a1)		edge	[bend left]	(y)
					edge	[bend right]	(n)
			(a2)		edge	[bend left]	(y)
					edge	[bend right]	(n)
			(a3)		edge		(y)
					edge[bend right]		(n)
			(a4)		edge		(y)
					edge	[bend right]	(n)
			(b1)		edge	[bend left]	(y)
					edge	[bend right]	(n)
			(b2)		edge	[bend left]	(y)
					edge	[bend right]	(n)
			(b3)		edge	[bend left]	(y)
					edge[bend right]		(n)
			(b4)		edge	[bend left]	(y)
					edge		(n)
			(y)		edge	[out=155,in=125,looseness=1.3]	(s)
			(n)		edge	[out=25,in=55,looseness=1.3]	(s);
	\end{tikzpicture}
	\vspace{-1cm}
	\caption{The POMCs modeling the sequential interaction between the bank and the loan applicants.
	The states $S$, $Y$, and $N$ respectively denote the start state, the event that the loan was granted (``$Y$'' stands for ``Yes''), and the event that the loan was rejected (``$N$'' stands for ``No'').
	Every state $(X,i)$, for $X\in \set{A,B}$ and $i\in \set{1,2,3,4}$, represents the group ($A$ or $B$) and the credit score $i$ of the current applicant.
	The labeling function is:
	$S\mapsto \varepsilon$, $Y\mapsto Y$, $N\mapsto N$, and for $i\in\set{1,2,3,4}$, for $\pomclending$, $(A,i)\mapsto A$ and $(B,i)\mapsto B$, whereas for $\mclending$, $(A,i)\mapsto (A,i)$ and $(B,i)\mapsto (B,i)$.
%	The states $S,Y,N$ are fully observable, i.e., their observation symbols coincide with their state symbols.
%	The middle states are partially observable, with every $(A,i)$ being assigned the observation $A$ and every $(B,i)$ being assigned the observation $B$.
%	The states with the same observation belong to the same shaded box.
	}
	\label{fig:illustrative POMC model of loan example}
\end{figure}
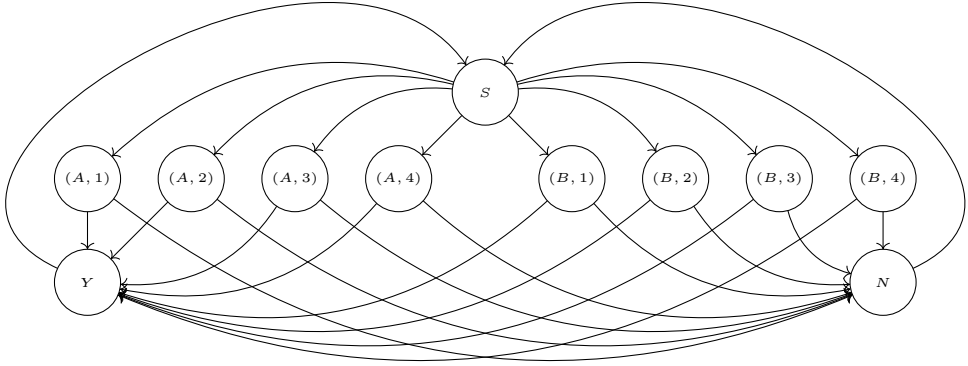

\subsection{Register Monitors}
Our register monitors are adapted from the polynomial monitors of Ferr\`ere et al.\ \cite{ferrere2018theory}, and were also used in our previous work (in a more general randomized form) \cite{henzinger2023monitoring}.

\paragraph{Syntax.}
Let $R$ be a finite set of integer variables called registers.
A function $v\colon R\to \mathbb{N}$ assigning concrete value to every register in $R$ is called a valuation of $R$.
Let $\mathbb{N}^R$ denote the set of all valuations of $R$.
Registers can be read and written according to relations in the signature $S=\langle 0,1,+,-,\times,\div,\leq \rangle$.
We consider two basic operations on registers:
\begin{itemize}[noitemsep,topsep=0pt]
	\item A \emph{test} is a conjunction of atomic formulas over $S$ and their negation;
	\item An \emph{update} is a mapping from variables to terms over $S$.
\end{itemize}
We use $\Phi(R)$ and $\Gamma(R)$ to respectively denote the set of tests and updates over $R$.
\emph{Counters} are special registers with a restricted signature $S=\langle 0,1,+,-,\leq \rangle$.

\begin{definition}[Register monitor]
	A register monitor is a tuple $\tup{\Sigma,\Lambda,R,v_{\mathsf{in}},f,\transrel}$ where 
	$\Sigma$ is a finite input alphabet, 
	$\verdict$ is an output alphabet, 
	$R$ is a finite set of registers,
	$v_{\mathsf{in}}\in \mathbb{N}^R$ is the initial valuation of the registers,
	$f\colon \mathbb{N}^R\to \verdict$ is an output function, and
	$\transrel\colon \Sigma\times\Phi(R)\to \Gamma(R)$ is the transition function such that
	for every $\sigma\in \Sigma$ and for every valuation $v\in \mathbb{N}^R$, there exists a unique $\phi\in \Phi(R)$ with $v\models\phi$ and $\transrel(\sigma,\phi)\in \Gamma(R)$.
\end{definition}

We refer to register monitors simply as monitors, and we fix the output alphabet $\Gamma$ as the set of every real interval.

\paragraph{Dynamics.}
A \emph{state} of a monitor $\monitor$ is a valuation of its registers $v\in \mathbb{N}^R$; the initial valuation $v_\init$ is the initial state.
The monitor $\monitor$ \emph{transitions} from state $v$ to another state $v'$ on input $\sigma\in \Sigma$ if there exists $\phi$ such that $v\models \phi$, there exists an update $\gamma=\transrel(\sigma,\phi)$, and if $v'$ maps every register $x$ to $v'(x)=v(\gamma(x))$.
The transition from $v$ to $v'$ on input $\sigma$ is written as $v\xrightarrow{\sigma} v'$.
A \emph{run} of $\monitor$ on a word $\cO_1\ldots \cO_t\in \Sigma^*$ is a sequence of transitions $v_1=v_\init\xrightarrow{\cO_1} v_2\xrightarrow{\cO_2}\ldots\xrightarrow{\cO_t} v_{t+1}$. 
\paragraph{Semantics.}
The \emph{semantics} of the monitor is the function $\sem{\monitor}\colon \Sigma^*\to \Lambda$ that maps every finite input word to the last output of the monitor on the respective run.
For instance, the semantics of $\monitor$ on the word $\seq{\cO}$ is $\sem{\monitor}(\seq{\cO}) = f(v_{t+1})$.
% An illustrative example of register monitors can be found in our earlier work \cite[Sec.~2.2]{henzinger2023monitoring}.

%!TEX root=main.tex

\section{Algorithmic Fairness as Statistical Properties}
\label{sec:properties}
We formalize algorithmic fairness properties as statistical properties. Statistical properties are defined as the limit of an arithmetic expressions consisting of empirical averages of atomic functions. We specify statistical properties using the Bounded Specification Expressions (\quanvbse) introduced by Henzinger et al.~\cite{henzinger2023monitoring2}.
In this section, we motivate the representation of algorithmic fairness properties as statistical properties, present the syntax of \quanvbse, and introduce a new path-based semantics. We highlight two useful fragments and show their applicability to existing fairness measures. In addition, we pove that our new path-based semantics generalises the existing model-based semantics by first showing their equivalence on a subclass of POMCs and then demonstrating that the path-based semantics is well-defined for a much more general class of POMCs.

\paragraph{Model-based perspective.}
The simplest example of fairness as a statistical property is a coin. Assume we want to assess the fairness of an infinite sequence of independent coin flips. 
We can quantify the fairness of the sequence by looking at the bias $p$ of our coin.
We assess the fairness of the coin by comparing the probability of obtaining heads with the probability of obtaining tails. One comparison measure is demographic parity, which considers the difference between these two probabilities, i.e., $p-(1-p)$.
This is the model-based perspective on fairness, which can be generalised to more complex stationary systems, specifically irreducible and positively recurrent POMCs. Example~\ref{ex:coin2} illustrates how this can be achieved by utilising the stationary distribution of such systems.

\begin{example}
\label{ex:coin}
Suppose there are two coins $A$ and $B$, where $A$ comes up heads with probability $0.9$ and $B$ comes up tails with probability $0.9$. 
At each step, either coin $A$ or coin $B$ is chosen uniformly at random and tossed. We can observe that the biases of the individual coins balance out in expectation. 
We can quantify the expected fairness value of the process using a version of demographic parity computed from the initial state, i.e., $0.5 (p_A-(1-p_A)) + 0.5 (p_B-(1-p_B))$.
% The expected fairness value, quantified using a version of demographic parity, at each iteration is $0.5 (p_A-(1-p_A)) + 0.5 (p_B-(1-p_B))$. This is the unique fairness value of this system, because the stationary distribution of the induced Markov chain, i.e., the occupancy measure, is $0.5$ for both coins. 
% \begin{center}
% \begin{tikzpicture}[shorten >=1pt, node distance=2cm and 2cm, on grid, auto]
%   \node[state] (s0) {$\star$};
%   \node[state] (s1) [right=of s0, yshift=1.5cm] {$A$};
%   \node[state] (s2) [right=of s0, yshift=-1.5cm] {$B$};
%   \node[state] (s3) [right=of s1, yshift=1cm] {$1$};
%   \node[state] (s4) [right=of s1, yshift=-1cm] {$0$};
%   \node[state] (s5) [right=of s2, yshift=1cm] {$1$};
%   \node[state] (s6) [right=of s2, yshift=-1cm] {$0$};
%   \node[state] (s7) [right=of s0, xshift=4cm] 
%   {$\star$};

%   \path[->]
%     (s0) edge node {0.5} (s1)
%          edge node {0.5} (s2)
%     (s1) edge node {0.9} (s3)
%          edge  node {0.1} (s4)
%     (s2) edge node {0.1} (s5)
%          edge  node {0.9} (s6)
%     (s3) edge   node {1} (s7)
%     (s4) edge  node {1} (s7)
%     (s5) edge  node {1} (s7)
%     (s6) edge  node {1} (s7);
% \end{tikzpicture}
% \end{center}
\end{example}
We used this notion of fairness in one of our previous papers~\cite{henzinger2023monitoring2}. This definition is sensible because the system admits a stationary distribution, i.e., the long run average of the time spent with each coin is equal to the initial random choice, i.e., $0.5$ for each coin. Therefore, for every stationary distribution $\st$ we have the unique fairness value $\st_A (p_A-(1-p_A)) + \st_A (p_B-(1-p_B))$. 
This fails for non-stationary systems. We illustrate this in the example below.
\begin{example}
\label{ex:coin2}
As before, we consider two coins $A$ and $B$, where $A$ comes up heads with probability $0.9$ and $B$ comes up tails with probability $0.9$. 
Different to before, we start by selecting either coin $A$ or coin $B$ with equal probability, after which we continuously toss the chosen coin. 
If we compute the expected fairness value as before, we obtain the same result, i.e., the system is fair in expectation. However, every possible run of this system will show a bias, i.e., $0.8$ in one scenario and $-0.8$ in the other. 

\end{example}

\paragraph{Path-based perspective.}
Example~\ref{ex:coin2} demonstrated that a system may be fair in expectation at the beginning, but its actual behaviour is severely unfair for every realisation.
We can address the problems with non-stationarity by adopting a path-based perspective. 
We propose to avoid quantifying fairness using the limit of the empirical mean. Formally, for the sequence of coin tosses $(\cO_t)_{t\in \pNN}$ we use 
\begin{align*}
   \hat{p}  \coloneqq  \lim_{t\to \infty} \frac{1}{t}\sum_{i=1}^t \cO_i 
\end{align*}
to compute demographic parity, i.e., $ \hat{p} - (1- \hat{p})$. This measure evaluates to $0.5$ in Example~\ref{ex:coin} and $0.8$ or $-0.8$ in Example~\ref{ex:coin2}.

\subsection{Bounded Specification Expressions}
We specify statistical properties using the rich bounded specification expressions~(BSE), which are interpreted over POMCs.
In this section, we introduce a path-based semantics that unifies and generalises the model-based semantics introduced in two previous works~\cite{henzinger2023monitoring,henzinger2023monitoring2}.

\paragraph{Syntax.}
Given a POMC $\tup{\aS,\trm,\initd,\aO,\obs}$, statistical properties use atomic functions of the form $\var\colon \aO^n\to [a,b]$ mapping finite words of a fixed length $n\in\pNN$ into a bounded interval over the reals $[a,b]\subseteq \RN$.
The specified word length $n$ for any atomic function $\var$ is called the \emph{arity} of $\var$. We define the set $\varsym$ as the set of all atomic functions. 
A \quanvbse contains arithmetic connectives to express complex value-based properties of an underlying POMC. The syntax is given as:
\begin{align}\label{equ:syntax property}
	\text{(\quanvbse)}\qquad\iprop &\Coloneqq \kappa \in \mathbb{R} \ | \ \var\in \Sigma \ | \  \iprop + \iprop  \ | \ \iprop\cdot\iprop  \ | \ 1\div\iprop  \ | \ (\iprop).
 \end{align} 
We note that atomic functions of arity $0$ are equivalent to real-valued constants $\kappa \in \mathbb{R}$. We have made the constant explicit to simplify notation.
We use $V_\iprop$ to denote the set of variables appearing in the expression $\iprop$.
The \emph{size} of an expression is the total number of operators ($+,-,\cdot,\div, \land, \neg$) in the expression.
Next, we present two examples to illustrate atomic functions.

\begin{example}
	Consider the MC $\mclending$ introduced in Ex.~\ref{ex:lending POMC}, and suppose we are interested in checking whether two individuals from different groups but with the same credit score were treated differently (one was granted a loan and the other one was not) in the last $10$ rounds.
	We can express this property using the boolean atomic function $\var_{10}$ such that for every $\seq{\cS} = \cS_1\ldots \cS_{10} \in\aS^{10}$, $\var_{10}(\seq{\cS}) = 1$ iff there exist $k,l\in [1;9]$ and $i\in\set{1,2,3,4}$ with $\cS_k=(A,i)$, $\cS_{l} = (B,i)$, and $\cS_{k+1}\neq \cS_{l+1}$ (i.e., one of $\set{\cS_{k+1},\cS_{l+1}}$ is $Y$ and the other one is $N$).
	Observe that the function $\var_{10}$ defined above is invalid for $\pomclending$, because the observation labels do not contain $(A,i)$-s and $(B,i)$-s.
\end{example}

\begin{example}
	Consider the POMC $\pomclending$ introduced in Ex.~\ref{ex:lending POMC}, and suppose we are interested in calculating the difference between the ratios of accepted individuals from the two groups in the last $10$ rounds, assuming we have seen at least one individual from either group.
	We can formalize this property using the atomic function $\var_{10}$ such that for every $\seq{\cO} = \cO_1\ldots \cO_{10} \in\aO^{10}$, 
    \begin{align*}
        \var_{10}(\seq{\cO}) = \left|\frac{\sum_{i=1}^9\indi{\cO_i=A,\cO_{i+1}=Y}}{\sum_{i=1}^9\indi{\cO_i=A}} - \frac{\sum_{i=1}^9\indi{\cO_i=B,\cO_{i+1}=Y}}{\sum_{i=1}^9\indi{\cO_i=B}}\right|.
    \end{align*}

\end{example}

\paragraph{Path-based semantics.}
The path-based semantics of a \quanvbse $\iprop$ is defined as follows. Let $\var_n\in \Sigma$ be an atomic function of arity $n$.
The infinite semantics of \quanvbse formulas is built using their finitary semantics stated below.
For every finite prefix $\seq{\cO}_t\in\aO^t$ of length $t\geq n$ from an arbitrary infinite word $\seq{\cO}\in\aO^\omega$, and $k\coloneqq n-1$:
\begin{equation}\label{eq:semantics}
\begin{aligned}
    &\sem{\kappa}(\seq{\cO}_t) = \kappa\\
    &\sem{\var_n}(\seq{\cO}_t) = \frac{1}{t-k}\sum_{i=1}^{t-k}\var_n(\cO_i\ldots \cO_{i+k})\\
    &\sem{\iprop+\psi}(\seq{\cO}_t) = \sem{\iprop}(\seq{\cO}_t) + \sem{\psi}(\seq{\cO}_t)\\
    &\sem{\iprop\cdot\psi}(\seq{\cO}_t) = \sem{\iprop}(\seq{\cO}_t) \cdot \sem{\psi}(\seq{\cO}_t)\\
    &\sem{1\div\psi}(\seq{\cO}_t) = 1 \div \sem{\psi}(\seq{\cO}_t) \quad (\text{provided } \sem{\psi}(\seq{\cO}_t)\neq 0)
\end{aligned}
\end{equation}
and the actual semantics of a \quanvbse $\iprop$ on the entire infinite word $\seq{\cO}$ is obtained by taking the limit over the length of the prefix, i.e.,
\begin{align}\label{eq:inf_semantics}
	&\sem{\iprop}(\seq{\cO}) = \lim_{t\to\infty} \; \sem{\iprop}(\seq{\cO}_t).
\end{align}
This limit does not exist in general. However, it exists for all POMCs whose the state space can be decomposed into a finite transient component and a finite number of irreducible and positively recurrent components, e.g., the system in Example~\ref{ex:coin2}. Under those conditions, the semantic value of a $\quanvbse$ is well-defined. 
This generalizes the model-based semantics introduced by Henzinger et al.~\cite{henzinger2023monitoring2}, which is restricted to irreducible and positively recurrent POMCs, e.g., Example~\ref{ex:coin}. 
% This improves the utility of \quanvbse, because it m applicable to a broader range of applications~\cite{kwiatkowska2012prism}.

% \KK{Shou we add this?}
% Assumption~\ref{ass:stationarity} is necessary for the equivalence of the semantics. 
% In the following, we present a simple example to demonstrate that the violation of Assumption~\ref{ass:stationarity} may lead to scenarios when $\sem{\iprop}(\pomc)$ differs from $\sem{\iprop}(\seq{\rO})$.

% \begin{example}\label{ex:two coins:one way switch}
% Consider the setting of Ex.~\ref{ex:two coins:uniform switch}, and suppose now only the initial selection of the coin happens uniformly at random but subsequently the same coin is used forever.
% If we consider the underlying POMC, both $\pi_a,\pi_b$ will be $0.5$, because the initial selection of the coin happens uniformly at random.
% However, after the first step the path-based semantics 
% \end{example}

\subsection{Useful Fragments}
We present useful fragments of the \quanvbse language. The fragments are characterized syntactically; however, we give an intuition about their semantic value.
Moreover, we demonstrate that some of those fragments can be used to express existing fairness measures.

\paragraph{Polynomial and division free.}
An expression is called \emph{division-free} if it does not contain the division operator ``$\div$''. 
An expression is called \emph{polynomial} if it is a weighted sum of monomials. A monomial is a product of powers of variables with integer exponents\footnote{Although monomials and polynomials usually only have positive exponents, we take the liberty of using the terminology even when negative exponents are present.}

\paragraph{Probabilistic specification expressions (PSEs).}
In our prior work \cite{henzinger2023monitoring}, we introduced PSEs to model algorithmic fairness properties of Markov chains with a fully observable state space.
PSEs are arithmetic expressions over atomic variables of the form $\pr(r|q)$, where $q$ and $r$ are states of the given Markov chain. The semantic value of such a variable equals the transition probability from $q$ to $r$. Therefore, the semantics of a PSE is the valuation of the expression obtained by plugging in the respective transition probabilities.
We can express PSEs using \quanvbse formulas as follows.
Every variable $\pr(r|q)$ appearing in a given PSE is replaced by the subformula 
\begin{align}
\label{eq:encoding}
    \pr(r|q) \coloneqq \frac{\indi{qr}}{\sum_{s\in \aS} \indi{qs}}
\end{align}
where $\indi{\cdot}$ is the indicator function, i.e., $\seq{\cO}_2\in\aO^2$, $\indi{xx'}(\seq{\cO}_2) = 1$ if $\seq{\cO}_2=xx'$. 
On the other hand, \quanvbse formulas are strictly more expressive than PSEs. For instance, unlike PSEs, they can express properties such as the probability of transitioning from one observation label to another, and the average number of times a given state is visited on any finite path of a Markov chain.

% \frac{\indi{qr}}{ \indi{qr}}
\paragraph{Probabilities of sequences.}
We consider a useful fragment that expresses the probability that a sequence from a given set $U\subseteq \aO^*$ of finite observation sequences appears at any point during the infinite observed path.
We assume that the length of every sequence in $U$ is uniformly bounded by some integer $n$.
Let $\overline{U}\subseteq \aO^n$ denote the set of extensions of sequences in $U$ up to length $n$, i.e., $\overline{U}\coloneqq\set{\seq{\cO}\in\aO^n\mid \exists \seq{u}\in U.\; \seq{u} \text{ is a prefix of } \seq{\cO}}$.
Then the desired property will be expressed simply using an atomic function with $\var\colon \aO^n\to \set{0,1}$ being the indicator function of the set $\overline{U}$, i.e., $\var(\seq{\cO})=1$ if and only if $\seq{\cO}\in \overline{U}$. 
For a set of finite words $U\subseteq \aO^*$, we introduce the shorthand notation $\pr(U)$ to denote the probability of seeing an observation from the set $U$ at any given point in time.
Furthermore, for a pair of sets of finite words $U,V\subseteq \aO^*$, we use the shorthand notation $\pr(U\mid V)$ to denote $\sfrac{\pr(UV)}{\pr(U)}$, which represents the conditional probability of seeing a word in $V$ after we have seen a word in $U$.

\paragraph{Fairness and \vbse.}
In the following examples, we demonstrate that the \vbse specification language is sufficiently expressive to cover an array of common fairness properties.
\begin{example}[Group fairness.]
\label{ex:fairness properties in lending POMC}
	Consider the setting in Ex.~\ref{ex:lending POMC}.
	We show how we can represent various group fairness properties using \quanvbse-s.
	Demographic parity \cite{dwork2012fairness} quantifies bias as the difference between the probabilities of individuals from the two demographic groups getting the loan, which can be expressed as $\pr(Y\mid A)-\pr(Y\mid B)$.
	Disparate impact \cite{feldman2015certifying} quantifies bias as the ratio between the probabilities of getting the loan across the two demographic groups, which can be expressed as $\pr(Y\mid A)\div \pr(Y\mid B)$.
\end{example}
\begin{example}[Social fairness.]\label{ex:social fairness in lending}
	Consider the setting in Ex.~\ref{ex:lending POMC}, except that now the credit score of each individual is observable along with their group memberships, i.e., each observation is a pair of the form $(X,i)$ with $X\in \set{A,B}$ and $i\in\set{1,2,3,4}$.
	There may be other non-sensitive features, such as age, which may be hidden.
	We use the social fairness property  \cite{henzinger2023runtime} quantified as the difference between the expected credit scores of the groups $A$ and $B$.
	To express this property, we use the unary atomic functions $\var_1^X\colon \Sigma\to \mathbb{N}$, for $X\in \set{A,B}$, such that $\var_1^X\colon (Y,i)\mapsto i$ if $Y=X$ and $0$ otherwise.
	The semantics of $\var_1^X$ is the expected credit score of group $X$ scaled by the probability of seeing an individual from group $X$.
	Then social fairness is given by the \quanvbse $\iprop = \frac{\var_1^A}{\pr(A)} - \frac{\var_1^B}{\pr(B)}$.
\end{example}
\begin{example}[Quantitative group fairness.]\label{ex:fair salary distribution in hiring}
	Consider a sequential hiring scenario where at each step the salary and a sensitive feature (like gender) of a new recruit are observed.
	We denote the pair of observations as $(X,i)$, where $X\in \set{A,B}$ represents the group information based on the sensitive feature and $i$ represents the salary.
	We can express the disparity in expected salary of the two groups in a similar manner as in Ex.~\ref{ex:social fairness in lending}.
	Define the unary functions $\var_1^X\colon \Sigma\to \mathbb{N}$, for $X\in \set{A,B}$, such that $\var_1^X\colon (Y,i)\mapsto i$ if $Y=X$ and is $0$ otherwise.
	The semantics of $\var_1^X$ is the expected salary of group $X$ scaled by the probability of seeing an individual from group $X$.
	Then the group fairness property is given by the \quanvbse $\iprop = \frac{\var_1^A}{\pr(A)} - \frac{\var_1^B}{\pr(B)}$.
\end{example}

\subsection{Equivalence to Model-based Semantics}
In this subsection, we show that the model-based semantics for both \vbse and PSEs, defined respectively by Henzinger et al.~\cite{henzinger2023monitoring2} and Henzinger et al.~\cite{henzinger2023monitoring} are equivalent to the new path-based semantics for the class of irreducible and positively recurrent POMCs.
This class is formally specified in Assumption~\ref{ass:stationarity}.

\begin{assumption}\label{ass:stationarity}
	We assume that the POMCs are irreducible and positively recurrent.
\end{assumption}

\paragraph{Model-based semantics.}
Under the model-based semantics introduced by Henzinger et al.~\cite{henzinger2023monitoring2}, the semantic value of an atomic function is equivalent to the expected value of the atomic function w.r.t.\ the POMC's stationary distribution.
Formally, the model-based semantics of a \vbse $\iprop$ is defined as follows.
For every atomic function $\var_n$ with arity $n\in \NN$ and a POMC $\pomc$ satisfying Assumption~\ref{ass:stationarity}, we have
\begin{align*}
	&\sem{\kappa}(\pomc) = \kappa\\
	&\sem{\var_n}(\pomc) = \expe_{\st}^{\pomc}(\var_n(\rO_1, \dots, \rO_n))\\
	&\sem{\iprop+\psi}(\pomc) = \sem{\iprop}(\pomc) + \sem{\psi}(\pomc)\\
	&\sem{\iprop\cdot\psi}(\pomc) = \sem{\iprop}(\pomc) \cdot \sem{\psi}(\pomc)\\
	&\sem{1\div\psi}(\pomc) = 1 \div \sem{\psi}(\pomc)	\qquad (\text{provided } \sem{\psi}(\pomc)\neq 0)
\end{align*}
where 
\begin{align*}
    \expe_{\st}^{\pomc}(\var_n(\rO_1, \dots, \rO_n)) \coloneqq \sum_{(\cS_1, \dots, \cS_n) \in \aS^n}\var_n(\obs(\cS_1), \dots, \obs(\cS_n)) \cdot \st(\cS_1) \cdot \prod_{i=1}^{n-1} \trm_{\cS_i, \cS_{i+1}} .
\end{align*}
% Next, we show the equivalence between the path- and model-based semantics of PSEs for MCs satisfying Assumption~\ref{ass:stationarity}. This is a direct consequence of Theorem~\ref{theorem:semantic_equivalence}.

\paragraph{Equivalence.}
We show the equivalence between the path- and model-based semantics 
of \vbse for POMCs satisfying Assumption~\ref{ass:stationarity}. This is a direct consequence of the ergodic theorem \cite{Norris_1997}.
First, the existence of the stationary distribution is guaranteed by the irreducibility and positive recurrence of the POMC. Second, the stationary distribution $\st$ of a POMC is a distribution over its states that remains constant over time, i.e., $\st = \trm\st$. Hence, the semantic value of the \vbse is time-invariant. 
Third, in the limit the sample average of a function converges to its expected value computed w.r.t.\ $\st$. 

\begin{lemma}
    \label{lemma:semantic_equivalence}
    Let $\seq{\rO}\sim \pomc$ be an irreducible and positively recurrent POMC (Assumption~\ref{ass:stationarity}) and let $\var_n$ be an atomic function of arity $n\in \NN$. 
    Then 
    \begin{align*}
       \sem{\var_n}(\seq{\rO}) = \sem{\var_n}(\pomc) \quad a.s.
    \end{align*}
\end{lemma}
\begin{proof}
    Let $\pomc\coloneqq \tup{\aS,\trm,\initd,\aO,\obs}$. Let $\mc \coloneqq  (\aS,\trm, \initd)$ be the irreducible and positive recurrent Markov chain of $\pomc$, with unique stationary distribution $\st$. We construct a $n^{\text{th}}$ order Markov chain.
    The chain $\mc^{(n)}=(\aS^n,\trm^{(n)}, \initd^{(n)}) $ is defined by its transition matrix $\trm^{(n)}$ where
    for $(\cS_1, \dots, \cS_n)\in \aS^n$ and $\cS_{n+1}\in \aS$ 
    \begin{align*}
        \trm_{(\cS_1, \dots, \cS_n),(\cS_2, \dots, \cS_{n+1})}^{(n)}= \trm(\cS_n, \cS_{n+1})
    \end{align*}
    and $0$ everywhere else, and its initial distribution 
        \begin{align*}
        \initd^{(n)}(\cS_1, \dots, \cS_n) \coloneqq \initd(\cS_1) \cdot \prod_{i=1}^{n-1} \trm_{\cS_i, \cS_{i+1}}.
    \end{align*}
    This Markov chain represent a $n$ sliding window of the original. 
    We define the distribution $\st^{(n)}$ such that for every $(\cS_1, \dots, \cS_n)\in \aS^n$ we have 
    \begin{align*}
        \st^{(n)}(\cS_1, \dots, \cS_n) \coloneqq \st(\cS_1) \cdot \prod_{i=1}^{n-1} \trm_{\cS_i, \cS_{i+1}}.
    \end{align*}
    This distribution is the stationary distribution of the augmented Markov chain, i.e., for every $(\cS_1, \dots, \cS_n)\in \aS^n$ we have
    \begin{align*}
        &\sum_{(r_1, \dots, r_n) \in \aS^n} \st^{(n)}(r_1, \dots, r_n) \cdot \trm_{(r_1, \dots, r_n), (\cS_1, \dots, \cS_n)}^{(n)} \\
        &= 
        \sum_{r\in \aS} \st^{(n)}(r, \cS_1, \dots, \cS_{n-1})  \trm_{\cS_{n-1}, \cS_n} \\
        &=
        \sum_{r\in \aS} \st(r) \cdot \trm_{r, \cS_1} \cdot \prod_{i=1}^{n-2} \trm_{\cS_i, \cS_{i+1}}  \trm_{\cS_{n-1}, \cS_n} \\
        &=
        \left( \sum_{r\in \aS} \st(r) \cdot \trm_{r, \cS_1} \right) \cdot \prod_{i=1}^{n-1} \trm_{\cS_i, \cS_{i+1}}  \\
        &= 
        \st(\cS_1) \cdot  \prod_{i=1}^{n-1} \trm_{\cS_i, \cS_{i+1}} \\
        &= \st^{(n)}(\cS_1, \dots, \cS_n) .
    \end{align*}  
    The first equality is because any non-zero entry in $\trm^{(n)}$ must be between two states overlapping on $n-1$ terms. 
    For $(\cS_1, \dots, \cS_n) \in \aS^n$ we define $f(\cS_1, \dots, \cS_n)\coloneqq \var_n(\obs(\cS_1), \dots, \obs(\cS_n))$.
    By the ergodic theorem (Theorem 1.10.2~\cite{Norris_1997}) we have almost sure convergence for $\seq{X}^{(n)}\sim \mc^{(n)}$, i.e., 
    \begin{align*}
        \lim_{t\to \infty} \frac{1}{t}\sum_{i=1}^t f(X_i^{(n)}) 
        = \sum_{(\cS_1, \dots, \cS_n) \in \aS^{n}}  \st^{(n)}(\cS_1, \dots, \cS_n) \cdot f(\cS_1, \dots, \cS_n) .
    \end{align*}
    By construction we know that
    \begin{align*}
         \frac{1}{t}\sum_{i=1}^t f(X_i^{(n)})  
         &=  \frac{1}{t}\sum_{i=1}^{t} \var_n(\obs(X_i), \dots, \obs(X_{i+n-1})) =  \frac{1}{t}\sum_{i=1}^{t} \var_n(\rO_i, \dots, \rO_{i+n-1}) 
    \end{align*}
    and that
    \begin{align*}
         &\sum_{(\cS_1, \dots, \cS_n) \in \aS^{n}}  \st^{(n)}(\cS_1, \dots, \cS_n) \cdot f(\cS_1, \dots, \cS_n) \\
         &= 
          \sum_{(\cS_1, \dots, \cS_n) \in \aS^{n}} \st(\cS_1) \cdot \prod_{i=1}^{n-1} \trm_{\cS_i, \cS_{i+1}} \cdot \var_n(\obs(\cS_1), \dots, \obs(\cS_n)) \\
          &= \expe_{\st}^{\pomc}(\var_n(\rO_1, \dots, \rO_n)) .
    \end{align*}
    Hence, we prove our claim by concluding that
\begin{align*}
        \lim_{t\to \infty}  \frac{1}{t-(n-1)}\sum_{i=1}^{t-(n-1)} \var_n(\rO_i, \dots, \rO_{i+n-1})  = \expe_{\st}^{\pomc}(\var_n(\rO_1, \dots, \rO_{n})) \quad a.s..
    \end{align*}
\end{proof}
\noindent
From Lemma~\ref{lemma:semantic_equivalence} we can establish the equivalence of the semantics by induction. 
\begin{theorem}
    \label{theorem:semantic_equivalence}
    Let $\seq{\rO}\sim \pomc$ be an irreducible and positively recurrent POMC (Assumption~\ref{ass:stationarity}) and let $\iprop$ be a \vbse. Then 
    \begin{align*}
        \sem{\iprop}(\seq{\rO}) = \sem{\iprop}(\pomc) \quad a.s.
    \end{align*}
\end{theorem}
\begin{proof}
    We prove this by induction on the structure of the expression. The induction hypothesis is $\sem{\iprop}(\seq{\rO})=\sem{\iprop}(\pomc)$. The induction hypothesis implies that the limit $\lim_{t\to \infty} \sem{\iprop}_t(\seq{\rO})$ exists almost surely.
    There are two base cases. The first, $\iprop=\kappa$, holds by definition $\sem{\kappa}(\seq{\rO})=\kappa=\sem{\kappa}(\pomc)$. The second, $\iprop=\var_n$, follows from Lemma~\ref{lemma:semantic_equivalence}.
    The induction step is identical for every $\star\in \{+, \cdot, \div\}$ and follows from the limit laws.
    Let $\iprop= \psi\star\chi$ then by induction hypothesis $\sem{\psi}(\seq{\rO})=\sem{\psi}(\pomc)$ and $\sem{\chi}(\seq{\rO})=\sem{\chi}(\pomc)$. This implies that the limits of $\sem{\psi}(\seq{\rO})$ and $\sem{\chi}(\seq{\rO})$ exist almost surely. Hence, the limit of $\sem{\psi\star \chi}(\seq{\rO})$ exists, resulting in
    \begin{align*}
       \sem{\iprop}(\seq{\rO}) &= \lim_{t \to \infty} \sem{\psi\star \chi}(\seq{\rO}) =    \lim_{t\to \infty} \sem{\psi}(\seq{\rO})  \star \lim_{t \to \infty} \sem{\chi}(\seq{\rO}) \\
       &= \sem{\psi}(\pomc) \star \sem{\chi}(\pomc) = \sem{\iprop}(\pomc).
    \end{align*}
\end{proof}

\paragraph{PSE model-based semantics.}
Under the model-based semantics introduced by Henzinger et al.~\cite{henzinger2023monitoring}, the semantic value of $\pr(r|q)$, where $q$ and $r$ are states of the MC, is equivalent to the transition probability from $q$ to $r$ as given by $\trm_{qr}$, i.e., 
\begin{align*}
    \forall q, r \in \aS \colon \sem{\pr(r\mid q)}(\mc) = \trm_{qr}.
\end{align*}

\paragraph{Equivalence between path-based and model-based semantics of PSEs.}
We show the equivalence between the path- and model-based semantics of PSE for MCs satisfying Assumption~\ref{ass:stationarity}. This is a direct consequence of Theorem~\ref{theorem:semantic_equivalence}.
We demonstrate the equivalence by showing that the encoding in Equation \ref{eq:encoding} evaluates to the transition probabilities under the model-based semantics of general \vbse.
\begin{corollary}
    \label{corr:trm_prob}
    Let $\seq{\rO}\sim \mc$ be an irreducible and positively recurrent Markov chain  (Assumption~\ref{ass:stationarity}). For all states $q,r\in \aS$ we have $\sem{\pr(r|q)}(\seq{\rO})=\trm_{qr}$ almost surely.  
\end{corollary}
\begin{proof}
    Let $\seq{\rO}\sim \mc$. From Theorem~\ref{theorem:semantic_equivalence} we know
    \begin{align*}
        \sem{\pr(r|q)}(\seq{\rO}) &= \frac{\expe_{\st}^{\mc}(\indi{qs}(\seq{\rO}_t))}{\expe_{\st}^{\mc}(\sum_{r\in \aS}\indi{qr}(\seq{\rO}_t))} \\
        &= \frac{\prob_{\st}^{\mc}(X_1=q,X_2=s)}{\sum_{r\in \aS}\prob_{\st}^{\mc}(X_1=q,X_2=r)} \\
        &= \frac{\prob_{\st}^{\mc}(X_1=q,X_2=s)}{\prob_{\st}^{\mc}(X_1=q)} \\ 
        &= \prob_{\st}^{\mc}(X_2=s \mid X_1=q)  \\
        &= \trm_{qs}.
    \end{align*}
\end{proof}

\subsection{Generalisation of Model-based Semantics}
We show that the path-based semantics for \vbse is not well-defined in general. However, it is well-defined over the class of POMCs in which the state space can be decomposed into a finite transient component and a finite number of irreducible and positively recurrent components. 
This class is formally specified in Assumption~\ref{ass:base}.

\begin{assumption}
\label{ass:base}
    We assume that the state space of a POMCs can be decomposed into a finite set of transient states and a finite set of irreducible and positively recurrent components. Formally, let $\pomc$ be a POMC, then
    \begin{align*}
        \aS \coloneqq \aS_T \cup \aS_{C_1} \cup \dots \cup \aS_{C_n}
    \end{align*}
   $|\aS_T|< \infty$, $n\in \NN$, and for all $i\in [n]$ the states in $\aS_{C_i}$ are irreducible and positively recurrent.
\end{assumption}

Two important classes covered by Assumption~\ref{ass:base} are irreducible and positively recurrent POMCs, as well as POMCs with finite state space. We start with an example that demonstrates the path-based semantics is not well-defined for all POMCs. 

\begin{example}
    Let $\seq{\cO}$ be a sequence in $\{0,1\}^{\omega}$ consisting of alternating blocks of $0$s and $1$s, each of exponentially increasing length, i.e., for every $t\in \pNN$ we have
    \begin{align*}
        \cO_t \coloneqq 
        \begin{cases}
            1  \quad & \lfloor \log_2 t \rfloor\;  \text{is odd}, \\
            0  \quad & \lfloor \log_2 t \rfloor \; \text{is even}.
        \end{cases}
    \end{align*}
    We can easily construct a POMC that generates only this sequence. 
    Let $\var\coloneqq \identity$ be the identity function, then we can show that the limit $ \lim_{t\to \infty} \frac{1}{t}\sum_{i=1}^t \cO_i$
    does not exist. We prove that the $\liminf$ is upper bounded by $\frac{1}{3}$ and that the $\limsup$ is lower bounded by $\frac{2}{3}$.
   Let $k\coloneqq \lfloor \log_2 t \rfloor$, we compute the value of $ \frac{1}{t}\sum_{i=1}^t \cO_i$ after block number $k$. Let $\cS_k$ the value of the sum and let $t_k$ the value of $t$ after the $k^{\text{th}}$ block. 
    Using the closed form expression for the partial sums of geometric sequences we compute the value of $t_k$ after every block
    \begin{align*}
         t_k = \sum_{i=0}^{k-1} 2^{i} =  \frac{2^{k}-1}{2-1} = 2^{k} - 1 \\
    \end{align*}
    Assume $k$ to be odd. Hence, we just ended a $1$ block, thus
    \begin{align*}
        \cS_k &=\sum_{i=0}^{(k-1)/2} 2^{2i} = \sum_{i=0}^{(k-1)/2} 4^{i} = \frac{4^{(k-1)/2+1}-1}{4-1} = \frac{2^{k+1}-1}{3} 
    \end{align*}
    This results in the value
    \begin{align*}
       \lim_{k\to \infty} \frac{\cS_k}{t_k} = \frac{\frac{1}{3}(2\cdot 2^k-1)}{2^{k} - 1} =  \frac{\frac{2}{3}(2-\frac{1}{2^k})}{1-\frac{1}{2^k}} = \frac{2}{3}
    \end{align*}
    Since $0$ do not contribute to the sum we have for odd $k$'s the equality $\cS_{k+1}=\cS_{k}$ and therefore we obtain
    \begin{align*}
        \lim_{k\to \infty} \frac{\cS_k}{t_{k+1}} = \frac{\frac{1}{3}(2^{k+1}-1)}{2^{k+1} - 1} =  \frac{1}{3}.
    \end{align*}
    Hence, we know that the value $\sem{\identity}{\seq{\cO}}$ alternates between $\frac{1}{3} $ and $\frac{2}{3}$ indefinitely. 
\end{example}

Next we show that the path-based semantics is well-defined over the POMCs satisfying Assumption~\ref{ass:base}.
\begin{theorem}
\label{theorem:well-defined}
    Let $\pomc$ be a POMC satisfying Assumption~\ref{ass:base}. Let $\seq{\rO}\sim \pomc$ and let $\iprop$ be a \vbse, then for every realization $\seq{\cO}$ of $\seq{\rO}$ the $\lim_{t\to\infty} \; \sem{\iprop}(\seq{\cO}_t)$ exists. 
\end{theorem}
\begin{proof}
    Because $\aS_T$ is finite we know that the underlying Markov chain $\seq{\rS}$ of $\seq{\rO} \sim  \pomc$ almost surely enters an irreducible and positively recurrent component $\aS_{C_k}$ and staying there forever, i.e.,
    \begin{align*}
        \prob^{\pomc}\left( \exists i \in \pNN \exists k \in [n]  \forall t \geq i \colon\rS_t\in \aS_{C_k} \right) = 1.
    \end{align*}
    From Theorem~\ref{theorem:semantic_equivalence} we know that the limit exists for every irreducible and positively recurrent component. Formally, let $\aS_{C_k}$ be the set of states associated with component $k\in [n]$.
    Let $\pomc_{C_k} \coloneqq (\aS_{C_k},\trm_{C_k},\initd_{C_k},\aO,\obs )$ be the POMC restricted to the state space $\aS_{C_k}$ with an arbitrary initial distribution $\initd_{C_k}$ with support $\aS_{C_k}$. Then $\pomc_{C_k}$ is an irreducible and positively recurrent POMC and thus for $\seq{U} \sim \pomc_{C_k}$ the limit is equivalent to the expectation w.r.t. the stationary distribution $\st_{\aS_{C_k}}$ of $\pomc_{C_k}$, i.e., the following equality holds almost surely
    \begin{align*}
        \lim_{t\to\infty} \; \sem{\iprop}(\seq{U}_t)  = \sem{\iprop}(\pomc_{C_k}).
    \end{align*}
    Now let $\seq{s}$ and $\seq{\cO}$ be a realization of both $\seq{\rS}$ and $\seq{\rO}$ respectively. We know there exist a point $N\in \pNN$ after which every state will be contained in some component $k\in [n]$, therefore we know that the infinite suffix $\seq{\cO}_{N:}$ will be a realization of the $\pomc_{C_k}$ with the initial distribution $\conf_{\cS_N}$, i.e., the process $\seq{U}\sim\pomc_{C_k}$ starting in $\cS_N$. Because the $\seq{U}$ almost surely converges we know that $\sem{\iprop}(\seq{\cO})$ exists. However, due to the prefix $\seq{\cO}_{N}$ the value of $ \sem{\iprop}(\seq{\cO})$ does not necessarily align with $\sem{\iprop}(\pomc_{C_k})$ as given by the model-based semantics. 
\end{proof}

% \begin{lemma}
	
% \end{lemma}

\section{Monitoring Statistical Properties}
Informally, our goal is to build monitors that observe randomly generated observation sequences of increasing length from a given unknown POMC, and, after each observation, will generate an updated estimate of how fair or biased the system is.
Since the monitor's estimate is based on statistics collected from a finite path, the output may be incorrect with some probability.
That is, the source of randomness is from the fact that the prefix is a finite sample of the fixed but unknown POMC.

\paragraph{Soundness.}
For a subset of POMCs $\sproc$, a given \vbse $\iprop$, and a given $\conf\in (0,1)$, we define a \emph{problem instance} as the tuple $\tup{ \sproc, \iprop,\conf}$.
Given a problem instance, we first formalize two notions of soundness for monitors.
\begin{definition}[Pointwise soundness of monitors.]
	Suppose $\tup{ \sproc, \iprop,\conf}$ is a problem instance. 
	A monitor $\monitor$ with output alphabet $\set{[l,u]\mid l,u\in\mathbb{R}\;.\; l < u}$ is called \emph{$\conf$ pointwise sound for $\iprop$ over $\sproc$} iff for every POMC $\pomc$ in $\sproc$, the following holds:
 	\begin{align*}
       \forall t\in\NN : \prob_{\initd}^{\pomc}\left( \sem{\iprop}(\seq{\rout}) \in \monitor(\seq{\rout}_t)\right) \geq 1-\conf.
    \end{align*}   	
\end{definition}

\begin{definition}[Uniform soundness of monitors.]
	Suppose $\tup{ \sproc, \iprop,\conf}$ is a problem instance.
	A monitor $\monitor$ with output alphabet $\set{[l,u]\mid l,u\in\mathbb{R}\;.\; l < u}$ is called \emph{$\conf$ uniformly sound for $\iprop$ over $\sproc$} iff for every POMC $\pomc$ in $\sproc$, the following holds:
     \begin{align*}
       \prob_{\initd}^{\pomc}\left(  \forall t\in\NN :\sem{\iprop}(\seq{\rout}) \in \monitor(\seq{\rout}_t)\right) \geq 1-\conf.
    \end{align*}   	
\end{definition}
Intuitively, the verdict of a given monitor is \emph{correct}, if its output interval includes the true semantic value of the \vbse on the monitored POMC.
A pointwise sound monitor guarantees that at every time point its verdict is correct with probability $1-\conf$. In contrast, a uniformly sound monitor guarantees with probability $1-\conf$ that there is not a single incorrect verdict in the (potentially) infinite sequence of verdicts generated by the monitor. Every uniformly sound monitor is trivially pointwise sound.
However, a direct consequence of the law of the iterated logarithm is that the other direction is not true, as shown in the following example.

\begin{example}
	% \KK{\rOe should talk here. We can not use it directly.}
		Consider a simple POMC $\pomc\coloneqq (\aS,\trm,\initd,\aO,\obs)$ where $\aS \coloneqq \{a, b\}$, $\aO\coloneqq \{-1,1\}$, 
		\begin{align*}
			\trm\coloneqq \begin{small}\begin{pmatrix}0.5 & 0.5 \\ 0.5 & 0.5\end{pmatrix}\end{small},  
			\quad \initd\coloneqq \begin{small}\begin{pmatrix}0.5 \\ 0.5\end{pmatrix}\end{small},\; \text{and} \quad \obs(q) \coloneqq 
			\begin{cases}
				1 &\quad \text{if} \;\;  q=a \\
				-1 &\quad \text{if} \;\; q=b
			\end{cases}.
		\end{align*}
		The resulting sequence of observations $\seq{\rO}\coloneqq (\rO_t)_{t\in \NN}$ is equivalent to a sequence of i.i.d.\ Radermacher random variables. 
		Moreover, $(\iprop, \conf)$ be a problem instance where, the BSE is the identity function, i.e., $\iprop \coloneqq \identity$.
		We construct the monitor $\monitor:\aO^* \to \inter{\RN}$ such that for all $\seq{\cO}\in\aQ^*$ with $t=|\seq{\cO}|$
		\begin{align*}
			\monitor(\seq{\cO})\coloneqq \frac{1}{t} \sum_{i=1}^t \cO_i \pm \sqrt{\frac{2}{t} \log\left(\frac{2}{\conf}\right)}.
		\end{align*}
        For simplicity, we define $\bar{\cO}_t \coloneqq \sum_{i=1}^t \cO_i$.
		We show that the monitor $\monitor$ is pointwise sound but not uniformly sound for the set of POMC $\{\pomc\}$. To be specific we show that for $\pomc$ the statement 
		\begin{align*}
			\forall t\in \NN :  \prob\left( \sem{\iprop}(\seq{\rO}) \in \monitor(\seq{\rO}_t) \right) \geq 1-\conf
		\end{align*}
		holds, while at the same time 
		\begin{align*}
			\prob_{\initd}^{\pomc}\left( \forall t\in \mathbb{N} : \sem{\iprop}(\seq{\rO}) \in \monitor(\seq{\rO}_t)  \right) = 0.
		\end{align*}
		We establish pointwise soundness using Hoeffding's inequality. 
		That is, for all $t\in \mathbb{N}$ we know that
		\begin{align*}
			\prob_{\initd}^{\pomc}\left( 0 \in \frac{1}{t}\bar{\rO}_t\pm \sqrt{\frac{2}{t} \log\left(\frac{2}{\conf}\right)}\right)=\prob_{\initd}^{\pomc}\left(\frac{1}{t}\left|\bar{\rO}_t \right| \leq \sqrt{\frac{2}{t} \log\left(\frac{2}{\conf}\right)}\right) \geq 1- \conf.
		\end{align*}
		At the same time we observe that 
		\begin{align*}
			&\prob_{\initd}^{\pomc}\left(\exists t\in \mathbb{N} : \frac{1}{t}\left|\bar{\rO}_t \right| \geq \sqrt{\frac{2}{t} \log\left(\frac{2}{\conf}\right)}\right) =\prob_{\initd}^{\pomc}\left(\exists t\in \mathbb{N} : \frac{\left|\bar{\rO}_t \right|}{\sqrt{2t\log(2/\conf)}} \geq 1\right) \\
			&=\prob_{\initd}^{\pomc}\left(\sup_{t\in \mathbb{N}} \frac{\left|\bar{\rO}_t \right|}{\sqrt{2t\log(2/\conf)}} \geq 1\right) 
		 \stackrel{\text{(i)}}{\geq} \prob_{\initd}^{\pomc}\left(\lim_{t\to \infty}\sup \frac{\left|\bar{\rO}_t \right|}{\sqrt{2t\log(2/\conf)}} \geq 1\right) \\
			& \stackrel{\text{(ii)}}{\geq} \prob_{\initd}^{\pomc}\left(\lim_{t\to \infty}\sup \frac{\left|\bar{\rO}_t \right|}{\sqrt{2t\log\log t}} \geq 1\right)  \stackrel{\text{(iii)}}{=} 1.
		\end{align*}
		The first inequality (i) follows from the fact that the limsup will always be smaller than or equal to the supremum. The second inequality (ii) follows from the fact that $\log(2/\conf) \leq \log \log t$ for $t\geq e^{2/\conf}$. The equality (iii) follows from the law of iterated logarithm. 
		The proof concludes by taking the complement and applying De Morgan.
		
	%     Let  $(\aQ, \sproc, \{\iprop\},\conf)$ be a problem instance where $\aQ=\RN$, $\sproc\coloneqq\{\seq{\gauss}(\mu, 1) \mid \forall \mu \in\aQ \}$  where $\seq{\rQ}\sim \seq{\gauss}(\mu, 1)$ is a sequence of i.i.d. Gaussian random variables with mean $\mu\in \RN$ and $\sigma^2=1$, and $\iprop(\seq{\gauss}(\mu, 1))=\mu$ for all $\mu\in \RN$. We construct the monitor $\monitor:\aQ^* \to \inter{\RN}$ such that for all $\seq{\cS}\in\aQ^*$, $t=|\seq{\cS}|$ and $\bar{\cS}_t\coloneqq \sum_{i=1}^t\cX_i$.
    % \begin{align*}
    %     \monitor(\seq{\cS})\coloneqq \frac{1}{t}\bar{\cX}_t \pm \sqrt{\frac{2}{t} \log\left(\frac{2}{\conf}\right)}.
    % \end{align*}
    % The monitor $\monitor$ is sound but not uniformly-sound, i.e., for all $\iproc\in\sproc$
    % \begin{align*}
    %     \forall t\in \NN :  \prob_{\iproc}\left( \iprop(\iproc) \in \monitor(\seq{\rX}_t) \right) \geq 1-\conf
    % \end{align*}
    % but there exists $\iproc\in\sproc$ such that
    % \begin{align*}
    %     \prob_{\iproc}\left( \forall t\in \mathbb{N} :  \iprop(\iproc) \in \monitor(\seq{X}_t)  \right) = 0
    % \end{align*}
    % is tight.

\end{example}

\begin{problem}
\label{prob:monitoring}
   The \emph{monitoring problems} that we consider are: given a problem instance $\tup{ \sproc, \iprop,\conf}$, algorithmically construct monitors for $\iprop$ that are $\conf$ pointwise sound and $\conf$ uniformly sound over $ \sproc$.
\end{problem}
% In the rest of the paper, we will refer to the output interval $[l,u]$ of a given monitor as \emph{confidence interval}, its radius, given by $\varepsilon=0.5\cdot (u-l)$, as the \emph{estimation error}, the quantity $\conf$ as the \emph{failure probability}, and the quantity $1-\conf$ as the \emph{confidence}.
% Intuitively, the monitor outputs the estimated confidence interval that contains the range of values within which the true semantic value of $\iprop$ falls with $(1-\conf)\cdot 100\%$ probability.
% The estimate gets more precise as the error gets smaller, and the confidence gets higher.
% We will prefer the monitor with the maximum possible precision, i.e., having the least estimation error for a given $\conf$.

\section{Monitoring BSEs on POMCs}
\label{sec:rv}
In this section we present a pointwise and a uniformly sound monitor for monitoring general $\vbse$ over irreducible, positively recurrent, and aperiodic POMCs, which start in their stationary distribution and have a known upper bound on their mixing time. 
This class of POMCs, as specified below, is a more restrictive than the class of POMCs specified in Assumptions~\ref{ass:stationarity}. 
\begin{assumption}\label{ass:rv:stationarity}
    We assume that the unknown POMC $\pomc$ generating the sequence $\seq{\rO}\sim \pomc$ is from the set $\rvsproc$ of all irreducible, positively recurrent, and aperiodic POMC starting in their stationary distribution.
    Additionally, we assume knowledge of an upper bound on the mixing time $\taumix$ of $\pomc$. 
\end{assumption}
\noindent
When given a realization of an unknown POMC from this class of POMCs, our monitor will be able to provide converging confidence intervals for any \vbse under the assumption that it is provided with an upper bound on the POMC's mixing time. 
Intuitively, this is because irreducibility and positive recurrence guarantee that the monitor observes every state infinitely many times and that there exists a stationary distribution reflecting the proportion of time the POMC will spend in a particular state. The mixing time and aperiodicity is required by the monitor to compute the confidence interval.

\subsection{Monitoring Algorithm}
\label{sec:interval estimator for atoms}
% Fix a problem instance  $(\aS,\rvsproc, \rvsprop,\conf)$,  and fix some $\rviprop\in\rvsprop$ let the size of $\rviprop$ be $n$. 
We start by presenting the point estimate for atomic functions, which we then combine to obtain a monitor for general $\vbse$ using union bounds and interval arithmetic. 

\subsubsection{Monitoring Atomic Functions}
A monitor for each individual atomic function is called an atomic monitor, which serves as the building block for the overall monitor.
The output confidence interval of each atomic monitor on a given finite observed path is constructed by first computing a point estimate of the semantic value of the atom based on the observed path, and second computing the estimation error for the given confidence level using a McDiarmid-style concentration inequality~\cite{paulin2015concentration}. The chosen inequality ensures the pointwise soundness of the monitor and can be modified to obtain a uniformly sound monitor (details are in Lemma \ref{lemma:confseq-naive}).

\paragraph{Point estimate.}
Consider an arbitrary \vbse atom $\rvvar$ of arity $n$. We propose the finitary semantic value of $\rvvar$ as a suitable point estimate of the semantic value of $\rvvar$.
The finitary semantic value of $\rvvar$ over the finite prefix $\seq{\cO}_t$ of length $t>n$ of the sequence $\seq{\cO} \in \aO^{\omega}$ is defined in Equation~\ref{eq:semantics} as 
\begin{align}\label{equ:point estimator}
    \sem{\rvvar}(\seq{\cO}_t) = \frac{1}{t-(n-1)}\sum_{i=1}^{t-(n-1)} \rvvar(\cO_i,\dots, \cO_{i+(n-1)}).
\end{align}
% Since every POMC $\rviproc\in\rvsproc$ starts in the stationary distribution $\st$, we can show that this is an unbiased estimate of $\sem{\rvvar}(\seq{\cO})$. 
In Proposition \ref{lemma:expectation atom}, we establish the unbiasedness of the estimator $\sem{\rvvar}(\cdot)$ for POMCs in $\rvsproc$ because they start in their stationary distribution. We need the estimator to be unbiased as it guarantees that the expected value of the estimator's output will coincide with the true value of the property that is being estimated. Similar to Proposition~\ref{lemma:expectation atom}, the Corollary~\ref{lemma:probability atom} does the same for the fragment of \vbse with probabilities of sequences.

\paragraph{Confidence interval.}
We now summarize the estimation errors of the atomic monitors; the correctness will be established in Theorem \ref{thm:soundness of atomic monitor}.
We use $\rvcifoo$ to denote the estimation error bound, whose value is based on whether a pointwise sound or uniformly sound monitor is required.
For pointwise sound monitors we use
\begin{align*}
    \rvconferrors(\conf,t) \coloneqq \sqrt{\log(2/\conf)\cdot \frac{t\cdot n^2\cdot (b-a)^2 \cdot 9 \cdot \tau_{mix}}{2 (t-(n-1))^2}},
\end{align*}
and for uniformly sound monitors we use
\begin{align*}
      \rvconferroru(\conf,t) \coloneqq \sqrt{\log\left( \frac{\pi^2t^2}{3\conf}\right)\cdot \frac{t\cdot n^2\cdot (b-a)^2 \cdot 9 \cdot \tau_{mix}}{2 (t-(n-1))^2}} .
\end{align*}
Since uniform soundness is a stronger requirement than pointwise soundness, the computed interval is by a factor of $\log t$ larger.
Both the point estimate and error calculation are done online by our monitor. The corresponding calculations are presented in Algorithm~\ref{alg:atomic monitor}.

% In order to use McDiarmid's inequality, we will need the following standard \cite{levin2017markov} assumption on such as aperiodicity and a bounded mixing time.
% \begin{assumption}\label{assump:aperiodicity}
% 	We assume that the POMCs are aperiodic, and that the mixing time of the POMC is bounded by a known constant $\taumix$.
% \end{assumption}

\begin{algorithm}
 	\caption{$\mathit{Monitor}_{(\rvvar,\conf)}$: Monitor for $(\rvvar,\conf)$ where $\rvvar\colon \aO^n\to [a,b]$ is an atomic function of a \vbse}
 	\label{alg:atomic monitor}
 		\begin{minipage}{0.34\textwidth}
 			\begin{algorithmic}[1] % init
 			\Function{$\mathit{Init}()$}{}
 				\State $t\gets 0$ \Comment{current time}
 				\State $y\gets 0$ \Comment{current point estimate}
 				\State $\seq{\cO}\gets \underbrace{\bot \ldots \bot}_{n \text{ times}}$ \Comment{a dummy word of length $n$, where $\bot$ is the dummy symbol}
 			\EndFunction
 			\end{algorithmic}
 		\end{minipage}
 		\begin{minipage}{0.65\textwidth}
 			\begin{algorithmic}[1]
 			\Function{$\mathit{Next}(\sigma)$}{}
 				\State $t \gets t+1$ \Comment{progress time}
 				\If{$t<n$} \Comment{too short observation sequence}
 					\State $\seq{\cO}_t \gets \sigma$ 
 					\State \Return $\bot$ \Comment{inconclusive}
 				\Else
 					\State $\seq{\cO}_{1: n-1} \gets \seq{\cO}_{2: n}$ \Comment{shift window}\label{step:shift window}
 					\State $\seq{\cO}_n \gets \sigma$ \Comment{add the new observation}
 					\State $x \gets \rvvar(\seq{\cO})$ \Comment{latest evaluation of $\rvvar$}
 					\State $y \gets \left(y*(t-n)+x\right)/(t-(n-1))$ \Comment{running av.\ impl.\ of Eq.~\ref{equ:point estimator}}
 					\State $\varepsilon\gets \rvcifoo(\conf, t)$ 
 					% \Comment{ PAC bound, see \cite{henzinger2023monitoring2}}
 					\State \Return $[y-\varepsilon,y+\varepsilon]$  \Comment{confidence interval}
 				\EndIf
 			\EndFunction
 		\end{algorithmic}
 		\end{minipage}
 \end{algorithm}

The following theorem summarizes the correctness of atomic monitors.

\begin{theorem}
\label{thm:soundness of atomic monitor}
Let $(\rvsproc, \rviprop ,\conf)$  be a problem instance. Let $\rviprop\coloneqq \rvvar$ be an atomic function of arity $n\in \pNN$. 
Algorithm \ref{alg:atomic monitor} implements a monitor that is pointwise sound with error bound $\rvcifoo= \rvconferrors$ and uniformly sound with error bound $\rvcifoo= \rvconferroru$ for the set $\rvsproc$.
The monitor requires $\mathcal{O}(n)$-space, and, after arrival of each new observation, computes the updated output in $\mathcal{O}(n)$-time.
\end{theorem}

\subsubsection{Monitoring {\vbse}s.}
The final monitors for {\quanvbse}s is presented in Alg.~\ref{alg:quant monitor}, where we recursively combine the interval estimates of the constituent sub-expressions using interval arithmetic and the union bound.
Similar idea was used by Albarghouthi et al.~\cite{albarghouthi2019fairness}.
The correctness and computational complexities of the monitors are formally stated in Theorem \ref{thm:rv:quantitative}.

\begin{algorithm}[H]
		\caption{$\mathit{Monitor}_{(\varphi_1\odot\varphi_2,\delta_1+\delta_2)}$}
		\label{alg:quant monitor}
		\begin{algorithmic}[1]
			\Function{$\mathit{Init}()$}{}
				\State $\monitor_1\gets \mathit{Monitor}_{(\varphi_1,\delta_1)}$
				\State $\monitor_2\gets \mathit{Monitor}_{(\varphi_2,\delta_2)}$
				\State $\monitor_1.\mathit{Init}()$
				\State $\monitor_2.\mathit{Init}()$
			\EndFunction
		\end{algorithmic}
		\begin{algorithmic}[1]
			\Function{$\mathit{Next}$}{$\sigma$}
				\State $[l_1,u_1]\gets \monitor_1.\mathit{Next}(\sigma)$
				\State $[l_2,u_2]\gets \monitor_2.\mathit{Next}(\sigma)$
				\State \Return $[l_1,u_1]\odot [l_2,u_2]$ \Comment{interval arithmetic}
			\EndFunction
		\end{algorithmic}
	\end{algorithm}

\begin{theorem}[Solution of Prob.~\ref{prob:monitoring}]
\label{thm:rv:quantitative}
	Let  $(\rviprop, \conf)$  be a problem instance. Let $\rviprop \coloneqq \rviprop_1\odot\rviprop_2$ for $\odot\in\set{+,\cdot,\div}$, and let $\delta \coloneqq \delta_1+\delta_2$. 
	Algorithm~\ref{alg:quant monitor} implements a monitor that is pointwise sound with error bound $\rvcifoo= \rvconferrors$ and uniformly sound with error bound $\rvcifoo= \rvconferroru$ for the set $\rvsproc$.
	If the total number of atoms in $\rviprop_1\odot\rviprop_2$ is $k$ and if the arity of the largest atom in $\rviprop_1\odot\rviprop_2$ is $n$, then $\monitor$ requires $\mathcal{O}(k+n)$-space, and, after arrival of each new observation, computes the updated output in $\mathcal{O}(k\cdot n)$-time.
\end{theorem}

\subsection{Technical Proofs of Correctness}
In this sub-section we present the proofs of Theorem~\ref{thm:soundness of atomic monitor} and Theorem~\ref{thm:rv:quantitative}.
We start with Theorem~\ref{thm:soundness of atomic monitor}, which establishes the pointwise and uniform soundness of Algorithm~\ref{alg:atomic monitor}. First, we show that the point estimator for atomic functions is indeed an unbiased estimator (see Lemma \ref{lemma:expectation atom}). 
Here we rely on the fact that the Markov chain starts in its stationary distribution. Second, we demonstrate that the interval estimate constructed around the point estimate captures the true value of the property with high probability. We accomplish this by leveraging a McDiarmid-style concentration inequality for Markov chains \cite{paulin2015concentration}, which directly generalizes to POMCs (see Corollary 2.17 \cite{paulin2015concentration}). We restate the Theorem for Markov chains below. 

\begin{theorem}[\cite{paulin2015concentration}]
	\label{thm:mcdiarmid}
	Let $\seq{\rS}_n\coloneqq \rS_1, \dots, \rS_n$ be an ergodic Markov chain $\mc\coloneqq (\aS, \trm, \st)$ with countable state space $\aS$, unique stationary distribution $\st$, and finite mixing time bounded by $\taumix$.
	Suppose that some function $f: \aS^n \to \RN$ with artiy $n$ satisfies s.t. $\forall \seq{x},\seq{y}\in\aS$ 
	\begin{align*}
		|f(\seq{x})- f(\seq{y})| \leq \sum_{i=1}^n  c_i \indi{x_i \neq y_i}
	\end{align*}
	for some $c\in \mathbb{R}^n$ with positive entries. Then for any $\varepsilon>0$
	\begin{align*}
		\prob_{\st}^{\mc}\left(\left| f(\seq{\rS_n}) - \expe_{\st}^{\mc}(f(\seq{\rS_n}))\right|\geq \varepsilon \right) \leq 2\exp\left( - \frac{2\varepsilon^2}{\sum_{i=1}^n c_i^2 \cdot 9\cdot \taumix} \right).
	\end{align*}

\end{theorem}
This is a generalization of the classical McDiarmid's inequality which bounds the distance between the sample value and the expected value of a function satisfying the bounded difference property when evaluated on independent random variables. The inequality we us is only one of many generalizations~\cite{paulin2015concentration,esposito2023concentration,kontorovich2017concentration}. 
The particular theorem assumes that the partially observed Markov chain has a unique stationary distribution and requires a known bound on mixing time $\taumix$ of the Markov chain. Moreover, it requires that our estimator is unbiased and that it satisfies the bounded difference property, which we demonstrate in Lemma \ref{lemma:expectation atom} and Lemma \ref{lemma:bounded difference} respectively. This establishes the pointwise soundness of the atomic monitor. We lift this to uniform soundness by performing a union bound across time. We show that Theorem~\ref{thm:rv:quantitative} is a consequence of Theorem~\ref{thm:soundness of atomic monitor}, interval arithmetic and union bounds. 

% \begin{lemma}
%     For every atomic function $\rvvar\in \rvvarsym$ of arity $n$, every POMC $\rviproc\in \rvsproc$, and every sequence $\seq{\cO}\in\aS^*$ with $t=|\seq{\cO}|$ we have 
% \begin{align*}
%     \eval{\rviprop}_{\interp_s}(\rviproc,\seq{\cO})= \expe_{\seq{\rS}\sim \rviproc}(\rvvar(\seq{\rO}_{1:n-1}))
% \end{align*}
% \end{lemma}
% \begin{proof}
%     \todo{}
% \end{proof}

\paragraph{Unbiasedness of the estimator}
We start our work toward applying Theorem \ref{thm:mcdiarmid} by establishing the unbiasedness of the estimator $\sem{\rvvar}(\cdot)$. Here we leverage the fact that the hidden Markov model is in the stationary distribution.
\begin{lemma}
	\label{lemma:expectation atom}
	Let $\rviproc\in\rvsproc$ with stationary distribution $\st$, let $\rvvar\colon\aO^{n}\to[a,b]$ be a function for fixed $n$, $a$, and $b$. For $\seq{\rO} \sim \rviproc$ let $\seq{\rS}$ be the underlying MC. Then for every $t\geq n $ we have $\expe_{\st}^{\pomc}(\sem{\rvvar}(\seq{\rO}_t))=\sem{\rvvar}(\seq{\rO})$.
\end{lemma}
\begin{proof}
	Let $N = t - n +1$. By definition and linearity of expectation, 
	\begin{align*}
		\expe_{\st}^{\pomc} (\sem{\rvvar}(\seq{\rO}_t))  =  \expe_{\st}^{\pomc} \left(\frac{1}{N} \sum_{i=1}^{N} \rvvar(\seq{\rO}_{i: i+n-1}) \right)  =  
        \frac{1}{N} \sum_{i=1}^{N} \expe_{\st}^{\pomc} \left(\rvvar(\seq{\rO}_{i: i+n-1}) \right) .
	\end{align*}
    By stationarity we know that for every $i\in [N]$ that $\rW_i \sim \st\trm^{i-1} = \st$. Therefore, 
    \begin{align*}
        \frac{1}{N} \sum_{i=1}^{N} \expe_{\st}^{\pomc} \left(\rvvar(\seq{\rO}_{i: i+n-1}) \right)   =   \frac{1}{N} \sum_{i=1}^{N} \expe_{\st}^{\pomc} \left(\rvvar(\seq{\rO}_{n}) \right)  =  \expe_{\st}^{\pomc} \left(\rvvar(\seq{\rO}_{n}) \right) =  \sem{\rvvar}(\seq{\rO}).
    \end{align*}
    
\end{proof}

As a simple corollary we obtain the same result for the estimators of the atoms $\pr$, i.e., for the probability sequence fragment.

\begin{corollary}
	\label{lemma:probability atom}
	Let $\rviproc\in\rvsproc$ with stationary distribution $\st$, let $\specset\subset \aO^*$ be a set of bounded length observation sequences with bound $n$, $\rvvar:\aO^n\to \set{0,1}$ be the indicator function of the set $\overline{\specset}$, let $\seq{\rW} \sim \rviproc$, and $t\geq n$.
	Then $\expe_{\st}^{\rviproc}(\sem{\rvvar}(\seq{\rO}))=\sem{\pr(\specset)}(\seq{\rW})$.
\end{corollary}
\begin{proof}
	This follows directly from Prop.~\ref{lemma:expectation atom} and the fact that $\expe_{\st}^{\rviproc}(\rvvar(\seq{\rO}_n))=\prob_{\st}^{\rviproc}(\seq{\rO} \in \overline{\specset}).$
\end{proof}

\noindent

% Having established the unbaisdness of the estimator, we are one step closer to applying a McDiarmid-style inequality for the construction of confidence intervals. 
% What remains to be shown is the bounded difference property. 
% \

\paragraph{Establishing the Bounded Difference Property.}
To apply Theorem \ref{thm:mcdiarmid} it is required that the observation labels should not interfere with the so-called bounded difference property of the function. 
We start by upper-bounding the maximal deviation of the atomic function estimator on almost identical inputs. 
\begin{lemma}\label{lemma:bounded difference}
	Let $\rvvar\colon\aO^n\to[a,b]$ be a function with fixed $n$, $a$, and $b$, $t\geq n$ be a constant, $\seq{\cX}, \seq{\cY} \in \aO^{\omega}$ be a pair of observation sequences such that $\seq{\cX}_t$ and $\seq{\cY}_t$ differ at position $k\in [t]$ only. 
	Then 
	\begin{align*}
		|\sem{\rvvar}(\seq{\cX}_t)-\sem{\rvvar}(\seq{\cY}_t) | \leq \frac{b-a}{N} \min(k; t-(k-1); n;  t-(n-1)).
	\end{align*}
\end{lemma}
\begin{proof}
    Let $N = t-(n-1)$ By definition we have 
    \begin{align*}
        \sem{\rvvar}(\seq{\cX}_t)-\sem{\rvvar}(\seq{\cY}_t) = \frac{1}{N}\sum_{i=1}^{N} \rvvar(\cX_i, \dots, \cX_{i+n-1}) - \rvvar(\cY_i, \dots, \cY_{i+n-1}).
    \end{align*}
    The value of $\rvvar(\cX_i, \dots, \cX_{i+n-1}) - \rvvar(\cY_i, \dots, \cY_{i+n-1})$ is $b-a$ if $k \in [i; i+n-1]$ and $0$ otherwise. Therefore, we get 
    \begin{align*}
        \sem{\rvvar}(\seq{\cX}_t)-\sem{\rvvar}(\seq{\cY}_t) = \frac{1}{N} \sum_{i\in I}b-a  =  \frac{b-a}{N} \min(k; t-(k-1); n;  t-(n-1)).
    \end{align*}
    where $I=[\max(0,k-(n-1)), \min(k+(n-1), N)]$. The interval $I$ can be encoded using the $\min(n; t-(n-1); k; t-(k-1) )$.
    The first term takes into account that, if the string is long enough we know that the symbols at position $k$ can be evaluated at most $n$ times, once for each input position of the function. The second term takes into account that, if the string is short we know that the symbols at position $k$ can be evaluated at most $t-(n-1)$ times. 
    However, this is an over approximation if $k$ is at the beginning or at the end of the word. 
	% Since the Hamming distance is $1$, $\seq{\cX}_t$ and $\seq{\cY}_t$ differ only in one symbol. 
	% We know that $\rvvar$ is evaluated on a substring of length $n$. Hence, if the word is sufficiently long only $n$ evaluations of $\rvvar$ in $\sem{\rvvar}$ (i.e., only $n$ terms in the sum in Eq.~\ref{equ:point estimator}) will be affected, while if the word is short, then only $t-(n-1)$ evaluations of $\rvvar$ will be affected. Therefore, the evaluation of $\rvvar$ can differ in the worst case by at most $\frac{\min(t-(n-1),n)}{t-(n-1)} \cdot (b-a)$.
\end{proof}

\paragraph{Confidence Interval.}
The confidence intervals generated by McDiarmid-style inequalities for Markov chains tighten in relation to the mixing time of the Markov chain. 
This means the slower a POMC mixes, the longer the monitor needs to watch to be able to obtain an output interval of the same quality. This relationship is given in Theorem $\ref{thm:mcdiarmid}$ taken from \cite{paulin2015concentration}. We translate this into a confidence interval for \vbse atoms. 
\begin{lemma}
	\label{lemma:bounded_vale}
 	Let $\rviproc\in\rvsproc$ with stationary distribution $\st$, $\rvvar:\aO^n\to[a,b]$ be a function for a fixed $n$, $a$, and $b$, $t\geq n$ be a constant, and let $\seq{\rO} \sim \rviproc$. Then for every $\conf\in (0,1)$
	\begin{align*}
		\prob\left( |\sem{\rvvar}(\seq{\rO}_t) - \expe_{\st}^{\pomc}(\sem{\rvvar}(\seq{\rO}_t)) |\geq \sqrt{\ln(2/\conf)\cdot \frac{t\cdot n^2\cdot (b-a)^2 \cdot 9 \cdot \tau_{mix}}{2 (t-(n-1))^2}} \right) \leq \conf.
	\end{align*}
\end{lemma}
\begin{proof}
	From Lemma~\ref{lemma:bounded difference}  we know that for every $\seq{\cX}_t, \seq{\cY}_t \in \aO^t$
	\begin{align*}
		|\sem{\rvvar}(\seq{\cX}_t)-\sem{\rvvar}(\seq{\cY}_t) |  &\leq \sum_{i=1}^{t} \frac{b-a}{N} \min(i; t-(i-1); n;  t-(n-1)) \cdot \indi{\cX_i \neq \cY_i'} \\
       & \leq \sum_{i=1}^{t} \frac{n}{N}\cdot  (b-a)\cdot \indi{\cX_i \neq \cY_i'}
	\end{align*}
	we conclude that
	\begin{align*}
		\left(\sqrt{\sum_{i=1}^t  \left(\frac{n}{N}\cdot  (b-a)\right)^2}\right)^2 =  \frac{t\cdot n^2}{(t-(n-1))^2} \cdot (b-a)^2
	\end{align*}
	as required by Theorem \ref{thm:mcdiarmid}. This gives us
    \begin{align*}
		\prob\left( |\sem{\rvvar}(\seq{\rO}_t) - \expe_{\st}^{\pomc}(\sem{\rvvar}(\seq{\rO}_t)) |\geq \varepsilon \right) \leq 2\exp\left(- \frac{2\cdot \varepsilon^2 (t-(n-1))^2}{t\cdot n^2 \cdot (a-b)^2\cdot 9 \cdot \tau_{mix}}\right).
	\end{align*}
    from that the required bound follows from simple arithmetic. 
\end{proof}
From this we construct uniform confidence bounds by performing a union bound over all possible time points. We show this standard result in general form in the lemma below.
% \begin{theorem}[\cite{jerison2013general}]
% 	\label{thm:mcdiarmid}
% 	Let $\seq{\rS}\coloneqq \rS_1, \dots, \rS_n$ be an ergodic Markov chain with countable state space $\aS$, unique stationary distribution $\st$, and finite mixing time bounded by $\taumix$.
% 	Suppose that some function $f: \aS^n \to \RN$ with artiy $n$ satisfies s.t. $\forall \seq{x},\seq{y}\in\aS$ 
% 	\begin{align*}
% 		f(\seq{x})- f(\seq{y}) \leq \sum_{i=1}^n  c_i \indi{x_i \neq y_i}
% 	\end{align*}
% 	for some $c\in \mathbb{R}^n$ with positive entries. Then for any $\varepsilon>0$
% 	\begin{align*}
% 		\prob\left(\left| f(\seq{\rS}) - \expe(f(\seq{\rS}))\right|\geq \varepsilon \right) \leq 2\exp\left( - \frac{2\varepsilon^2}{\sqrt{\sum_{i=1}^n c_i^2}^2 \cdot 9\cdot \taumix} \right)
% 	\end{align*}

% \end{theorem}

\begin{lemma}
    \label{lemma:confseq-naive}
    Let $\seq{\rX}$ be any stochastic process over $\aX$, let $\contrfoo\colon \RN_{\geq0}\to \RN_{\geq0}$ increasing s.t.\ $\sum_{i=0}^{\infty}\frac{1}{\contrfoo(i)}\leq 1$,
    ,let $\conferror\colon (0,1)\times \NN\to \RN_{>0}$ be an error function such that for all $\conf\in(0,1)$
    \begin{align*}
        \exists t\in\NN : \prob(|\rX_t-\expe(\rX_t)|\geq \conferror(\conf, t))\leq \conf
    \end{align*}
    then
        \begin{align*}
        \prob(\exists t\in\NN :  |\rX_t-\expe(\rX_t)|\geq \conferror(\conf\cdot \contrfoo(t)^{-1},t))\leq \conf.
    \end{align*}
\end{lemma}
\begin{proof}
    First, we set $\conf_t\coloneqq \conf\cdot \contrfoo(t)^{-1}$ for each $t\in\NN$ and then apply union bound to obtain 
    \begin{align*}
    &\prob(\exists t\in\NN : |\rX_t-\expe(\rX_t)|\geq \conferror(\conf_t, t))\leq \sum_{t=0}^{\infty}\conf_t \\
    \iff
    &\prob(\exists t\in\NN : |\rX_t-\expe(\rX_t)|\geq \conferror(\conf\cdot\contrfoo(t)^{-1} , t))\leq \conf \sum_{t=0}^{\infty} \contrfoo(t)^{-1} \leq \conf .
\end{align*}
\end{proof}
We can combine Lemma~\ref{lemma:bounded_vale} and Lemma~\ref{lemma:confseq-naive} to obtain the uniform bound.
\begin{lemma}
	\label{lemma:unif_bounded}
 	Let $\rviproc\in\rvsproc$ with stationary distribution $\st$, $\rvvar:\aO^n\to[a,b]$ be a function for a fixed $n$, $a$, $b$, and let $\seq{\rO} \sim \rviproc$. Then for every $\conf \in (0,1)$
	\begin{align*}
		\prob\left( |\sem{\rvvar}(\seq{\rO}_t) - \expe_{\st}^{\pomc}(\sem{\rvvar}(\seq{\rO}_t)) |\geq \sqrt{\ln( \pi^2t^2/(3\conf))\cdot \frac{t\cdot n^2\cdot (b-a)^2 \cdot 9 \cdot \tau_{mix}}{2 (t-(n-1))^2}} \right) \leq \conf.
	\end{align*}
\end{lemma}
\begin{proof}
    First we set for every $\conf'\in (0,1)$ and every $t\geq n$
    \begin{align*}
        \rvconferrors(\conf', t) \coloneqq\sqrt{\ln( 2/\conf')\cdot \frac{t\cdot n^2\cdot (b-a)^2 
        \cdot 9 \cdot \tau_{mix}}{2 (t-(n-1))^2}} .
    \end{align*}
    Next, we start by using the pointwise bound given by Lemma~\ref{lemma:bounded_vale} with $\conf_t\in (0,1)$ and perform a union bound over all points in time $t\in \pNN$, i.e.,
		\begin{align*}
		\prob\left( \forall t\geq n\colon |\sem{\rvvar}(\seq{\rO}_t) - \expe_{\st}^{\pomc}(\sem{\rvvar}(\seq{\rO}_t)) |\geq \rvconferrors(\conf_t, t)  \right) \leq \sum_{t=n}^{\infty}\conf_t.
	\end{align*}
    Next we need to set $\conf_t$ such that the series converges.
    We choose $\conf_t\coloneqq \frac{6}{\pi^2}\frac{\conf}{t^2}$ as a consequence we get 
	\begin{align*}
     \sum_{t=n}^{\infty}\conf_t \leq  \sum_{t=1}^{\infty}\conf_t 
     =\conf \cdot  \frac{6}{\pi^2} \cdot   \sum_{t=1}^{\infty}\frac{1}{t^2} = \conf.
	\end{align*}
    This concludes the proof.
\end{proof}

\paragraph{Correctness.}
We now can establish the correctness of the atomic monitor as stated in Theorem \ref{thm:soundness of atomic monitor}. Theorem \ref{thm:rv:quantitative} follows directly from Theorem \ref{thm:soundness of atomic monitor} using interval arithmetic and the union bound.

\begin{proof}[Proof of Thm.~\ref{thm:soundness of atomic monitor}]
	The soundness claims follow as a consequence of Lem.~\ref{lemma:bounded_vale}, Prop.~\ref{lemma:expectation atom}, and 
    Lem.~\ref{lemma:unif_bounded}.
    . By combining Theorem \ref{thm:mcdiarmid} and Corollary 2.17 from \cite{paulin2015concentration} we obtain the result for POMCs.
	The computational complexity is dominated by the use of the set of $n$ registers $\seq{\cO}$ to store the last $n$ sub-sequence of the observed path: allocation of memory for $\seq{\cO}$ takes $n$ space, and, after every new observation, the update of $\seq{\cO}$ takes $n$ write operations (Line~\ref{step:shift window}).
\end{proof}

\section{Monitoring \pse-s on MCs}
\label{sec:cav}
In this section we present a pointwise and a uniformly sound monitor for \pse fragment of \vbse over irreducible and positively recurrent fully-observed MCs. 
When given a realization of an unknown MC from this class, our monitor will be able to provide converging confidence intervals for any \pse.
Intuitively, this is because irreducibility and positive recurrence guarantees that the monitor observes every state infinitely many times. We restate Assumption~\ref{ass:stationarity} for MCs below. 
\begin{assumption}\label{ass:cav}
    We assume that the unknown MC $\mc$ generating the sequence $\seq{\rS}\sim \mc$ is from the set $\cavsproc$ containing all irreducible and positively recurrent MCs.
\end{assumption}
\noindent
By limiting the expressivity of the language and restricting ourselves to perfectly observed MCs we can construct monitors that are significantly more accurate and require less additional information about the monitored process, e.g., mixing time, compared to the general purpose monitors for POMCs and \vbse as presented in Section~\ref{sec:rv}.

\paragraph{Simplification.}
For simplicity we restrict our attention to \pse without division operator, i.e., \emph{division-free} \pse. 
This is without loss of any generality, since it can be shown that every arbitrary \pse $\caviprop$ of size $n$ can be transformed into a semantically equivalent \pse of the form $\caviprop_a+\frac{\caviprop_b}{\caviprop_c}$ of size $\mathcal{O}(n^22^n)$, where $\caviprop_a$, $\caviprop_b$, and $\caviprop_c$ are all division-free \pse expressions.
From this, we can obtain individual monitors for each of the division-free sub-expressions, whose outputs can then be combined using Algorithm~\ref{alg:quant monitor} to obtain the monitor for the overall \pse $\caviprop$.
Like Section~\ref{sec:rv}, we will proceed in two steps to obtain the output of the monitor for division-free \pse expressions: first, we will show how to obtain the point estimates, and, second, we will present the calculations for the approximation error required for the confidence interval construction.

%The structure of this language in conjunction with the Markov property removes the dependency on the mixing time of the Markov chain. This allows for a significant improvement over the general monitors introduced in Section \ref{sec:rv}. 
%We start this section by formally stating the problem, present the algorithm behind our monitor in Sub-Sec.~\ref{?}, analyse its performance and proof its soundness in Sub-Sec.~\ref{?}, and present a battery of experiments in Sub-Sec.~\ref{?}.
%The steps for constructing the (general-case) \algfreq are shown in Alg.~\ref{alg:frequentist monitor}, and the detailed analysis can be found in the proof of Thm.~\ref{thm:cav:soundness}.
\subsection{Key Idea}
Suppose $\mc = \tup{\aS,\trm,\initd}$ be the monitored unknown MC and let $\seq{\rS}\sim \mc$.
For every atomic variable $\pr(j\mid i)\in V_\caviprop$ in the given \pse, where $i,j\in \aS$, we introduce a $\bernoulli{(\trm_{ij})}$ random variable $\rX^{ij}$ with the mean ${\trm_{ij}}$ unknown to us.
We make an observation $\cX^{ij}_p$ for every $p$-th visit to the state $i$ on a run, and if $j$ follows immediately afterwards then record $\cX^{ij}_p=1$ else record $\cX^{ij}_p=0$.
This gives us a sequence of observations $\seq{\cX}^{ij}=\cX^{ij}_1,\cX^{ij}_2,\ldots$ corresponding to the sequence of i.i.d.\ random variables $\seq{\rX}^{ij}= \rX^{ij}_1,\rX^{ij}_2,\ldots$.
For instance, for the run $121123$ we obtain $\seq{\cX}^{12}=1,0,1$ for the variable $\pr(2|1)$.
The heart of our algorithm is an aggregation procedure of every sequence of random variable $\set{\seq{\rX}^{ij}}_{\pr(j\mid i)\in V_{\caviprop}}$ to a single i.i.d.\ sequence $\seq{\rY}\coloneqq (Y_t)_{t\in \pNN}$ of auxiliary random variables, such that for all $t\in \NN$, $\rY_t$ is an unbiased point estimator of $\caviprop$, i.e., $\expe_{\initd}^{\mc}(\rY_t) = \sem{\caviprop}(\seq{\rS})$.
In the following, we explain how the aggregation works for different arithmetic operators in division-free \pse sub-expressions.

\paragraph{Sums and differences.}
Let $\caviprop = \pr(j\mid i) + \pr(l\mid k)$.
We combine $\seq{\rX}^{ij}$ and $\seq{\rX}^{kl}$ as $\rY_p = \rX^{ij}_p + \rX^{kl}_p$, so that $\cY_p = \cX^{ij}_p+\cX^{kl}_p$ is the corresponding observation of $\rY_p$.
Then  $\expe_{\initd}^{\mc}(\rY_p) = \expe_{\initd}^{\mc}(\rX^{ij}_p + \rX^{kl}_p) = \expe_{\initd}^{\mc}(\rX^{ij}_p) + \expe_{\initd}^{\mc}(\rX^{kl}_p) = \trm_{ij}+ \trm_{kl}$, which is equivalent to $ \sem{\pr_{ij} + \pr_{kl}}(\seq{\rS})$. A similar approach works for $\caviprop = \pr(j\mid i) - \pr(l\mid k)$.

\paragraph{Multiplications.}
For multiplications, the same linearity principle will not always work, since for random variables $A$ and $B$, $\expe(A\cdot B)=\expe(A)\cdot \expe(B)$ \emph{only if} $A$ and $B$ are statistically independent, which will not be true for specifications of the form $\caviprop = \pr(j\mid i)\cdot \pr(k\mid i)$.
In this case, the respective Bernoulli random variables $\rX^{ij}_p$ and $\rX^{ik}_p$ are dependent:
$\prob_{\initd}^{\mc}(\rX^{ij}_p=1)\cdot \prob_{\initd}^{\mc}(\rX^{ik}_p=1)=\trm_{ij}\cdot \trm_{ik}$, but $\prob_{\initd}^{\mc}(\rX^{ij}_p=1 \land \rX^{ik}_p=1)$ is always 0 (since \emph{both} $j$ and $k$ cannot be visited following the $s$-th visit to $i$).
To benefit from independence once again, we temporally shift one of the random variables by defining $\rY_p = \rX^{ij}_{2s}\cdot \rX^{ik}_{2s+1}$, with $\cY_p = \cX^{ij}_{2s}\cdot \cX^{ik}_{2s+1}$.
Since the random variables $\rX^{ij}_{2p}$ and $\rX^{ik}_{2p+1}$ are independent, as they use separate visits of state $i$, hence we obtain $\expe_{\initd}^{\mc}(\rY_p) = \trm_{ij}\cdot \trm_{ik}$.
For independent multiplications of the form $\caviprop = \pr(j\mid i)\cdot \pr(l\mid k)$ with $i\neq k$, we can simply use $\rY_p = \rX^{ij}_m\cdot \rX^{ik}_m$. In both cases we obtain $\sem{\pr_{ij} \cdot \pr_{kl}}(\seq{\rS})$

\subsection{Monitoring Algorithm}
In general, we use the ideas of aggregation and temporal shift on the syntax tree of the division-free \pse $\caviprop$, inductively. Thereby, we create a realization of the process $\seq{\rY}$ obtained by transforming $\seq{\rS}\sim\mc$. The procedure $\algfreqdivfree$ is presented in Algorithm~\ref{alg:frequentist monitor division-free} in detail.
We now show how we can use the process $\seq{\rY}$ to construct both an unbiased point estimator and a confidence interval for $\iprop$.

\paragraph{Point estimate.}
Let $\iprop$ be a \pse, let $\seq{\rS}\sim \mc$ and let $\seq{\rY}$ be the process of i.i.d. random variables with expectation $\sem{\iprop}(\seq{\rS})$ obtained from $\seq{\rS}$. 
We claim that the finitary semantic value of the identity function
evaluated over a finite realization $\seq{\cY}_t$ of the constructed $\seq{\rY}$, i.e., 
\begin{align*}
    \sem{\identity}(\seq{\cY}_t) \coloneqq \frac{1}{t} \sum_{i=1}^t \cY_i
\end{align*}
is an unbiased estimator for the semantic value of $\iprop$, i.e., $\sem{\iprop}(\seq{\rS})$. We prove this claim in Lemma~\ref{lemma:cav:iid}.

\paragraph{Confidence interval.}
The constructed process $\seq{\rY}$ consists of a sequence of i.i.d.\ random variables with mean $\sem{\iprop}(\seq{\rS})$. 
Hence, we can apply Hoeffdings inequality to construct the confidence interval required for a pointwise sound monitor and the stitching bound from Howard et al.~\cite{howard2021time} to obtain a uniformly sound.
We use $\cavcifoo$ to denote the estimation error bound, whose value is based on whether a pointwise sound or uniformly sound monitor is required.
Let $t\in \NN$ and $\sigma=b_\caviprop-a_\caviprop$ where $a_\caviprop$ and $b_\caviprop$ are the worst-case lower and upper bound of the value of expression $\iprop$.
A sound monitor is obtained by setting $\cavcifoo$ to  
\begin{align*}
\cavconferrors(t,\conf,\sigma^2)\coloneqq\sqrt{\frac{\sigma^2}{2t}\cdot \log\left(\frac{2}{\conf}\right)}
\end{align*}
A uniformly sound monitor is obtained by setting $\cavcifoo$ to
\begin{align*}
\cavconferroru(t,\conf,\sigma^2)\coloneqq 
\frac{1}{t}\sqrt{ 1.064  \max(1,\sigma^2 t)  \left(2 \cdot \log\left(\frac{\pi\log(\max(1,\sigma^2 t) )}{\sqrt{6}}\right) + \log(2/\conf)\right)}
\end{align*}

% \begin{align*}
% \cavconferroru(t,\conf,\sigma^2)\coloneqq \frac{1}{t}\left( \sqrt{2.128 \cdot \sigma^2 \cdot n \cdot \ell(t) + 1.754\cdot \sigma^4\ell(n)^2 } + 1.325\cdot\sigma^2\cdot\ell(t) \right)
% \end{align*}

% \subsection{Complexity Reduction}
% We present the detailed algorithm of this monitor, namely \algfreqdivfree, in Alg.~\ref{alg:frequentist monitor division-free}. 
% The algorithm relies on a randomisation trick to make the monitor memory efficient. 

\paragraph{Optimizing memory.} 
Consider a \pse $\caviprop=\pr(j\mid i)+\pr(l\mid k)$.
The outcome $\cY_p$ for $\caviprop$ can only be computed when both the Bernoulli outcomes $\cX^{ij}_p$ and $\cX^{kl}_p$ are available.
If at any point only one of the two is available, then we need to store the available one so that it can be used later when the other one gets available.
It can be shown that the storage of ``unmatched'' outcomes may need unbounded memory.
To bound the memory, we use the insight that a \emph{random reshuffling} of the i.i.d.\ sequence $\rX^{ij}_1, \dots, \rX^{ij}_p$ would still be i.i.d.\ and thus we do not need to store the exact order in which the outcomes appeared.
Instead, for every $\pr(j\mid i)\in V_\caviprop$, we only store the number of times we have seen the state $i$ and the edge $(i,j)$ in counters $\counter_i$ and $\counter_{ij}$, respectively.
Observe that $\counter_i\geq\sum_{\pr(k\mid i)\in V_\caviprop} \counter_{ik}$, where the possible difference accounts for the visits to irrelevant states, denoted as a dummy state $\top$.
Given $\{\counter_{ik}\}_{k}$, whenever needed, we generate in $\seq{\oldx}_i$ a \emph{random reshuffling} of the sequence of states, together with $\top$, seen after the past visits to $i$. 
From the sequence stored in $\seq{\oldx}_i$, for every $\pr(k\mid i)\in V_\caviprop$, we can consistently determine the value of $\cX^{ik}_p$ (consistency dictates $\cX^{ik}_p=1\Rightarrow \cX^{ij}_p=0$).
Moreover, we reuse space by resetting $\seq{\oldx}_i$ if no longer needed.
It is shown in the proof of Thm.~\ref{thm:cav:soundness} that the size of every $\cX_i$ can be at most the size of the expression. This random reshuffling of the observation sequences is the cause of the probabilistic transitions of the monitor.

%\vspace{-0.3cm}
\paragraph{Implementation.}
We start by transforming the given property $\caviprop$ into $\caviprop^l$ by relabeling duplicate occurrences of $\pr(j\mid i)$ using distinct labels $\pr(j\mid i)^1,\pr(j\mid i)^2,\ldots$.
The set of labeled variables in $\caviprop^l$ is $V_{\caviprop}^l$, and $|V_{\caviprop}^l|=\mathcal{O}(n)$.
Let $\mathit{SubExpr}(\caviprop)$ denote the set of every subexpression in the expression $\caviprop$, and use $[a_\caviprop, b_\caviprop]$ to denote the range of values the expression $\caviprop$ can take for every valuation of every variable as per the domain $[0,1]$.
Let $\mathit{Dep}(\caviprop)=\{i\mid \exists \pr(j\mid i)\in V_{\caviprop}\}$, and every subexpression $\caviprop_1\cdot\caviprop_2$ with $\mathit{Dep}(\caviprop_1)\cap \mathit{Dep}(\caviprop_2)\neq\emptyset$ is called a \emph{dependent multiplication}.
Implementation of $\algfreqdivfree$ in Alg.~\ref{alg:frequentist monitor division-free} has two main functions. $\mathit{Init}$ initializes the registers. $\mathit{Next}$ implements the transition function of the monitor, which attempts to compute a new observation $\cY$ for $\seq{\rY}$ (Line~\ref{alg:freq:next:eval}) after observing a new input $\cX'$, and if successful it updates the output of the monitor by invoking the $\mathit{UpdateEst}$ function.
In addition to the registers in $\mathit{Init}$ and $\mathit{Next}$ labeled in the pseudocode, following registers are used internally:
\begin{itemize}[noitemsep,topsep=0pt,parsep=0pt,partopsep=0pt]
	\item $\seq{\oldx}_i,\, i\in \mathit{Dom}(V_{\caviprop})$: reshuffled sequence of states that followed $i$.
	\item $t^{l}_{ij}$: the index of $x_i$ that was used to obtain the latest outcome of $\pr(j\mid i)^l$.
\end{itemize}
For a given concrete finite path $\seq{\cS}\in\aS^*$ of the Markov chain, Alg.~\ref{alg:frequentist monitor division-free} computes a sequence $\seq{\cY}$ (of possibly shorter length), so that if $\seq{\cS}$ is a concrete sample of the evolution $\seq{\rS}$ of the Markov chain then $\seq{\cY}$ is a sample of a sequence $\seq{\rY}=\rY_1,\rY_2,\ldots$ of i.i.d.\ random variables such that for all $p\in \pNN$  we have $\expe_{\initd}^{\mc}(Y_p) = \sem{\caviprop}(\seq{\rS})$. After which it invokes either the $\cavconferrors$ or $\cavconferroru$  depending on the required soundness guarantee, i.e., Algorithm \ref{alg:atomic monitor} implements a monitor that is pointwise sound with error bound $\cavcifoo= \cavconferrors$ and uniformly sound with error bound $\cavcifoo= \cavconferroru$ for the set $\cavsproc$.

%!TEX root=main.tex

\renewcommand{\algorithmicensure}{\textbf{Output:}}
	\begin{algorithm}
	\caption{\algfreqdivfree}
	\label{alg:frequentist monitor division-free}
	\begin{minipage}{0.4\textwidth}
		\begin{algorithmic}[1] % init
			\renewcommand{\algorithmicrequire}{\textbf{Parameters:}}
			\Require $\aX,\caviprop,\delta, \cavcifoo$
%			\renewcommand{\algorithmicrequire}{\textbf{Input:}}
%			\Require $\cX'$
			\Ensure $\verdict$
			\Function{$\mathit{Init}(\cX)$}{}
			\State $\caviprop^l\xleftarrow{\text{unique labeling}} \caviprop$
			\ForAll{$\cavvar_{ij}\in V_{\caviprop}$}
				\State $c_{ij}\gets 0$ \Comment{$\#$ of $(i,j)$}
				\State $c_{i}\gets 0$ \Comment{$\#$ of $i$}
			\EndFor
			\State $n\gets 0$ \Comment{length of $\seq{w}$}
			\State $\cX \gets \cX$ \Comment{prev. symbol}
			\State $\mu_\verdict \gets \bot$ \Comment{est. mean}
			\State $\varepsilon_\verdict\gets\bot$ \Comment{est. error}
%			\State $S\gets 0$ \Comment{est.\ variance}
			\State $\mathit{ResetX}()$ \Comment{reset $\cX_i$-s}
%			\State $\mathit{ResetComp}()$ \Comment{reset $r_\caviprop$-s}
			\State Compute $l_\caviprop,u_\caviprop$ \Comment{int.\ arith.}
			\EndFunction
		\end{algorithmic}
%		\vspace{0.5cm}
	\end{minipage}
	\begin{minipage}{0.6\textwidth}
		\begin{algorithmic}[1] % next(\cX)
			\Function{$\mathit{Next}$}{$\cX'$}
					\State $c_{\cX}\gets c_{\cX}+1$ \Comment{update counters} \label{alg:freq:next:a}
					\State $c_{\cX\cX'}\gets c_{\cX\cX'}+1$ \label{alg:freq:next:b}
						%\S%		\begin{mdframed}[
			\State $w\gets \mathit{Eval}(\caviprop^l)$ \label{alg:freq:next:eval}
			\If{$w \neq \bot$} %\Comment{failed (insuff.\ samples)}
%				\If{ $\forall r_{\caviprop'}\;.\; \left(r_{\caviprop'}=\bot \Leftrightarrow \caviprop'\equiv 1\div \xi\right)$ }\Comment{only divisions left}\label{alg:freq:next:reset X when only division is left}
%					\State $\mathit{ResetX}()$ \Comment{reset $\cX_i$-s}
%				\EndIf
%%			\EndIf
%%			\If{$w\neq \bot$}
%			\Else \Comment{new outcome of $\seq{w}$}
				\State $n\gets n+1$
				\State $\verdict\gets\mathit{UpdateEst}(w,n)$ \label{alg:freq:next:l}
				\State $\mathit{ResetX}()$
%				\State $\mathit{ResetComp}()$ \Comment{reset $r_\caviprop$-s}
			\EndIf
			\State $\cX \gets \cX'$
			\State \Return $\verdict$
			\EndFunction
		\end{algorithmic}
%		\vspace{0.1cm}
	\end{minipage}
%	\rule{5cm}{0.4pt}
%	\hfill
			\rule{\textwidth}{0.4pt}

	\begin{minipage}{0.57\textwidth}
		\begin{algorithmic}[1]
%		\vspace{0.1cm}
			\Function{$\mathit{Eval}$}{$\caviprop^l$}
				\If{$r_{\caviprop^l} = \bot$}
					\If{$\caviprop^l\equiv\caviprop^l_1+\caviprop^l_2$}
						 \State $r_{\caviprop^l}\gets\mathit{Eval}(\caviprop^l_1)+\mathit{Eval}(\caviprop^l_2)$\label{alg:eval:+}
					\ElsIf{$\caviprop^l\equiv\caviprop^l_1-\caviprop^l_2$}
						 \State $r_{\caviprop^l}\gets\mathit{Eval}(\caviprop^l_1)-\mathit{Eval}(\caviprop^l_2)$ \label{alg:eval:-}
					\ElsIf{$\caviprop^l\equiv\caviprop^l_1\cdot \caviprop^l_2$}
						\If{$\depends(V_{\caviprop_1}^l)\cap \depends(V_{\caviprop_2}^l)=\emptyset$} %\label{alg:eval:a}
							 \State $r_{\caviprop^l}\gets\mathit{Eval}(\caviprop^l_1)\cdot \mathit{Eval}(\caviprop^l_2)$\label{alg:eval:ind *}
						\Else  \Comment{dep. mult.}
							\For{$\cavvar_{ij}^l\in V_{\caviprop_2}^l\cap \depends(V_{\caviprop_1}^l)$}\label{alg:eval:dep * start}
								\State $t_{ij}^l\gets \max(\{t_{ik}^m \mid \cavvar_{ik}^m\in V_{\caviprop_1}^l\})$
								\State $t_{ij}^l\gets t_{ij}^l + 1$ \Comment{make indep.} %\KM{If max is all we need, then can we keep a running max?}
							\EndFor
							\State $r_{\caviprop^l}\gets \mathit{Eval}(\caviprop^l_1)\cdot \mathit{Eval}(\caviprop^l_2)$ \label{alg:eval:dep *}
						\EndIf
%				\ElsIf{$\caviprop^l\equiv 1\div \xi$}
%					\State $z\gets \mathit{Eval}(\xi)$
%					\If{$z\neq\bot$} \Comment{$z\in \set{0,1}$}
%						\State $b_{\caviprop^l}\gets b_{\caviprop^l}+1$ \Comment{$\#$ of $0$ till $1$}
%						\If{$z =1$}
%							 $r_{\caviprop^l}\gets b_{\caviprop^l}$
%						\EndIf
%						\ForAll{$\caviprop'\in\mathit{SubExpr}(\xi)$} \label{alg:eval:div:reset r}
%							\State $r_{\caviprop'}\gets\bot$ \Comment{$\xi$'s new round}
%						\EndFor
%					\EndIf
%					\IfThenElse
%						{$x\neq 1$}
%						{$b_\caviprop\gets b_\caviprop+1$}
%						{$r_{\caviprop^l}\gets b_\caviprop+1$}
				\ElsIf{$\caviprop^l\equiv \cavvar_{ij}^l$}%\label{alg:eval:var}
					\If{$\seq{\oldx}_{i}[t_{ij}^l+1] = \bot$}
						\State $\mathit{ExtractOutcome}(\seq{\oldx}_i,t_{ij}^l+1)$ \label{alg:eval:extract}
					\EndIf
%					\IfThenElse
%						{$\cX_{i}[t_{ij}^l+1]=j\neq \bot$}
%						{$r_{\caviprop^l}\gets 1$}
%						{$r_{\caviprop^l}\gets 0$}\label{alg:eval:var}
					\If{$\seq{\oldx}_{i}[t_{ij}^l+1]=j\neq \bot$}\label{alg:eval:var}
						\State $r_{\caviprop^l}\gets 1$
					\Else 
						\State $r_{\caviprop^l}\gets 0$
					\EndIf
				\ElsIf{$\caviprop^l\equiv c$}
					\State $r_{\caviprop^l}\gets c$ \label{alg:eval:const}
				\EndIf
				\EndIf
				\State \Return $r_{\caviprop^l}$
			\EndFunction
		\end{algorithmic}
	\end{minipage}
	\begin{minipage}{0.42\textwidth}
		\begin{algorithmic}[1]
			\vspace{0.3cm}
			\Function{$\mathit{UpdateEst}$}{$w,n$} % \Comment{The quantitative case}
%			\rule[-0.5ex]{\textwidth}{0.4pt}
				\State $\mu_\verdict \gets \frac{\mu_\verdict\cdot (n-1)+w}{n}$ \label{alg:freq:output:update mu}
%				\State $\mu_\verdict\gets \mathit{UpdMean}(\mu_\verdict,w,n)$
%				\If{$l_\caviprop>-\infty$, $u_\caviprop < \infty$} 
					\State $\varepsilon_\verdict\gets \cavcifoo(n,\conf, (b_\caviprop-a_\caviprop)^2)$\label{alg:freq:output:hoeffding}
     \Comment{See \ref{?}}
%					\mathit{Hoeffding}(l_\caviprop,u_\caviprop,\delta,n)$
%				\Else
%					\State $S\gets \mathit{UpdVariance}(S,\mu_\verdict,w,n)$
%					\State $S\gets \frac{n-2}{n-1}\cdot S + \frac{(w-\mu_\verdict)^2}{n}$\label{alg:freq:output:update S}
%					\State $\varepsilon_\verdict\gets\mathit{Chebyshev}(\mathit{S},\delta,n)$\label{alg:freq:output:chebyshev}
%				\EndIf
				\State \Return $[\mu_\verdict\pm\varepsilon_\verdict]$
			\EndFunction
		\end{algorithmic}
		\vspace{0.2cm}
		\begin{algorithmic}[1]
			\Function{$\mathit{ExtractOutcome}$}{$_i,t$} \Comment{generate a shuffled sequence of symbols seen after $i$ so that $|\seq{\oldx}_i|=t$}
				\State Let $U\gets \{j\in \Q \mid \cavvar_{ij}\in V_\caviprop\}$
				\For{$p=|\seq{\oldx}_{i}|+1,\ldots,t$}
%					\State $q\gets \mathit{Rnd}\left(\set{(u, \frac{c_{iu}}{c_i})\mid u\in U}\cup \set{(0,\frac{c_i-\sum_j c_{ij}}{c_i}}\right)$
					\State $q\gets $ \parbox[t]{0.5\linewidth} {
									$ \forall u\in U\;.\;$\\
									 $\text{pick } u \text{ w/ prob.\ } \frac{c_{iu}}{c_i}, $\\ %\,{\big \vert}\,$\\
									$ \text{pick }\top \text{ w/ prob.\ }\frac{(c_i-\sum_j c_{ij})}{c_i}$} \label{alg:extractoutcome:rand}%\Comment{pick $u$ with probability $$}
					\State $c_i\gets c_i-1$
					\If{$q\neq \top$}
						\State $c_{iq}\gets c_{iq}-1$
					\EndIf
					\State $\seq{\oldx}_i[|\seq{\oldx}_i|+1]\gets q$
				\EndFor
			\EndFunction
		\end{algorithmic}
		\vspace{0.2cm}
		\begin{algorithmic}[1]
			\Function{$\mathit{ResetX}()$}{}
				\ForAll{$i\in \dom(V_{\caviprop})$}
					\State $\seq{\oldx}_{i}\gets \emptyset$				
				\EndFor
				\ForAll{$\cavvar_{ij}^l\in V_\caviprop^l$}
					\State $t_{ij}^l\gets 0$
				\EndFor
%				\For{$k=1,\ldots,m$}
%					\State $b^m\gets 0$
%				\EndFor
%				\ForAll{$\caviprop'\in\mathit{Subformulas}(\caviprop)$}
%					\State $r_{\caviprop'}\gets \bot$
%					\If{$\caviprop'\equiv 1\div \xi$}
%						$b_{\caviprop'}\gets 0$
%					\EndIf
%				\EndFor
			\EndFunction
		\end{algorithmic}
%		\begin{algorithmic}[1]
%			\Function{$\mathit{ResetComp}()$}{}
%%				\ForAll{$i\in \dom(V_{\caviprop})$}
%%					\State $\cX_{i}\gets \emptyset$				
%%				\EndFor
%%				\ForAll{$\cavvar_{ij}^l\in V_\caviprop^l$}
%%					\State $t_{ij}^l\gets 0$
%%				\EndFor
%%				\For{$k=1,\ldots,m$}
%%					\State $b^m\gets 0$
%%				\EndFor
%				\ForAll{$\caviprop'\in\mathit{SubExpr}(\caviprop^l)$}
%					\State $r_{\caviprop'}\gets \bot$
%					\If{$\caviprop'\equiv 1\div \xi$}
%						$b_{\caviprop'}\gets 0$
%					\EndIf
%				\EndFor
%			\EndFunction
%		\end{algorithmic}
%		\begin{algorithmic}[1]
%			\Function{$\mathit{Hoeffding}$}{$l_\caviprop,u_\caviprop,\delta,n$}
%				\State \Return $\sqrt{-\frac{(u_\caviprop-l_\caviprop)^2}{2n}\cdot \ln\left(\frac{\delta}{2}\right)}$
%			\EndFunction
%		\end{algorithmic}
%		\begin{algorithmic}[1]
%			\Function{$\mathit{Chebyshev}$}{$S,\delta,n$}
%				\State \Return $\sqrt{\frac{n^2-1}{n\cdot(n\delta-1)}\cdot S}$
%			\EndFunction
%		\end{algorithmic}	
	\end{minipage}\\
%	\begin{minipage}{1\textwidth}
%		
%	\end{minipage}	
\end{algorithm}

% -------------- Monitor with division -----------------
%\begin{algorithm}
%	\caption{\algfreq}
%	\label{alg:frequentist monitor}
%	\begin{minipage}{0.54\textwidth}
%		\begin{algorithmic}[1] % init
%			\renewcommand{\algorithmicrequire}{\textbf{Parameters:}}
%			\Require $\Q,\caviprop,\delta$ 
%			\Ensure $\verdict$
%			\Function{$\mathit{Init}(\cX)$}{}
%				\State $\caviprop_a + \frac{\caviprop_b}{\caviprop_c}\xleftarrow{\text{change form}} \caviprop^l\xleftarrow{\text{labeling}} \caviprop$
%				\State $\monitor_a \gets \algfreqdivfree(\Q,\caviprop_a,\delta/3)$
%				\State $\monitor_b \gets \algfreqdivfree(\Q,\caviprop_b,\delta/3)$ \label{line:alg:freq:phi_c}
%				\State $\monitor_c \gets \algfreqdivfree(\Q,\caviprop_c,\delta/3)$
%				\State $\monitor_a.\mathit{Init}(\cX)$
%				\State $\monitor_b.\mathit{Init}(\cX)$
%				\State $\monitor_c.\mathit{Init}(\cX)$
%			\EndFunction
%		\end{algorithmic}
%	\end{minipage}
%	\begin{minipage}{0.48\textwidth}
%		\begin{algorithmic}[1]
%			\Function{$\mathit{Next}$}{$\cX'$}
%				\State $[\mu_a\pm \varepsilon_a]\gets \monitor_a.\mathit{Next}(\cX')$
%				\State $[\mu_b\pm \varepsilon_b]\gets \monitor_b.\mathit{Next}(\cX')$
%				\State $[\mu_c\pm \varepsilon_c]\gets \monitor_c.\mathit{Next}(\cX')$
%				\If{$\mu_a\neq \bot \wedge \mu_b\neq \bot \wedge \mu_c\neq \bot$}
%					\State $[\mu_\verdict\pm\varepsilon_\verdict] \gets [\mu_a\pm \varepsilon_a] + \frac{[\mu_b\pm \varepsilon_b]}{[\mu_c\pm \varepsilon_c]}$
%				\EndIf
%				\State \Return $[\mu_\verdict\pm\varepsilon_\verdict]$
%			\EndFunction
%		\end{algorithmic}
%	\end{minipage}
%\end{algorithm}

\begin{theorem}
\label{thm:cav:soundness}
    Let $(\cavsproc, \caviprop ,\conf)$  be a problem instance.
    Algorithm \ref{alg:frequentist monitor division-free} implements a monitor that is pointwise sound with error bound $\cavcifoo= \cavconferrors$ and uniformly sound with error bound $\cavcifoo= \cavconferroru$ for the set $\cavsproc$.
    Suppose the size of $\caviprop$ is $n$.
	The monitor $\monitor$ requires $\mathcal{O}(n^42^{2n})$ registers, and takes $\mathcal{O}(n^42^{2n})$ time to update its output after receiving a new input symbol.
	For the special case of $\caviprop$ containing at most one division operator (division by constant does not count), $\monitor$ requires only $\mathcal{O}(n^2)$ registers, and takes only $\mathcal{O}(n^2)$ time to update its output after receiving a new input symbol.
% 	Let $(\aS, \cavsproc, \cavsprop ,\conf)$ be a problem instance.
% 	Alg.~\ref{alg:frequentist monitor} sound with $ \cavcifoo(n,\conf,\sigma^2)=\cavconferrors(n,\conf,\sigma^2)$ and uniformly sound
%  with $\cavcifoo(n,\conf,\sigma^2)=\cavconferroru(n,\conf,\sigma^2)$.
%  \begin{align*}
% \coloneqq\sqrt{-\frac{\sigma^2}{2n}\cdot \ln\left(\frac{\conf}{2}\right)}
%  \end{align*}
 
 %  \begin{align*}
 %     \coloneqq ??
 % \end{align*}
\end{theorem}

% \begin{theorem}[Computational resources]\label{thm:frequentist:complexity}
% 	Let $(\aS, \cavsproc, \cavsprop ,\conf)$ be a problem instance and $\monitor$ be the monitor implemented using the \emph{\algfreq} routine of Alg.~\ref{alg:frequentist monitor}. Let $\caviprop\in \cavsprop$. Suppose the size of $\caviprop$ is $n$.
% 	The monitor $\monitor$ requires $\mathcal{O}(n^42^{2n})$ registers, and takes $\mathcal{O}(n^42^{2n})$ time to update its output after receiving a new input symbol.
% 	For the special case of $\caviprop$ containing at most one division operator (division by constant does not count), $\monitor$ requires only $\mathcal{O}(n^2)$ registers, and takes only $\mathcal{O}(n^2)$ time to update its output after receiving a new input symbol.
% \end{theorem}

\subsection{Technical Proofs of Correctness}
In this section we prove the correctness and the resource requirements of the algorithm in Theorem \ref{thm:cav:soundness}. 
First we establish that the sequence of random variables constructed by the monitor is indeed a sequence of i.i.d. random variables. 
Second we use this fact and apply Hoeffding's inequality to construct the error bound giving us a sound monitor. 
The uniform soundness requires the error bounds to be slightly looser than the bounds obtained through Hoeffding. We construct them using a powerful result from Howard et.al.~\cite{howard2021time}. 
We conclude this sub-section by analyzing the complexity of the algorithm. 
Here we show that all \pse can be transformed into an equivalent \pse in polynomial form. 

% First we establish that the atomic symbols $\pr(j\mid i) \in \cavvarsym$ indeed evaluates to $\trm_{ij}$.

\paragraph{Sequence construction.}
We establish that the sequence of random variables $\seq{\rY}$ by Alg.~\ref{alg:frequentist monitor division-free} is indeed a sequence of i.i.d.\ random variables with mean $\sem{\caviprop}(\seq{\rS})$.

\begin{lemma}\label{lemma:cav:iid}
	The sequence $\seq{\rY}$ constructed by Alg.~\ref{alg:frequentist monitor division-free} is a sequence of i.i.d.\ random variables with mean  $\expe_{\initd}^{\mc}(\rY_p)= \sem{\caviprop}(\seq{\rS})$.
\end{lemma}
\begin{proof}
    We split the claim into (A) $\seq{\rY}$ a sequence of i.i.d.\ random variables and (B) with mean  $\expe_{\initd}^{\mc}(\rY_p) = \sem{\caviprop}(\seq{\rS})$.
    The proof is inductive over the structure of the formula $\caviprop$.
\noindent
\textbf{Base case:} If $\caviprop$ is a variable $\pr(j\mid i)\in V$, then the sequence $\seq{\rY} $ is the same as a \emph{uniformly random reordering} of the sequence of independent Bernoulli random variables $\seq{\rX}^{ij}$: 
that $\rY_i$-s are Bernoulli follows from Line~\ref{alg:eval:var} of Subr.~$\mathit{Eval}(\pr(j\mid i))$, and that the uniformly random reordering happens follows from the invocation of the Subr.~$\mathit{ExtractOutcomes}()$ in Line~\ref{alg:eval:extract} of $\mathit{Eval}(\pr(j\mid i))$.
Since $\seq{\rX}^{ij}$ is i.i.d., hence a uniform random reordering of $\seq{\rX}^{ij}$ is also i.i.d.\ with the same distribution.
On the other hand, if $\caviprop$ is a constant $\kappa\in \RN$, then for every $s$, $\prob_{\initd}^{\mc}(\rY_p=\kappa)=1$ (Line~\ref{alg:eval:const} of Subr.~$\mathit{Eval}$).
It follows that both (A) and (B) hold in both cases.

\noindent
\textbf{Induction hypothesis:}
If $\caviprop$ is neither a variable nor a constant, then it is of the form $\caviprop\equiv\caviprop_1\odot\caviprop_2$ (recall that $\caviprop$ is assumed to be division-free), where $\caviprop_1$ and $\caviprop_2$ are two subformulas over variables $V_1$ and $V_2$ respectively (with $V_1\cup V_2\subseteq V_{\caviprop}$) and $\odot\in \{+,-,\cdot\}$.
Suppose $\seq{U} = U_1,U_2,\ldots$ and $\seq{R} = R_1,R_1,\ldots$ are sequences of independent random variables internally generated by Alg.~\ref{alg:frequentist monitor division-free} for the subformulas $\caviprop_1$ and $\caviprop_2$ respectively.
Let (A) and (B) hold for both $\seq{U}$ and $\seq{R}$.

\noindent
\textbf{Induction step:}
Given $\caviprop=\caviprop_1\odot\caviprop_2$ as defined above, we have the following possibilities:

\begin{description}

	\item[Case $\caviprop\equiv \caviprop_1 + \caviprop_2$:] It follows from Line~\ref{alg:eval:+} of Subr.~$\mathit{Eval}$ that for every $s$, we have $W_p \coloneqq U_p + R_p$.
		Then using linearity of expectation it follows that $\expe_{\initd}^{\mc}(W_p) = \expe_{\initd}^{\mc}(U_p+R_p) = \expe_{\initd}^{\mc}(U_p)+\expe_{\initd}^{\mc}(R_p = \sem{\caviprop}(\seq{\rS})$, i.e., (B) holds.
	\item[Case $\caviprop\equiv \caviprop_1-\caviprop_2$:] It follows from Line~\ref{alg:eval:-} of Subr.~$\mathit{Eval}$ that for every $s$, we have $W_p \coloneqq U_p - R_p$.
		Then using linearity of expectation it follows that $\expe_{\initd}^{\mc}(W_p) = \expe_{\initd}^{\mc}(U_p-R_p) = \expe_{\initd}^{\mc}(U_p)-\expe_{\initd}^{\mc}(R_p)= \sem{\caviprop}(\seq{\rS})$, i.e., (B) holds.
	\item[Case $\caviprop\equiv \caviprop_1 \cdot \caviprop_2$:] We distinguish between two cases:
		\begin{description}
			\item[Independent multiplication:] It follows from Line~\ref{alg:eval:ind *} of Subr.~$\mathit{Eval}$ that if $\depends(R_1)\cap \depends(R_2) = \emptyset$, then for every $s$, we have $\rY_p \coloneqq U_m\cdot R_m$.
			Then $\expe_{\initd}^{\mc}(\rY_p) = \expe_{\initd}^{\mc}(U_p\cdot R_p) = \expe_{\initd}^{\mc}(U_p)\cdot \expe_{\initd}^{\mc}(R_p) = \sem{\caviprop}(\seq{\rS})$, since $U_p$ and $R_p$ are independent.
			Hence (B) holds.
			\item[Dependent multiplication:] It follows from Line~\ref{alg:eval:dep *} of Subr.~$\mathit{Eval}$ that if $\depends(R_1)\cap \depends(R_2) \neq \emptyset$, then for every $s$, we have $\rY_p \coloneqq U_{2s}\cdot R_{2s+1}$.
			Then $\expe_{\initd}^{\mc}(\rY_p) = \expe_{\initd}^{\mc}(U_{2s}\cdot R_{2s+1}) = \expe_{\initd}^{\mc}(U_{2s})\cdot \expe_{\initd}^{\mc}(R_{2s+1}) = \sem{\caviprop}(\seq{\rS})$, since $U_{2s}$ and $R_{2s+1}$ are independent.
			Hence (B) holds.
		\end{description}

\end{description}
	Claim (A) follows in all the above cases because the elements of $\seq{\rY}$ are all i.i.d.\ as $\seq{U}$ and $\seq{R}$ are i.i.d\ sequences.
\end{proof}

\paragraph{Confidence bounds.}
Next we construct the confidence intervals around the point estimate. Because $\seq{\rY}$ is a sequence of i.i.d.\ random variables we can obtain pointwise soundness by a simple application of Hoeffding's inequality. 
To obtain the uniform bounds we use the concentration inequality for martingales stated in Theorem \ref{thrm:confseq-ramdas}. It is a simplified form of the version in Howard et.al.\cite{howard2021time}. 
\begin{theorem}[\cite{howard2021time}]
    \label{thrm:confseq-ramdas}
    Let $\seq{\rX}=(\rX_t-\expe(\rX_t\mid\seq{\rX}_{t-1}))$ be a martingale over $\aX$ with a.s.\ s.t. $|\rX_t-\expe(\rX_t\mid\seq{\rX}_{t-1})| \leq \sigma$ for all $t\in \NN$.
    Let $\contrfoo\colon \RN_{\geq0}\to \RN_{\geq0}$ increasing s.t.\ $\sum_{i=0}^{\infty}\frac{1}{\contrfoo(i)}\leq 1$, and let $\conf\in(0,1)$ we define
\begin{align*}
\varepsilon(t,\conf,\sigma^2)\coloneqq  \sqrt{ 1.064  \cdot \max(1,\sigma^2 t) \cdot  \log h(\log \max(1,\sigma^2 t) ) + \log(2/\conf) } ,
\end{align*}
then the probability of ever crossing the threshold given by $\conferror$ is bounded uniformly, i.e.,
\begin{align*}
    \prob(\exists t\in\NN :  |\rX_t-\expe(\rX_t\mid \seq{\rX}_{t-1} )|\geq \conferror(\conf,t, \sigma^2))\leq \conf
\end{align*}
\end{theorem}
Equipped with the above theorem we show the soundness of our confidence bounds.
\begin{lemma}
    Let $\seq{\rY}$ be the sequence constructed from the \pse $\varphi$ and the Markov chain $\seq{\rS}\sim \mc$. Let $a_{\varphi}$ and $b_{\varphi}$ be almost sure bounds on $\sem{\identity}(\seq{\rY})$. Then we have for every $\conf\in (0,1)$ that
    \begin{align*}
        &\forall t\in \NN: \prob_{\initd}^{\mc}\left( \sem{\caviprop}(\seq{\rS}) \in  \sem{\identity}(\seq{\rY}_t)\pm \cavconferrors(t, \conf, (b_{\caviprop}-a_{\caviprop})^2) \right) 
    \end{align*}
    and 
     \begin{align*}
        &\prob_{\initd}^{\mc}\left(\forall t\in \NN: \sem{\caviprop}(\seq{\rS}) \in \sem{\identity}(\seq{\rY}_t)   \pm \cavconferroru(t, \conf, (b_{\caviprop}-a_{\caviprop})^2) \right) \\
    \end{align*}
\end{lemma}
\begin{proof}
     From Lemma \ref{lemma:cav:iid} we know that $\seq{\rY}$ is a sequence of i.i.d.\ random variables with mean  $\expe_{\initd}^{\mc}(\rY_p) = \sem{\caviprop}(\seq{\rS})$. Moreover, we know that its value is bounded a.s. on $[a_{\caviprop},b_{\caviprop}]$. 
    At time $t$ let $n\coloneqq |\seq{\rY}|$ be the random variable representing the length of $\seq{\rY}$ obtained from $\seq{\rS}_t$.
    Using Hoeffdings inequality and we obtain
    \begin{align*}
        &\forall t\in \NN: \prob\left( \sem{\caviprop}(\seq{\rS}) \in \frac{1}{t}\sum_{i=1}^{n} \rY_i \pm \cavconferrors(t, \conf, (b_{\caviprop}-a_{\caviprop})^2) \right) 
    \end{align*}
    Hence, the Alg.~\ref{alg:frequentist monitor division-free} with is pointwise sound if we use $\cavconferrors$ to construct the confidence interval.
    To obtain uniform soundness we use the stitched bound from Theorem \ref{thrm:confseq-ramdas}. In order to apply the theorem we simply notice that the sum of bounded centred i.i.d.\ random variables is a martingale, i.e.,
    \begin{align*}
        \expe\left(\sum_{i=1}^{t+1} (\rY_i - \expe_{\initd}^{\mc}(\rY_i))\mid \seq{\rY}_{t}\right) = \sum_{i=1}^{t}  (\rY_i - \expe_{\initd}^{\mc}(\rY_i))
    \end{align*}
    bounded a.s.\ by $|b_{\caviprop}-a_{\caviprop}|$. Therefore, we obtain for $h(x)=x^2$ the guarantee 
    \begin{align*}
        &\prob\left(\forall t\in \NN: \sem{\caviprop}(\seq{\rS}) \in \frac{1}{n}\sum_{i=1}^{n} \rY_i \pm \cavconferroru(t, \conf, (b_{\caviprop}-a_{\caviprop})^2) \right) \\
    \end{align*}
    Hence, the Alg.~\ref{alg:frequentist monitor division-free} with is uniformly sound if we use $\cavconferroru$ to construct the confidence interval.
\end{proof}

\paragraph{Complexity: intermediate results.}
Next we will investigate the resource requirements of the monitor. For that we require some auxiliary results.
First, we demonstrate that any \pse can be translated into polynomial form.
\begin{lemma}
    Any \pse containing only divisions of the form $\frac{1}{\pr(j\mid i)}$ can be transformed into a polynomial.
\end{lemma}
\begin{proof}
    Let $\caviprop$ and $\caviprop'$ be two polynomials.
    Then $\caviprop + \caviprop' $ is a polynomial, i.e.\
    \begin{align*}
        \caviprop + \caviprop' &= \sum_{k=1}^p \kappa_k\prod_{i=1,j=1}^{N} \pr(j\mid i)^{d_{ij}^k} + \sum_{l=1}^q \kappa_l'\prod_{i=1,j=1}^{N} \pr(j\mid i)^{{d'}_{ij}^l}
    \end{align*}
    Then $\caviprop \cdot \psi $ is a polynomial, i.e.\
    \begin{align*}
        \caviprop \cdot \caviprop' &= \sum_{k=1}^p \kappa_k \prod_{i=1,j=1}^{N} \pr(j\mid i)^{d_{ij}^k} \cdot \sum_{l=0}^q \kappa_l'\prod_{i=1,j=1}^{N}  \pr(j\mid i)^{{d'}_{ij}^l}\\
        &=\sum_{k=1}^p \sum_{l=1}^q \kappa_k \kappa_l'\prod_{i=1,j=1}^{N} \pr(j\mid i)^{d_{ij}^k} \cdot \prod_{i=1,j=1}^{N} \pr(j\mid i)^{{d'}_{ij}^l}
    \end{align*}
    Trivially the leafs, i.e. $\pr(j\mid i)$ or $1 \div \pr(j\mid i)$ of the formula tree are polynomials. Hence, by starting from the leafs and propagating the transformations upwards we obtain a formula in polynomial form.
\end{proof}
\noindent
Next we show that the size of the resulting polynomial grows exponentially.
For that we show the following lemma first.
\begin{lemma}
    \label{lemma:formula_m}
    Let $m\in\mathbb{N}$ s.t. $m\geq 2$, let 
    \begin{align*}
        \caviprop_m\coloneqq \prod_{i=0}^{m-1} (q_{2i} + q_{2i+1})
    \end{align*}
    containing $2m$ unique variables then its polynomial form is of size $2^{2m+1}-1$.
\end{lemma}
\begin{proof}
    For some $m\in \mathbb{N}$, we show by induction that the polynomial form of $\caviprop_m$ is 
$\sum_{i=0}^{2^{m}-1} \prod_{j=0}^{m-1} q_{x_{ij}}$
    where $x_{ij} \in \{0, \dots, m\}$. First,
    \begin{align*}
        \caviprop_2= (q_0+q_1)\cdot (q_2+q_3) = q_0q_2+q_0q_3+q_1q_2+q_1q_3 =  \sum_{i=0}^{2^2-1} \prod_{j=0}^{2-1} q_{x_{ij}}
    \end{align*}
    Second, by IH
    \begin{align*}
        \caviprop_{k+1}&= \caviprop_{k}\cdot (q_{2k}+q_{2k+1}) =  \left(\sum_{i=0}^{2^{k}} \prod_{j=0}^{k-1} q_{x_{ij}} \right) \cdot (q_{2k}+q_{2k+1}) =  \\
        &=  \sum_{i=0}^{2^{k}-1} q_{2k} \prod_{j=0}^{k-1} q_{x_{ij}}  + \sum_{i=0}^{2^{k}-1} q_{2k+1} \prod_{j=0}^{k-1} q_{x_{ij}} = \sum_{i=0}^{2^k} \prod_{j=0}^k q_{x_{ij}}
    \end{align*}
    Therefore, the sum consists of $2^m-1$ additions symbols and $2^m$ products, with each product containing $m$ variable symbols and $m-1$ product symbols.
    Thus we obtain $2^m(m+m-1)+2^{m}-1 = 2^{m+1}m -1$.
\end{proof}

\begin{lemma}
    \label{lemma:nf_blowup}
    Every \pse $\caviprop$ can be reduced to a \pse $\psi$ that is in polynomial form, such that $\caviprop$ and $\psi$ are semantically equivalent.
	If the size of $\caviprop$ is $n$, then the size of $\psi$ is bounded by $\mathcal{O}\left(n2^{\frac{n}{2}}\right)$.
    % Let $\caviprop^{NF}$ be $\caviprop$. 
    % Let $K \coloneqq |\caviprop|$, then $|\caviprop^{NF}|\leq 2^{\frac{K-1}{2}}(K-1)-1$. This bound is tight if $\variables(\caviprop)\leq |\aS|$. \KK{Show that this is the worst case.}
\end{lemma}
\begin{proof}
    Choose $m$ s.t. $4(m-1)-1\leq n\leq 4m-1$. Then $|\caviprop_{m-1}|<\caviprop\leq |\caviprop_{m}|$ and from Lemma \ref{lemma:formula_m} that $\caviprop_{m}$'s polynomial form is smaller than $2^{m+1}m -1$.
\end{proof}

\paragraph{Correctness and Complexity.}
With this we can already prove the correctness and the complexity of the monitor specified in Alg.~\ref{alg:frequentist monitor division-free}.
\begin{proof}[Proof of Theorem~\ref{thm:cav:soundness}]
    The last remaining step is to show that our algorithm can handle division in a sound manner.  Suppose $\caviprop$ has at least one division operator.
	After assigning distinct labels to the repeatedly occurring variables in $\caviprop$ to form $\caviprop^l$, we convert $\caviprop^l$ to the form $\caviprop_a+\frac{\caviprop_b}{\caviprop_c}$, where $\caviprop_a$, $\caviprop_b$, and $\caviprop_c$ are division-free.
	We employ the monitors $\monitor_a$, $\monitor_b$, and $\monitor_c$ to estimate the values $\caviprop_a(\caviproc)$, $\caviprop_b(\caviproc)$, and $\caviprop_c(\caviproc)$, respectively, and the correctness of the outputs of the respective monitors follow from the soundness of the confidence intervals.
	Note that the interval estimates $\caviprop_a(\caviproc)$, $\caviprop_b(\caviproc)$, and $\caviprop_c(\caviproc)$ are each with confidence $\delta/3$. Now using standard interval arithmetic and the union bound over the three sub-results we obtain the desired confidence interval for $\conf$.
    Details about the union bound method can be found in the paper by Albarghouthi et al.~\cite{albarghouthi2019fairness}).

    We now continue with the complexity of the monitor.
    First, let us assume that the \pse $\caviprop$ is division-free, so that effectively \algfreq reduces to \algfreqdivfree.
In this case, the number of registers for $\{\counter_{ij}\}$, $\{\counter_i\}$, $\{r_\caviprop\}$, $\{t_{ij}^l\}$, $\{b_{\caviprop'}\}$ can be at most $\mathcal{O}(n)$, where $n$ is the number of terms in the formula $\caviprop$.
The total number of registers is dominated by the total space occupied by all the $\oldx_i$-s (each location of the array $\oldx_i$ is interpreted as a register).
We first argue that every $\oldx_i$ can grow up to size at most $\mathcal{O}(n)$.
%Let us assume that $\caviprop$ does not contain any division operator.
Moreover, the most amount of registers in $\oldx_i$ are required when the operation involved is a dependent multiplication.
Observe that for every dependent multiplication $\caviprop=\caviprop_1\cdot\caviprop_2$ with $i\in \mathit{Dep}(\caviprop_1)\cap \mathit{Dep}(\caviprop_2)$, if $\caviprop_1$ and $\caviprop_2$ need $\mathcal{O}(\trm_1)$ and $\mathcal{O}(\trm_2)$ samples of $x_i$, then $\caviprop$ needs $\trm_1+\trm_2$ samples of $\oldx_i$.
As a result, the size of $\oldx_i$ can be at most $\mathcal{O}(n)$, and hence the total space occupied by all the $\oldx_i$ registers will be $\mathcal{O}(n^2)$.
%When $\caviprop$ contains division, in theory we will need unbounded number of samples (bounded in expectation) of every variable $\pr(j\mid i)$ appearing in the divisor.
%In such cases, we first continue our usual aggregation procedure until only the division subformulas are remaining to be assigned a sample (i.e., a value to the respective $r_{\caviprop'}$).
%At this phase, the $x_i$-s can grow up to at most $\mathcal{O}(n)$ length because of the reason explained above.
%After this phase, we reset $x_i$ after every of evaluation of the property $\caviprop$ (Line~\ref{alg:freq:next:reset X when only division is left} of $\mathit{Next}(\sigma')$), thus limiting the size of $x_i$-s to at most $\mathcal{O}(n)$, where the worst case bound $\mathcal{O}(n)$ is achieved for formulas of the form $\caviprop=\frac{1}{\cavvar_{i1}\cavvar_{i2}\ldots \cavvar_{in}}$, because every single execution of $\mathit{Eval}(\caviprop)$ would require $\mathcal{O}(n)$ samples of $x_i$.
%Thus, for every $i$, $x_i$ will take up to $\mathcal{O}(n)$ space, which gives us the total space bound of $\mathcal{O}(n^2)$ on the set of $x_i$-s.

The transition function of the monitor is implemented by the Subr.~$\mathit{Next}$ and the output function is implemented by the Subr.~$\mathit{UpdEst}$.
The computation time of the transition function is dominated by the $\mathit{Eval}(\caviprop^l)$ operation in Line~\ref{alg:freq:next:eval}.
Observe that computation time of $\mathit{Eval}(\cdot)$ is dominated by the computation time of dependent multiplications, where every dependent multiplication requires $\mathcal{O}(n)$ operations to shift every $t_{ij}^l$ by one place (there are $\mathcal{O}(n)$-many $t_{ij}^l$-s).
Thus, in the worst case there will be $\mathcal{O}(n)$ dependent multiplications, giving us the $\mathcal{O}(n^2)$ bound on the computation time.
The Subr.~$\mathit{UpdEst}$ requires constant amount of memory and runs in constant time, which can be easily observed from the pseudocode, giving us the overall quadratic bounds on the computation time and memory.
This proves the last part of the theorem.

When $\caviprop$ contains division, then first $\caviprop$ is converted to the form $\caviprop_a+\frac{\caviprop_b}{\caviprop_c}$, where $\caviprop_a$, $\caviprop_b$, and $\caviprop_c$ are all division-free.
We will argue that if the size of $\caviprop$ is $n$, then the sizes of $\caviprop_a$ and $\caviprop_c$ are each $\mathcal{O}(n2^{\frac{n}{2}})$, and the size of $\caviprop_b$ is $\mathcal{O}(n^22^n)$.
Therefore the computation will be dominated by the invocation of \algfreqdivfree on the sub-expression $\caviprop_b$.
First, observe that any arbitrary \pse $\caviprop$ can be translated into a semantically equivalent polynomial \pse $\caviprop'$ of size $\mathcal{O}(n2^{\frac{n}{2}})$; a formal treatment of this claim can be found in Lem.~\ref{lemma:nf_blowup}.
Given the polynomial \pse $\caviprop'$, we can collect all the division-free monomials as a sum of monomials and use it as our $\caviprop_a$, whose size will be at most the size of $\caviprop'$, which is $\mathcal{O}(n2^\frac{n}{2})$.
The rest of the monomials of $\caviprop'$, the ones which contain divisions, have only single variables in the denominator (because of the syntax of \pse-s).
Hence, when we combine them in the form of a single ratio $\frac{\caviprop_b}{\caviprop_c}$, the denominator $\caviprop_c$ is a single monomial, whose size can be at most the size of the \pse $\caviprop'$, which is $\mathcal{O}(n2^{\frac{n}{2}})$.
The numerator $\caviprop_b$, on the other hand, is a sum (or difference) of $\mathcal{O}(n2^{\frac{n}{2}})$-many monomials, and every monomial can be at most $\mathcal{O}(n2^{\frac{n}{2}})$ large (because in the worst case they are of the form $\caviprop_d\cdot \caviprop_c$, where $\caviprop_d$ is some division-free term and the size of the product can be at most the size of the formula $\caviprop'$).
Therefore, the size of $\caviprop_b$ can be at most $\mathcal{O}(n^22^n)$, and the invocation of \algfreqdivfree dominates the memory and the computation time.
Since \algfreqdivfree takes $\mathcal{O}(n^2)$ time and $\mathcal{O}(n^2)$ registers for its computation for an input \pse of size $n$, hence, for the input \pse $\caviprop_b$ of size $\mathcal{O}(n^22^n)$, \algfreqdivfree would take $\mathcal{O}(n^42^{2n})$ time and $\mathcal{O}(n^42^{2n})$ registers.
\end{proof}

\subsection{Comparison with the existing Approach}
\label{subsec:comparison}
\begin{figure}
    \centering
	\includegraphics[scale=0.5]{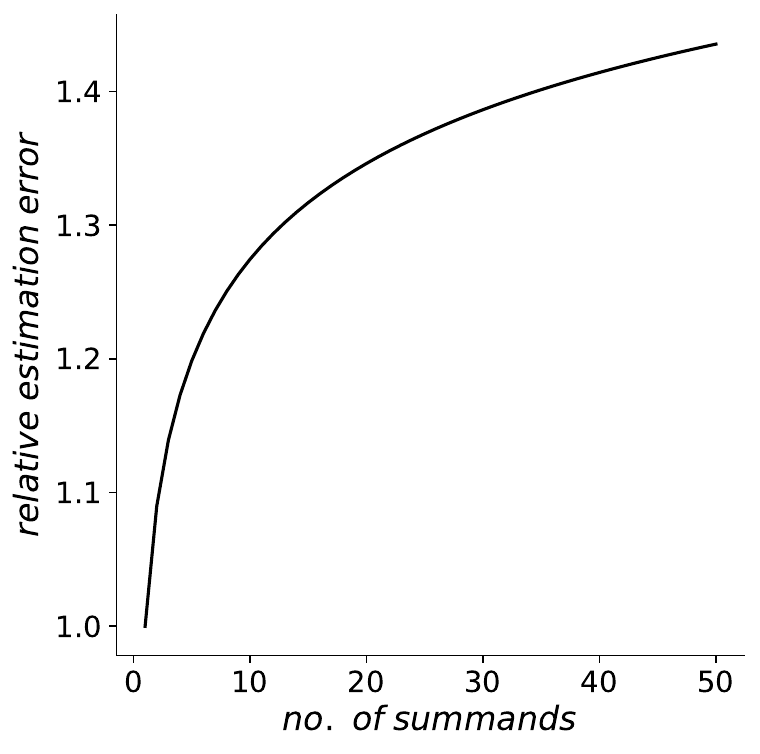}
	\caption{Variation of ratio of the est.\ error using the existing approach \cite{albarghouthi2019fairness} to est.\ error using our approach, w.r.t.\ the size of the chosen \pse.}
	\label{fig:frequentist comparison}
\end{figure}
Prior to our conference paper, the only available benchmark was published by Albarghouthi et al.~\cite{albarghouthi2019fairness}. In our conference paper, we improved upon their pointwise bound. In addition this paper improves their uniform bound exponentially.

\paragraph{Pointwise bounds.}
Suppose the given \pse is only a single variable $\caviprop=\pr(j\mid i)$. Hence, we are monitoring the probability of going from state $i$ to another state $j$.
The monitor $\monitor$ for $\caviprop$ can be constructed in two steps: (1) empirically compute the average number of times the edge $(i,j)$ was taken per visit to the state $i$ on the observed path of the Markov chain, and (2) compute the $(1-\delta)$ confidence interval using statistical concentration inequalities.
Now consider a slightly more complex \pse $\caviprop'=\pr(j\mid i)+\pr(k\mid i)$.
One approach to monitor $\caviprop'$, proposed by Albarghouthi et al.~\cite{albarghouthi2019fairness}, would be to first compute the $(1-\delta)$ confidence intervals $[l_1,u_1]$ and $[l_2,u_2]$ separately for the two constituent variables $\pr(j\mid i)$ and $\pr(k\mid i)$, respectively.
Then, the $(1-2\delta)$ confidence interval for $\caviprop'$ would be given by the sum of the two intervals $[l_1,u_1]$ and $[l_2,u_2]$, i.e., $[l_1+l_2,u_1+u_2]$; notice the drop in overall confidence due to the union bound.
The drop in the confidence level and the additional error introduced by the interval arithmetic accumulate quickly for larger \pse-s, making the estimate unusable.
Furthermore, we lose all the advantages of having any dependence between the terms in the \pse.
For instance, by observing that $\pr(j\mid i)$ and $\pr(k\mid i)$ correspond to the mutually exclusive transitions $i$ to $j$ and $i$ to $k$, we know that $\caviprop'(M)$ is always less than $1$, a feature that will be lost if we use plain merging of individual confidence intervals for $\pr(j\mid i)$ and $\pr(k\mid i)$.
We overcome these issues by estimating the value of the \pse as a whole as much as possible.
In Fig.~\ref{fig:frequentist comparison}, we demonstrate how the ratio between the estimation errors from the two approaches vary as the number of summands (i.e., $n$) in the \pse $\caviprop=\sum_{i=1}^n \pr(n|1)$ changes; in both cases we fixed the overall $\delta$ to $0.05$ ($95\%$ confidence).
The ratio remains the same for different observation lengths.
Our approach is always at least as accurate as their approach \cite{albarghouthi2019fairness}, and is significantly better for larger \pse-s.

\paragraph{Uniform bounds.}
To obtain uniform soundness Albarghouthi et al.~\cite{albarghouthi2019fairness} used the approach presented in Lemma \ref{lemma:confseq-naive}. It provides a general method for lifting any pointwise confidence bound to a uniform confidence bound. However, the resulting bound is quite loose. One commonly used method to improve this result exponentially, is called the stitching or peeling argument.
\begin{figure}
    \centering
	\includegraphics[scale=0.5]{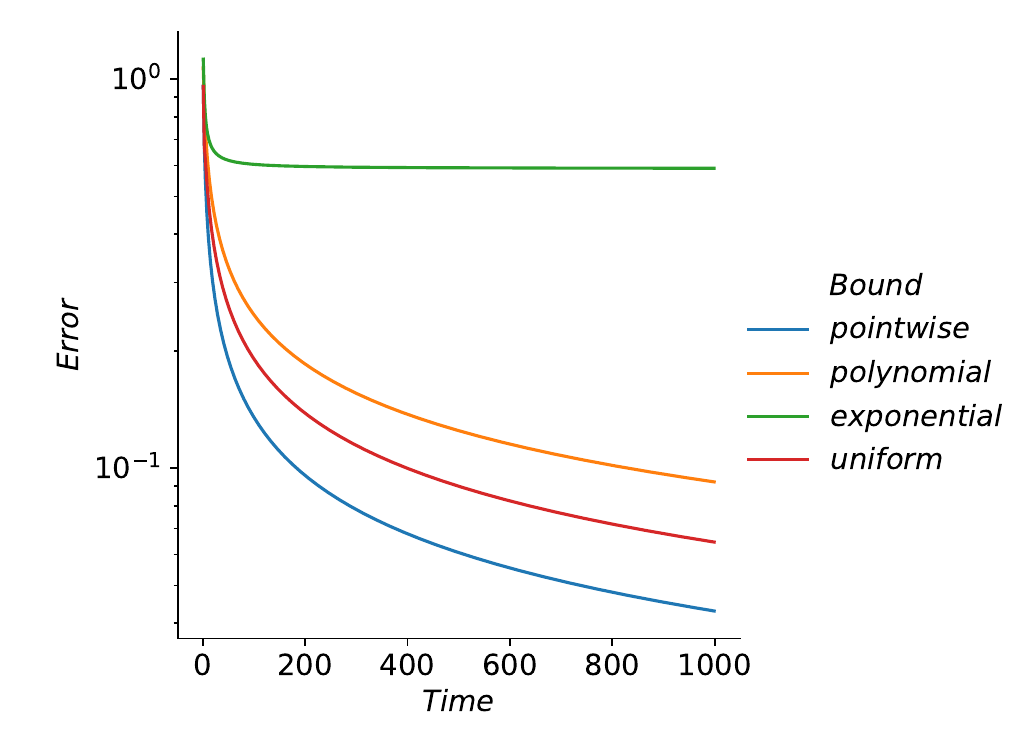}
	\caption{Confidence interval comparison for uniformly sound monitors using classical union bounds  \cite{albarghouthi2019fairness} with polynomial and exponential scaling and the martingale based approach used by us.}
	\label{fig:frequentist unif comparison}
\end{figure}
Here we segment $\NN$ into geometrically increasing intervals and effectively controlling the deviations within each interval. Using union bounds we can stitch the respective deviation bounds together. Because the intervals are geometrically spaces, the number of necessary union bounds decreases drastically. 
In fact, there is a trade-off, between the precision of the bound for the interval and the number of union bounds necessary. Lemma \ref{lemma:confseq-naive} for example, is very precise for each point in time, i.e.\ its interval length is $1$. However, due to the large number of union bounds the confidence drops drastically. 
To rectify this Theorem \ref{thrm:confseq-ramdas} by Howard et.al.\cite{howard2021time} provides a powerful machinery, striking a balance between union bounds and local deviations. 
In Figure~\ref{fig:frequentist unif comparison} we compare those two methods for a unbiased coin toss with $\delta=0.05$. The choice of $h$ (in Lemma~\ref{lemma:confseq-naive}) is not stated in Albarghouthi et al.~\cite{albarghouthi2019fairness}, thus we compare both exponential scaling, i.e., $\delta_t = \frac{\delta}{2^t}$, and polynomial scaling $\delta_t = \frac{\delta6}{ t^2\pi^2}$.

\section{Experimental Evaluation}
In this section we present the experimental evaluation of both our monitors. 
The first section concerns the evaluation of the monitor for general $\vbse$ over POMCs. 
The second section concerns the evaluation of the monitor for $\pse$ fragment over fully observed MCs. 
For simplicity we limit ourself to pointwise sound monitors. In Fig.~\ref{fig:frequentist unif comparison} we have shown that one has to pay $\mathcal{O}(\sqrt{\log \log t})$ in the best-case and $\mathcal{O}(\sqrt{\log t})$ in the worst-case to make the respective monitors uniformly sound.

\subsection{Monitoring BSEs on POMCs}
\label{sec:experiments:rv}
We implemented our monitoring algorithm in Python, and evaluate on the real-world lending example \cite{damour2020fairness} described in Ex.~\ref{ex:lending POMC} and to a synthetic example called hypercube. 
We ran the experiments on a MacBook Pro (2023) with Apple M2 Pro processor and 16GB of RAM.

\paragraph{Setup.}
The underlying POMC model of the system (hidden from the monitor) approximately follows the structure illustrated in Fig.~\ref{fig:illustrative POMC model of loan example}, with several modifications.
First, a low-probability self-loop was added to state $S$ to ensure aperiodicity.
Second, we restricted the model to only two credit score levels.
Third, the full system contains more hidden states—$171$ in total—capturing additional latent variables such as whether an individual repays or defaults on a loan.
We monitor two fairness properties: demographic parity, defined as $\varphi_{\mathsf{DP}} \coloneqq \pr(Y \mid A)-\pr(Y\mid B)$, and its absolute variant $\varphi_{\mathsf{TDP}} \coloneqq \pr(AY) - \pr(BY)$.
The first expresses the difference in loan approval rates between groups $A$ and $B$, while the second captures the difference in joint probabilities of selecting and approving individuals from each group.
Neither property can be expressed as a \pse, $\varphi_{\mathsf{DP}}$ involves conditioning, and $\varphi_{\mathsf{TDP}}$ involves absolute probabilities.

\paragraph{Experimental outcomes.}
Upon receiving new observations, the monitors for $\varphi_{\mathsf{DP}}$ and $\varphi_{\mathsf{TDP}}$ updated their outputs in $47,\si{\microsecond}$ and $18,\si{\microsecond}$ on average, respectively.
Across all runs, the update times ranged from $43,\si{\microsecond}$ to $0.2,\si{\second}$ for $\varphi_{\mathsf{DP}}$ and from $12,\si{\microsecond}$ to $3.2,\si{\second}$ for $\varphi_{\mathsf{TDP}}$, confirming their efficiency in practice.
Figure~\ref{fig:lending} displays the corresponding outputs for $\delta = 0.05$ (i.e., $95\%$ confidence intervals).

For the lending example, we used a conservative upper bound $\taumix = 170589.78$ on the mixing time (computed as in \cite{jerison2013general}), which results in slow convergence of the estimation error $\varepsilon$.
For example, to reduce $\varepsilon$ from $1$ to $0.1$, the monitor for $\varphi_{\mathsf{TDP}}$ needs about $4 \cdot 10^9$ observations.
For $\varphi_{\mathsf{DP}}$, this number rises to approximately $10^{12}$, due to the amplification of error by conditional probability calculations, which involve divisions and interval arithmetic.
We conclude that eliminating these divisions through direct estimation, along with using tighter bounds on the mixing time, would substantially improve long-term accuracy.

\subsubsection{The Hypercube Example}

\paragraph{Setup.}
As a second case study, we consider a POMC based on the random walk on the $n$-dimensional hypercube ${0,1}^n$ \cite[pp.\ 63]{levin2017markov}, illustrating the impact of mixing time bounds on monitor performance.
Each state in the POMC corresponds to a vertex of the hypercube. States beginning with $0$ and $1$ are mapped to observations $a$ and $b$, respectively.
We fix $n = 3$. At each step, the current vertex is retained with probability $\sfrac{1}{2}$, and a neighbor is selected with probability $\sfrac{1}{2n}$.
The true mixing time is tightly bounded by $\tau_{\mathsf{true,mix}} = n(\log n + \log 4)$ steps \cite[pp.\ 63]{levin2017markov}.
We monitor $\psi_{\mathsf{DP}} \coloneqq \pr(a \mid a) - \pr(b \mid b)$ and $\psi_{\mathsf{TDP}} \coloneqq \pr(aa) - \pr(bb)$.

\paragraph{Experimental outcomes.}
We assessed the confidence intervals for $\psi_{\mathsf{DP}}$ and $\psi_{\mathsf{TDP}}$ over $100$ sample runs.
The results, summarised in the third and fourth plots of Fig.~\ref{fig:lending}, clearly demonstrate that using the tighter mixing time $\tau_{\mathsf{true,mix}} = 7.45$ steps yields much sharper confidence intervals than the more conservative $\taumix = 204.94$ steps.
Compared to the lending example, the hypercube yields tighter estimates with far fewer observations, underscoring the critical role of mixing time bounds in monitoring accuracy.

\begin{figure}
	\begin{subfigure}{0.245\linewidth}
		\includegraphics[width=\linewidth]{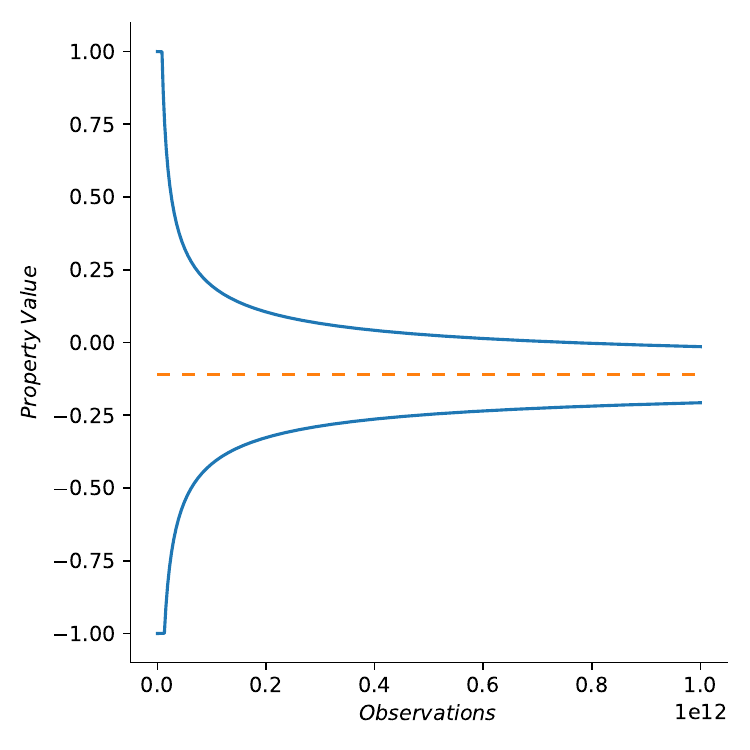}
		\phantomsubcaption
		\label{subfig:a}
	\end{subfigure}
	\begin{subfigure}{0.245\linewidth}
		\includegraphics[width=\linewidth]{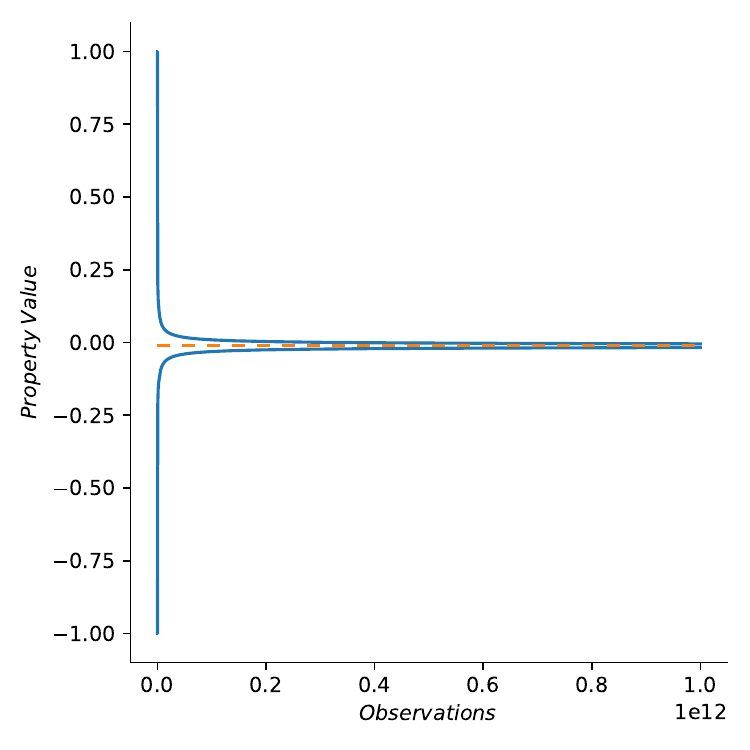}
		\phantomsubcaption
		\label{subfig:b}
	\end{subfigure}
	\begin{subfigure}{0.245\linewidth}
		\includegraphics[width=\linewidth]{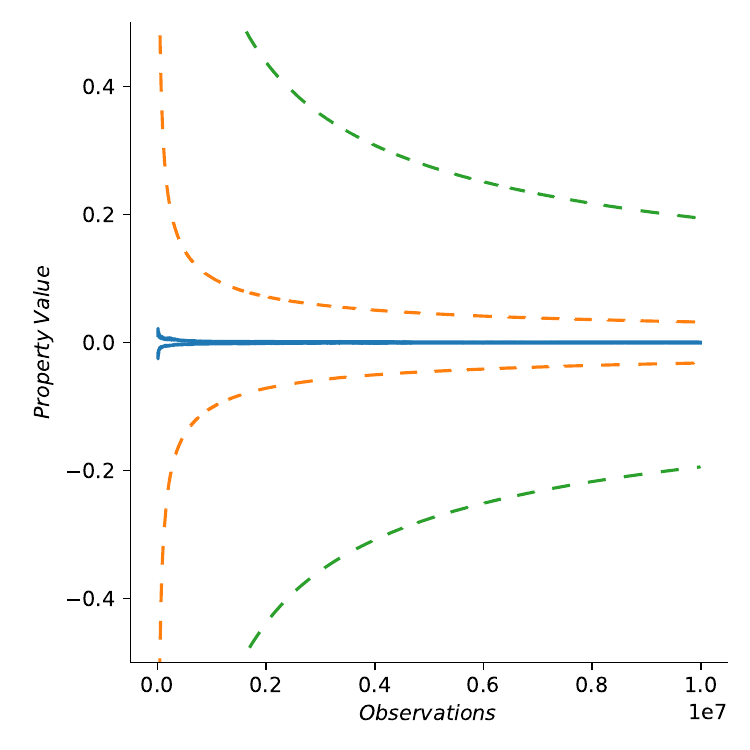}
		\phantomsubcaption
		\label{subfig:c}
	\end{subfigure}
	\begin{subfigure}{0.245\linewidth}
		\includegraphics[width=\linewidth]{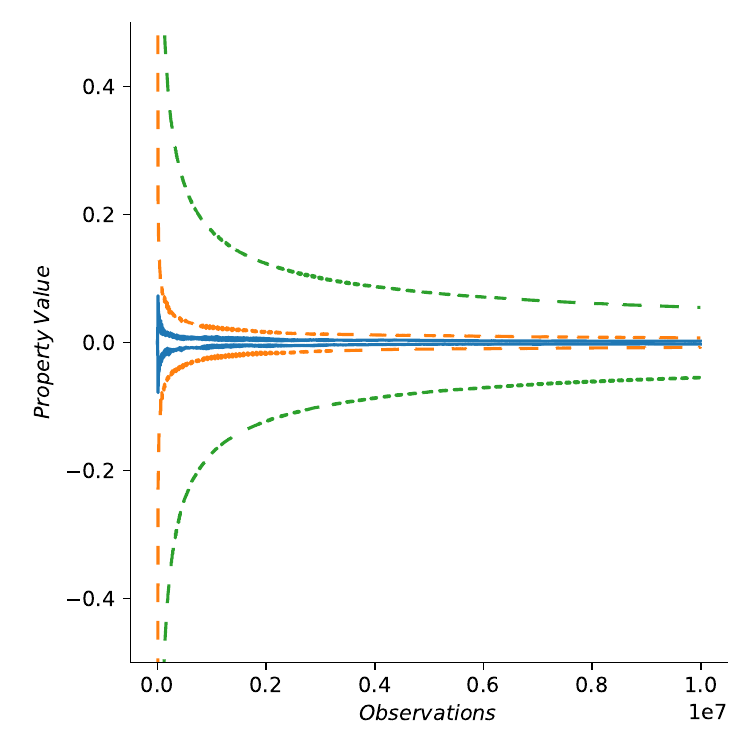}
		\phantomsubcaption
		\label{subfig:d}
	\end{subfigure}
  \caption{Monitoring $\varphi_{\mathsf{DP}}$ (first, third) and $\varphi_{\mathsf{TDP}}$ (second, fourth) on the lending (first, second) and the hypercube (third, fourth) examples. 
  The \underline{first and second} plots show the computed $95\%$-confidence interval (solid) and the true value of the property (dashed) for the lending POMC.
  In reality, the monitor was run for about $7\times 10^8$ steps until the point estimate nearly converged, though the confidence interval was trivial at this point (the whole interval $[-1,1]$), owing to the pessimistic bound $\taumix$.
  In the figure, we have plotted a projection of how the confidence interval would taper over time, had we kept the monitor running.
  The \underline{third and fourth} plots summarize the monitors' outputs over $100$ executions of the hypercube POMC. 
  The solid lines are the max and min values of the point estimates, the dashed lines are the boundaries of all the $95\%$-confidence intervals (among the $100$ executions) with the conservative bound $\taumix$ (green) and the sharper bound $\tau_{\mathsf{true\,mix}}$ (orange) on the mixing time.
  }
\label{fig:lending}
\end{figure}

\subsection{Monitoring PSEs on MCs}
\label{sec:experiments:cav}
We implemented our monitor in a tool written in Rust, and applied it to simplified real-world examples, i.e., the lending and the college admission examples, taken from the literature \cite{liu2018delayed,milli2019social}.
The generators are modelled as Markov chains (see Fig.~\ref{fig:markov chain examples})---unknown to the monitors---capturing the sequential interactions between the decision-makers (i.e., the bank or the college) and their respective environments (i.e., the loan applicants or the students), as described by D'Amour et al.~\cite{damour2020fairness}. 

\paragraph{Models.}
In Fig.~\ref{fig:markov chain examples} we show the Markov chains for the lending and the college admission examples from Sec.~\ref{sec:motivating examples}.
The Markov chain for the lending example captures the sequence of loan-related probabilistic events, namely, that a loan applicant is randomly sampled and the group information ($g$ or $\overline{g}$) is revealed, a probabilistic decision is made by the decision-maker and either the loan was granted ($gy$ or $\overline{g}y$, depending on the group) or refused ($\overline{y}$), and if the loan is granted then with some probabilities it either gets repaid ($z$) or defaulted ($\overline{z}$).
The Markov chain for the college admission example captures the sequence of admission events, namely, that a candidate is randomly sampled and the group is revealed ($g,\overline{g}$), and when the candidate is from group $g$ (truly qualified) then the amount of money invested for admission is also revealed.

\begin{figure}
		\centering
		\vspace*{-0.5cm}
		\begin{subfigure}{0.50\textwidth}
			\centering
			\def\H{1}
			\def\W{1.7}
			\begin{tikzpicture}[every node/.style={scale=0.75},scale=0.8,baseline=(init.base)]
				\node[state] (init)    at (0,0)        {\textit{init}};
				\node[state] (g-c)     at (\W^1.5, 0)  {$  \overline{y}$};
				\node[state] (g-c-A)   at (\W^2, \H)   {$ gy$};
				\node[state] (g-c-R)   at (\W^2,-\H)   {$\overline{g}y$};
				\node[state] (g-c-A-S) at (\W^2.9, \H) {$ z$};
				\node[state] (g-c-A-F) at (\W^2.9,-\H) {$\overline{z}$};

				\path[->] 
				  (g-c) edge (init)
					(g-c-A)   edge (g-c-A-S) edge (g-c-A-F)
					(g-c-A-S) edge[out=135,in=60] (init)
					(g-c-A-F) edge[out=-135,in=-60] (init)
					(g-c-R)   edge (g-c-A-S) edge	(g-c-A-F);
					
				\node[state]	(env-ch-1)	at	(\W^0.7,\H)		{$g$};
				\node[state]	(env-ch-2)	at	(\W^0.7,-\H)	{$\overline{g}$};
				\path[->]
					(env-ch-1) edge (g-c-A)	edge (g-c)
					(env-ch-2) edge (g-c-R)	edge (g-c)
					(init) edge (env-ch-1) edge	(env-ch-2);
			\end{tikzpicture}
%			\caption{}
%			\label{fig:lending:markov chain}
		\end{subfigure}
		\quad
		\begin{subfigure}{0.45\textwidth}
			\centering
			\def\H{1.3}
			\def\W{1.6}
			\begin{tikzpicture}[every node/.style={scale=0.75},scale=0.8,baseline=(init.base)]
				\node[state] (init) at (0,0) {\textit{init}};
					
				\node[state]	(env-ch-1)	at (\W^0.7,\H)	 {$g$};
				\node[state]	(env-ch-2)	at (\W^0.7,-\H)	 {$\overline{g}$};
				\node[state]  (merit-0)   at (\W^2.2, \H)  {$0$};
				\node[state]	(merit-1)	  at (\W^2.2,0.33) {$1$};
				\node       	(merit-2)	  at (\W^2.2,-0.5) {$\vdots$};
				\node[state]	(merit-3)	  at (\W^2.2,-\H)	 {$N$};
				\path[->]
					(env-ch-1) edge (merit-0)
					(init) edge	(env-ch-1) edge (env-ch-2)
					(merit-1)	edge (init)
					(merit-0)	edge (init)
					(merit-3)	edge (init)
					(env-ch-1) edge	(merit-1)	edge (merit-3)
					(env-ch-2) edge (merit-0)
					(env-ch-2) edge (merit-1)
					(env-ch-2) edge (merit-3);
 			\end{tikzpicture}
%			\caption{}
%			\label{fig:admission:markov chain}
		\end{subfigure}
		\vspace*{-0.25cm}
		\caption{Markov chains for the lending and the college-admission examples.
		(\textbf{left})~The lending example: The state $\textit{init}$ denotes the initiation of the sampling, and the rest represent the selected individual, namely, $g$ and $\overline{g}$ denote the two groups, $(gy)$ and $(\overline{g}y)$ denote that the individual is respectively from group $g$ and group $\overline{g}$ and the loan was granted, $\overline{y}$ denotes that the loan was refused,	and $z$ and $\overline{z}$ denote whether the loan was repaid or not.
		(\textbf{right})~The college admission example: The state \textit{init} denotes the initiation of the sampling, the states $g,\overline{g}$ represent the group identity of the selected candidate, and the states $\set{0,\ldots,N}$ represent the amount of money invested by a truly eligible candidate.}
		\label{fig:markov chain examples}
%		\vspace{-0.5cm}
\end{figure}
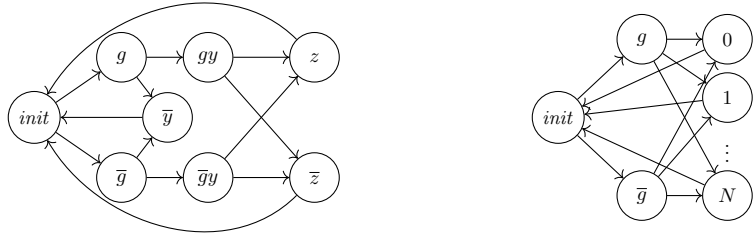

\paragraph{Setup.}
The setup of the experiments is as follows:
We created a multi-threaded wrapper program, where one thread simulates one long run of the Markov chain, and a different thread executes the monitor.
Every time a new state is visited by the Markov chain on the first thread, the information gets transmitted to the monitor on the second thread, which then updates the output.
The experiments were run on a MacBook Pro 2017 equipped with a 2,3 GHz Dual-Core Intel Core i5 processor and 8GB RAM.
The tool can be downloaded from the following url, where we have also included the scripts to reproduce our experiments: \url{https://github.com/ista-fairness-monitoring/fmlib}.

\paragraph{Experimental outcomes.}
The data in Table~\ref{tab:cav:results} indicates that our monitors are extremely lightweight, with low computational overhead across all scenarios. The average computation times per step range from $13.0 \mu s $ for simple expressions of size $1$ to $53.8 \mu s$ for more complex ones of size $19$. The number of registers required grows moderately with expression size, from 15 registers in the smallest case to 46 in the largest. In Figure~\ref{fig:experiments} we observe that the monitor's outputs are always centered around the ground truth values of the properties, empirically showing that they are always objectively correct.

\begin{table}[]
    \centering
\begin{tabular}{|m{4cm}|
					>{\centering\arraybackslash}m{1.5cm}|
					>{\centering\arraybackslash}m{1.5cm}|
					>{\centering\arraybackslash}m{1.5cm}|}
		\hline
		\multirow{2}{*}{Scenario} & \multirow{2}{*}{\shortstack{Size of\\ expression}} & \multicolumn{1}{c|}{Av.\ comp.} & \multicolumn{1}{c|}{number of}\\
		& & time/step & registers\\
		\hline
		Lending (bias) + dem.\ par. & $1$ & \num{13.0}\si{\us} & $15$ \\
		\hline
		Lending (fair) + eq.\ opp. & $5$ & \num{21.6}\si{\us} & $29$ \\
		\hline
		Admission + soc.\ burden & $19$ & \num{53.8}\si{\us} & $46$ \\
		\hline
	\end{tabular}
    \caption{The table summarizes various performance metrics.}
    \label{tab:cav:results}
\end{table}

\begin{figure}[t]
	\caption{
	The plots show the $95\%$ confidence intervals estimated by the monitors over time, averaged over $10$ different sample paths, for the lending with demographic parity (left), lending with equalized opportunity (middle), and the college admission with social burden (right) problems.	
	The horizontal dotted lines are the ground truth values of the properties, obtained by analyzing the Markov chains used to model the systems (unknown to the monitors). % on the monitored Markov chain.
	\label{fig:experiments}
	}
		\vspace{0pt}
		\includegraphics[scale=0.275]{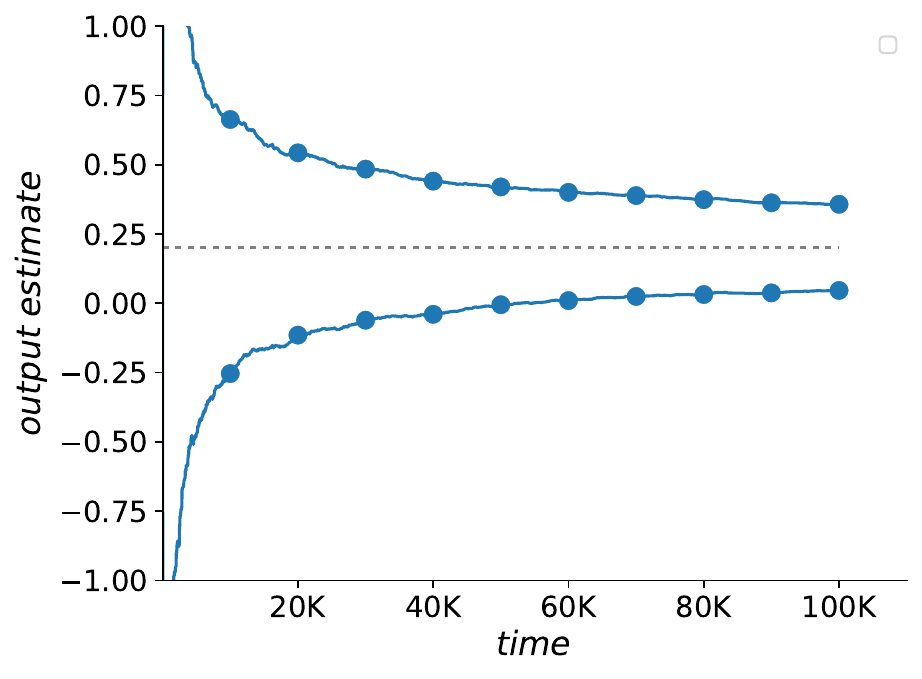}
		\includegraphics[scale=0.275]{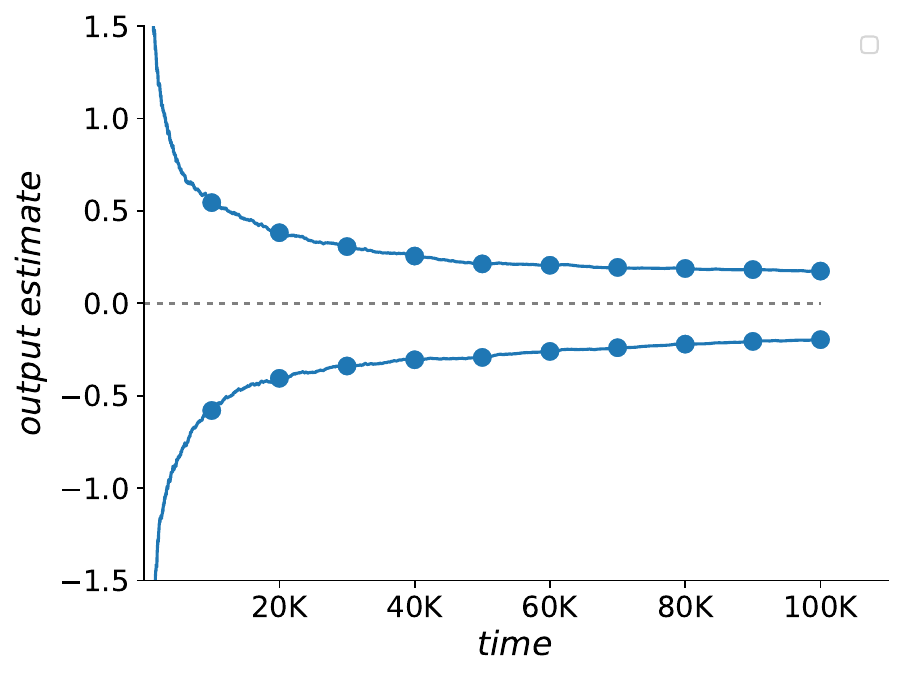}
		\includegraphics[scale=0.275]{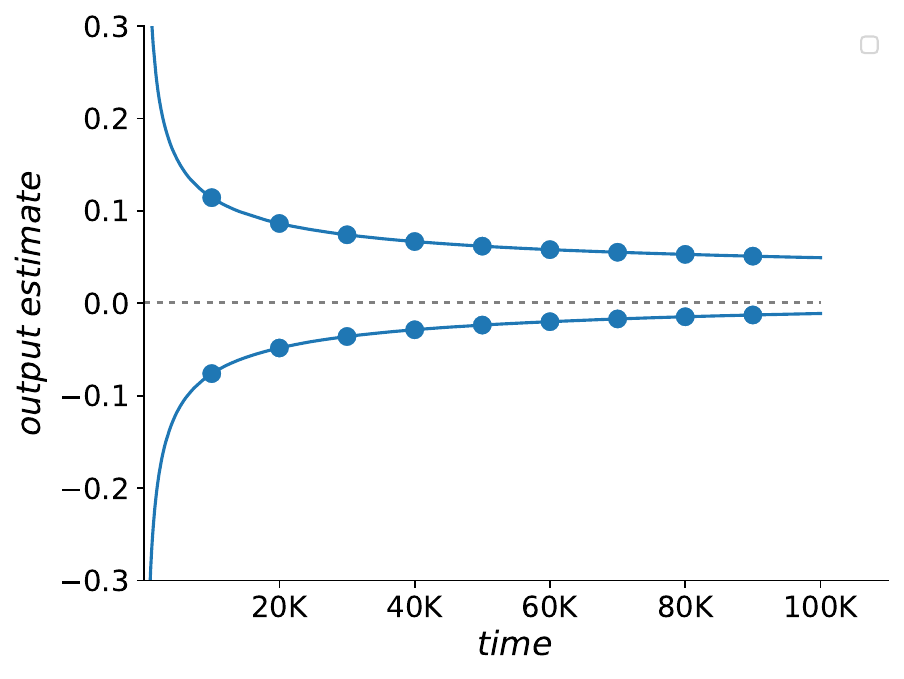}
\end{figure}

\section{Conclusion}
This work addressed the problem of monitoring algorithmic fairness properties over Markov chains with unknown transition dynamics, under both full and partial observability. For the partially observed case, we introduced a more expressive specification language alongside a corresponding, albeit slower, monitoring algorithm. For the fully observed case, we developed a specialized algorithm tailored to a restricted fragment of the specification language. We evaluate both approaches theoretically and empirically. For the empirical evaluation we utilize examples adapted from commonly used fairness benchmarks~\cite{damour2020fairness}.

Immediate future research directions are: improving the efficiency of monitoring under full observability for the general specification language; developing efficient model checking procedures for fairness properties; and exploring the applicability of statistical model checking techniques for offline fairness evaluation.

\section*{Acknowledgments}
This work is supported by the European Research Council under Grant No.: ERC-2020-AdG101020093.

\bibliographystyle{plain}
\bibliography{sn-bibliography}

\end{document}